\documentclass{article}

\usepackage[main,final]{neurips_2026}
\makeatletter

\usepackage{microtype}
\usepackage{graphicx}
\usepackage{subfigure}
\usepackage{booktabs} 

\usepackage{hyperref}

\usepackage[utf8]{inputenc} 
\usepackage[T1]{fontenc}    
\usepackage{hyperref}       
\usepackage{url}            
\usepackage{booktabs}       
\usepackage{amsfonts}       
\usepackage{nicefrac}       
\usepackage{microtype}      
\usepackage{xcolor}         

\usepackage{algorithm}
\usepackage{algorithmic}

\usepackage{wrapfig}

\usepackage{amsmath}
\usepackage{amssymb}
\usepackage{mathtools}
\usepackage{amsthm}

\usepackage{multirow}
\usepackage{multicol}

\usepackage{chngcntr}

\usepackage[capitalize,noabbrev]{cleveref}

\usepackage{thmtools,thm-restate}

\theoremstyle{plain}
\newtheorem{theorem}{Theorem}[section]

\newtheorem{lemma}[theorem]{Lemma}

\theoremstyle{definition}
\newtheorem{definition}[theorem]{Definition}

\theoremstyle{remark}
\newtheorem{remark}[theorem]{Remark}

\newcounter{motifcount}
\newenvironment{motif}
{
  \refstepcounter{motifcount}  
  \textbf{GEF \themotifcount: }\itshape
}

\usepackage[textsize=tiny]{todonotes}

\usepackage{amsfonts}
\usepackage{bm}
\usepackage{paralist}

 \newcommand{\COMMENTS}[1]{ \textit{// #1}}

\newcommand{\RETURN}{\STATE \textbf{return}~}

\DeclareMathOperator*{\argmin}{arg\,min}
\DeclareMathOperator*{\Poi}{Poi}

\newcommand{\Hom}{\text{Hom}}
\newcommand{\cheng}[1]{}
 \newcommand{\SK}[1]{}
\newcommand{\dan}[1]{}
\newcommand{\A}[1]{}
\def\vertex{\tikz[baseline=.1ex]{
\fill (0,0.1) circle (2pt) coordinate (A);}
}

\def\edge{\tikz[baseline=.1ex]{
\fill (0,0.15) circle (2pt) coordinate (A);
\fill (1.5ex,0.15) circle (2pt) coordinate (B);
\draw (A)--(B);}
}

\title{Flowette: Flow Matching with Graphette Priors for Graph Generation}

\author{%
  Asiri Wijesinghe$^1$\thanks{corresponding author: \texttt{asiriwijesinghe.wijesinghe@csiro.au}} \quad
  Sevvandi Kandanaarachchi$^1$ \quad
  Daniel M. Steinberg$^1$ \\
  \textbf{Cheng Soon Ong}$^{1,2}$ \quad \\
  $^1$CSIRO \quad $^2$Australian National University \\
}

\begin{document}

\maketitle

\begin{abstract}

  We study generative modeling of graphs with recurring subgraph motifs. We propose Flowette, a continuous flow matching framework that employs a graph neural network-based transformer to learn a velocity field over graph representations with node and edge attributes. Our model promotes topology-aware alignment through optimal transport-based coupling and preserves structural fidelity through endpoint- and chemistry-aware regularisation. To incorporate domain-driven structural priors, we introduce graphettes, a new probabilistic family of graph structure models that generalize graphons via controlled structural edits for motifs such as rings, stars, and trees. We theoretically analyze the coupling, invariance, and structural properties of the framework, evaluate it on synthetic and molecular benchmarks, and isolate the contributions of the structural prior, the optimal-transport coupling, and the regularisation terms through controlled ablations. Flowette achieves competitive performance overall, attaining state-of-the-art results on several metrics across multiple benchmarks, highlighting the effectiveness of combining structural priors with flow-based training for modeling complex graph distributions.
\end{abstract}

\section{Introduction} \label{sec:introduction}

Generative modeling of graphs has applications spanning molecular design, social networks, biological systems, and combinatorial optimization \cite{you2018graphrnn, kipf2016variational, zhu2022survey}. Graphs exhibit complex structure, permutation symmetry, and rich higher-order motifs such as cycles, hubs, and repeated subgraphs. These motifs encode essential domain semantics. For instance, molecular graphs are characterized by ring structures and valence constraints, 
social networks contain hub-and-spoke patterns, and many biological and combinatorial graphs are tree-like. Capturing such recurring substructures within a single framework, rather than via motif-specific architectures, remains a major open challenge for graph generative models.
Existing approaches rely on domain-specific architectures, resulting in specialized models for each motif \cite{jin2018junction, mahmood2021masked, karami2023higen}. This limits transferability and obscures the underlying principles of structure-aware graph generation. 


Recent progress in continuous-time generative modeling-notably diffusion models and flow matching-has provided powerful tools for high-dimensional generation \cite{songscore, lipman2023flowmatching}. Flow matching (FM) directly learns a velocity field transporting samples from a source distribution to the data distribution \citep{lipman2023flowmatching, liu2023rectifiedflow}. 
However, existing graph FM approaches typically rely on implicit or independent coupling between noise and data graphs, ignoring structural compatibility under permutation and size variability. As a result of forcing models to interpolate between topologically mismatched graphs, velocity supervision can be high-variance and inconsistent.


Three ingredients are essential for structure-aware flow-based graph generation. 
First, supervision pairs in flow matching should be structurally aligned.
Second, the training objective must remain reliable under finite-step integration, where local velocity errors can otherwise compound into invalid endpoints, while also respecting domain feasibility.
Third, the source distribution itself should encode meaningful structural priors, reflecting the dominant motifs.
\cheng{This is space, holding some space, much more space.
This is space, holding some space, much more space.
}

We introduce Flowette, a graph generation framework that contains these three ingredients. {\bf(1)} Flowette couples noise graphs and data graphs using Fused Gromov-Wasserstein (FGW) optimal transport \cite{titouan2019optimal}, yielding node permutation-consistent correspondences across variable-sized graphs. By providing supervision pairs that respect topology, we enable structure aligned flow matching on graphs. The velocity field is parameterized by a topology respecting Graph Neural Network (GNN) based transformer that jointly evolves node features, edge features, and adjacency values. {\bf(2)} We introduce regularisation terms in the training objective to preserve long range structures and permutation equivariance. The resulting training objective leads to more stable continuous time generation and improved structural fidelity. {\bf(3)} To enable structural inductive biases in the source distribution, we introduce a new mathematical object \emph{graphettes}, a probabilistic family of graph priors that generalize graphons via controlled edits. Graphettes allow principled injection or removal of motifs (rings, stars, cycles), providing a unified mechanism to model dense, sparse, and motif-rich graphs within a single framework. Unlike task-specific heuristics, graphettes decouple structural assumptions from the neural architecture and require only coarse structural knowledge of the target family (whether it contains rings, hubs, or trees), rather than empirical tuning of model components. We empirically validate that each of these three ingredients contributes independently to performance: replacing the graphette prior with a standard graphon, replacing FGW coupling with random or Euclidean pairing, or removing individual regularisation terms each produces substantial degradation across both synthetic and molecular benchmarks.

\vspace{-0.2cm}

\section{Background and setting}
\label{background}
Graphs have interchangeable representations, and we use the following conventions.
A graph $G$ with $n$ vertices (nodes) can be written in terms of vertices and edges, $(V, E)$, or equivalently by its $A \in \{0, 1\}^{n \times n}$ binary adjacency matrix.
We retrieve the sets of vertices and edges in $G$ by $V(G)$ and $E(G)$ respectively and denote the number of vertices by $|V(G)|$.

When we need to model attributed graphs with node and edge features, we use $X \in \mathbb{R}^{n \times d_x}$ for node features and $F\in \mathbb{R}^{n \times n \times d_f}$ for edge features of $G$. Sec. \ref{sec:methods} introduces the flow matching framework for  
attributed graphs, and we denote a graph $G$ by $(A, X, F)$, a triple consisting of the adjacency matrix, node and edge features of $G$. 
In Sec. \ref{sec:Graphettes} we introduce graphette priors focusing on non-attributed graphs $G$, which we denote by  $(V,E)$.  
We use $G \sim \mathcal{W}$ to indicate that a graph is sampled from some distribution or using a sampling function $\mathcal{W}$.

\cheng{This is space, holding some space, much more space.
}

\subsection{Flow Matching on Graphs}

\paragraph{Related Work.} Graph generation needs to capture local and global structures, as motivated by domains like molecular design, combinatorial optimisation, and biological networks \cite{mercado2021graph, sun2023difusco, yi2023graph}. Existing methods can be broadly grouped into autoregressive \citep{you2018graphrnn, liao2019efficient, mahmood2021masked, goyal2020graphgen} and one-shot paradigms \citep{dai2020scalable, shirzad2022td, karami2023higen, davies2023size}.
Diffusion-based generative models have recently become the dominant one-shot paradigm for graphs \citep{niu2020permutation, jo2022score}, but discrete time versions tie sampling trajectory to the training schedule.
Continuous-time discrete diffusion models based on Continuous Time Markov Chains (CTMCs) address this rigidity by decoupling the sampling grid from the training objective \citep{campbell2022continuous, xu2024discrete}. DisCo \cite{xu2024discrete}, for instance, defines a discrete-state continuous-time diffusion on graphs and inherits the permutation-equivariant/invariant properties of DiGress \citep{vignacdigress} while enabling flexible numerical integration and theoretical control of approximation error \citep{xu2024discrete}. 
Discrete flow matching (DFM) extends FM to discrete state spaces, where probability paths are defined by time-inhomogeneous CTMCs and the model parameterises transition rates rather than continuous velocities \citep{campbell2024discretefm, gat2024discrete,qinmadeira2024defog}. In parallel, some works have explored continuous FM and flow-based formulations for graphs by relaxing discrete structures into continuous spaces~\citep{liu2019graph,eijkelboom2024variational}.
However, they typically rely on implicit couplings between source and target graphs.
Concurrent work~\citep{chen2026prior} independently explores domain-informed source distributions for graph flow matching, focusing on graph reconstruction rather than generation.
Extended background on flow matching, together with a more thorough discussion of related work, is provided in Appendix~\ref{sec:more-related-work}.

\paragraph{Flow matching} offers an alternative to diffusion-based training for continuous-time generative models \citep{lipman2023flowmatching}. Rectified-flow variants adopt linear interpolants, and regress the associated constant conditional velocity field \citep{liu2023rectifiedflow}, leading to state-of-the-art results in high-dimensional image and sequence generation \citep{esser2024rectified, ma2024sit}. 
We need coupling between graph samples $G_0^i$ from a source (noise) distribution in $\mathcal{S}$ and target (data) graphs $G_1^j$ from the data distribution in $\mathcal{S}$.
$\mathcal{S}$ will later correspond to a continuous relaxation of adjacency matrices, node features, and edge features. 
FM learns a time-conditioned vector field $v_\theta : \mathcal{S} \times [0,1] \to \mathcal{S}$, parameterized by $\theta$, that approximates the conditional velocity along the path defined by time $t \in [0, 1]$.
It assumes paired samples $(G_0^i, G_1^j)$ together with interpolating states $(G_t)_{t\in[0,1]}$ defining a probability path between the noise distribution and the data distribution.
The FM training objective is defined as a regression loss:
\begin{equation}
\mathcal{L}_{\mathrm{FM}}(\theta)
=
\mathbb{E}_{\substack{G_0, G_1 \\ t \sim \mathcal{U}[0,1]}}
\left[
\left\|
v_\theta(G_t,t) - u^\star(G_t \mid G_0,G_1)
\right\|_2^2
\right].
\label{eq:fm_loss}
\end{equation}
At generation time, new samples are obtained by solving the ordinary differential equation $\frac{d G_t}{dt} = v_\theta(G_t,t)$, where $G_{t=0} \sim p_0$,
using a numerical integrator (e.g. Euler), and taking $G_{t=1}$ as a sample from the learned model distribution.


\begin{figure*}[t]
    \centering
    \includegraphics[width=1\linewidth]{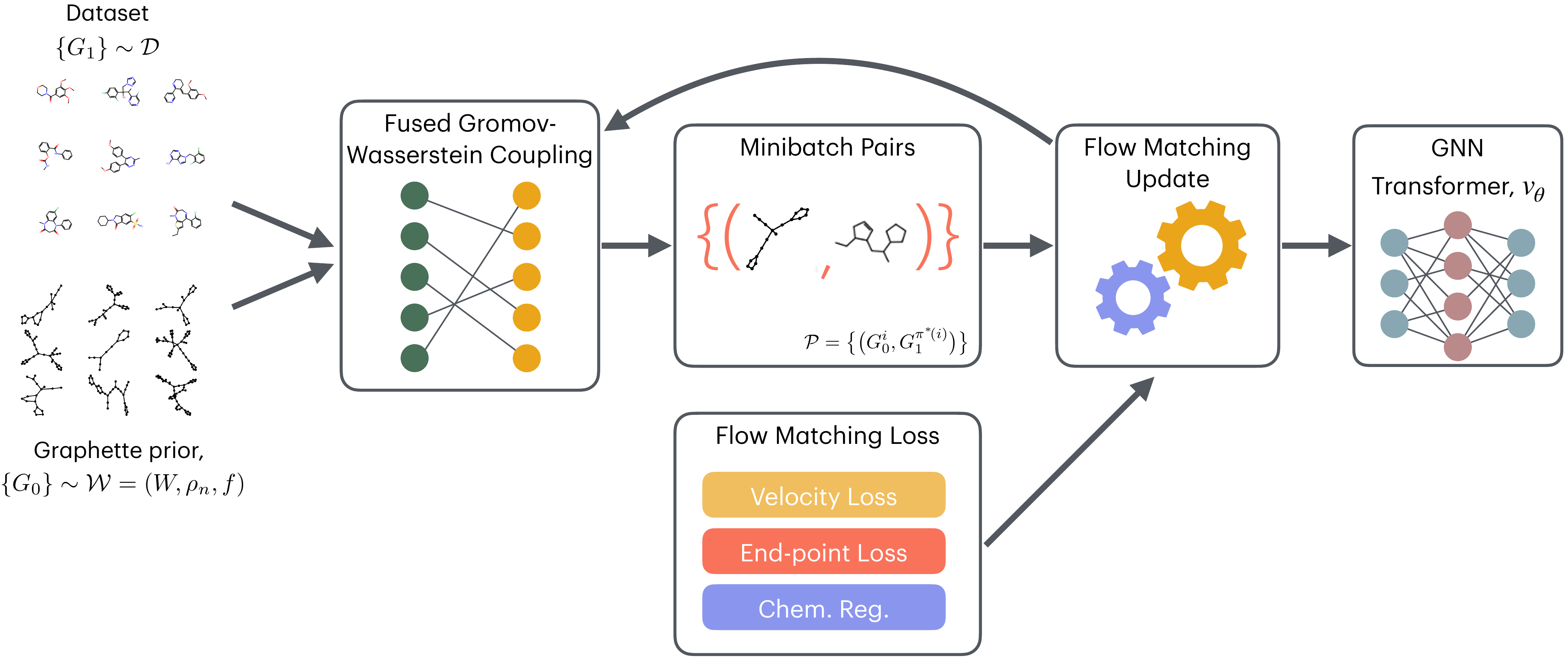} 
    \caption{Flowette training scheme. FGW optimal transport with Hungarian matching discussed in Sec. \ref{sec:fgw_fm}, GNN Transformer in App. \ref{sec:gnn_transformer},  Topology-aware loss and regularisation in Sec. \ref{sec:learning_objective} and graphettes in Sec. \ref{sec:Graphettes}.
    See Algorithm \ref{alg:fgw_fm_training} for details.}
    \label{fig:flowette}
    \vspace{-1em}
\end{figure*}

\subsection{Structural priors}

To provide suitable inductive bias for a particular application, we need a flexible class of probabilistic models to design meaningful priors. This work extends graphons, a widely used approach for graph priors \citep{han2022g}.
\begin{definition}\label{def:graphon}(\textbf{Graphon})
A \emph{graphon} is a symmetric, measurable function $W:[0,1]^2\to[0,1]$. Given $n\in\mathbb{N}$, a  graph $G$ with $n$ nodes can be sampled from $W$ by: (i) drawing latent variables $x_1,\dots,x_n \overset{iid}{\sim}\mathcal{U}[0,1]$, and (ii) for each $1\le i<j\le n$, sampling an edge $A_{ij}\sim\mathrm{Bernoulli}(W(x_i,x_j))$ independently, then set $A_{ji}=A_{ij}$. We write $G \sim \mathsf{G}(n,W)$.
\end{definition}

We give an accessible introduction to graphons, sparsified graphons, and graphexes (the structural priors that graphettes generalise) in Appendix~\ref{sec:graphons}.
The limitation of graphons is that graphs sampled from graphons are dense, i.e., the number of edges grow quadratically with the number of nodes. However, many real-world graphs such as social and molecular networks are sparse, i.e., the number of edges grow subquadratically with the nodes.  Sparsified graphons \citep{keriven2020convergence, klopp2017oracle} is an alternative used to model sparse graphs.



\begin{definition}\label{def:sparsifiedGraphon}(\textbf{Sparsified graphon}.) 
Let $W:[0,1]^2 \to [0,1]$ denote a graphon and $\{\rho_i\}_i=(\rho_1, \rho_2,\ldots, \rho_n, \ldots)$ denote a real-valued sequence converging to zero such that $0 \leq \rho_i \leq 1$. Then the graphon $W_n  := \rho_n W$ is called the sparsified graphon. 
We denote a graph $G$ with $n$ nodes  sampled from $W_n$ by  $G \sim \mathsf{G}(n,W, \rho_n)$. 
\end{definition}

\vspace{-0.2cm}

\section{Structure Preserving Flows} \label{sec:methods}
We visualise our overall design in \autoref{fig:flowette}.
In this section we address key challenges that arise when learning continuous-time generative dynamics on graphs and clarify how our method departs from the standard flow-matching pipeline. First, naively coupling noise and data graphs (e.g., by minibatch index) ignores structural compatibility and leads to high-variance supervision since the velocity field must interpolate between topologically mismatched source-target pairs. Second, unconstrained velocity fields defined over continuous graph representations can yield invalid structures under finite-step integration, particularly for molecular graphs where chemical constraints must be respected. 
To address these issues, we have three core components: (i) a structure-aware coupling mechanism for supervision (Sec.~\ref{sec:fgw_fm}), which replaces independent (random) coupling or Euclidean optimal transport coupling with FGW-distance and Hungarian graph matching; (ii) a permutation-equivariant GNN transformer that parameterizes the velocity field over evolving graph representations (Sec.~\ref{sec:fm_graph} and App. \ref{sec:gnn_transformer}); and (iii) an enhanced training objective (Sec.~\ref{sec:learning_objective}) that augments the flow-matching loss with endpoint consistency and chemistry-aware structural regularization. Together, these design choices yield coherent, permutation-invariant supervision across variable-sized graphs and promote structurally valid generations under finite-step integration.
In this section we use $G$ to refer to attributed graphs, $G = (A, X, F)$
and we present the full training scheme in Algorithm \autoref{alg:fgw_fm_training}.


\subsection{Graph Structure Coupling, $(G_0, G_1)$}
\label{sec:fgw_fm}

For our graphs we extract structural node embeddings $Z = [z_1,\dots,z_n]^\top = f_{\mathrm{enc}}(G) \in \mathbb{R}^{n \times d_z}$, using a pretrained edge-aware Graph Isomorphism Network (GIN) encoder $f_{\mathrm{enc}}$~\cite{xu2018powerful}. The encoder is a lightweight graph autoencoder trained on the training split only; never on validation or test data, and frozen thereafter. It contributes structural embeddings to the FGW cost but does not influence the velocity field, integration, or generated outputs (full details and a sensitivity ablation in Appendix~\ref{sec:sensitivity_analysis_encoder}).
%
%
%


\paragraph{FGW-Distance Computation.} 
Let $G_0 = (A_0, X_0, F_0)$ and $G_1 = (A_1, X_1, F_1)$ denote two undirected attributed graphs with $|V(G_0)| = n_0$ and $|V(G_1)| = n_1$, respectively. Let $Z_0 = f_{\mathrm{enc}}(G_0)$ and $Z_1 = f_{\mathrm{enc}}(G_1)$ denote the corresponding encoder-induced structural node embeddings. We define the intra-graph structural cost matrices $C(Z_0) \in \mathbb{R}^{n_0 \times n_0}$ and $C(Z_1) \in \mathbb{R}^{n_1 \times n_1}$ by $C(Z_a)_{ik} = \| z_i^{a} - z_k^{a} \|_2^2$, where $a \in \{0,1\}$. The cross-graph feature cost matrix $M(X_0, X_1) \in \mathbb{R}^{n_0 \times n_1}$ is defined as $M(X_0, X_1)_{ij} = \| x_i^{0} - x_j^{1} \|_2^2$. Using these costs, we can compute a fused optimal transport plan $T^\star \in \mathbb{R}^{n_0\times n_1}$ by solving,
\begin{equation}
\begin{split}
T^\star 
\;&\in\; 
\argmin_{T \in \Pi(p,q)} 
\; \bigl(1-\alpha\bigr)\,\langle T, M(X_0,X_1)\rangle
+
\alpha \,\mathcal{L}_{\mathrm{GW}}\!\big(T;\ C(Z_0), C(Z_1)\big)
,
\end{split}
\label{eq:fgw_objective}
\end{equation}
where $\alpha\in[0,1]$ trades off features and structure, 
$\Pi(p,q)=\{T\in\mathbb{R}_{\ge 0}^{n_0\times n_1}:T\mathbf{1}=p,\;T^\top\mathbf{1}=q\}$.
The corresponding FGW distance ~\citep{titouan2019optimal} is the optimal value of the fused cost minimized in Eq. ~\eqref{eq:fgw_objective}, i.e.,
\begin{equation}
\label{eq:app_fgw}
\mathrm{FGW}_\alpha(G_0,G_1)
:= \min_{T \in \Pi(p,q)}
\bigl(1-\alpha\bigr)\,\langle T, M(X_0,X_1)\rangle
+
\alpha \,\mathcal{L}_{\mathrm{GW}}\!\big(T;\ C(Z_0), C(Z_1)\big),
\end{equation}
where the squared-loss GW term is $
\mathcal{L}_{\mathrm{GW}}\!\big(T;C(Z_0),C(Z_1)\big) =
\sum_{i,k=1}^{n_0}\sum_{j,\ell=1}^{n_1}
\big(C(Z_0)_{ik}-C(Z_1)_{j\ell}\big)^2\,T_{ij}\,T_{k\ell}$.
%
%
We use uniform node measures $p \in \Delta_{n_0}$ and $q \in \Delta_{n_1}$, where $\Delta_{n} \;=\; \left\{ r \in \mathbb{R}^n_{\ge 0} \,\mid\, \mathbf{1}^\top r = 1 \right\}$ denotes the probability simplex on $n$ elements, and $p=n^{-1}_0\mathbf{1}_{n_0}$ for $G_0$ and $q=n^{-1}_1\mathbf{1}_{n_1}$ for $G_1$. 
The resulting optimal plan $T^\star$ provides the node-level correspondence used to align the target graph to the source-node ordering, while $\mathrm{FGW}_{\alpha}(G_0,G_1)$ denotes the corresponding optimal objective value and is used as the graph-to-graph discrepancy.


\paragraph{Batch-Level Coupling via Hungarian Matching.}
In each minibatch, we sample a set of $B$ noise graphs $\{G_0^i\}_{i=1}^B \sim \mathcal{W}$ and $B$ target graphs $\{G_1^j\}_{j=1}^B \sim \mathcal{D}$.
We compute a cost matrix $D\in\mathbb{R}^{B\times B}$ with entries
\begin{equation}
D_{ij} = \mathrm{FGW}(G_0^i,G_1^j),
\label{eq:batch_cost}
\end{equation}
then obtain a one-to-one bipartite assignment: $\pi^\star = \argmin_{\pi\in\mathfrak{S}_B} \sum_{i=1}^{B} D_{i,\pi(i)}$,
where $\mathfrak{S}_B$ is the set of permutations \cite{kuhn1955hungarian}. This prevents degenerate training in which many target graphs align to the same noise graph, and yields stable coverage across diverse graph topologies. 


We support three coupling protocols differing in how often this batch-level FGW computation is performed. The default, used throughout Algorithm~\ref{alg:fgw_fm_training} and all main results, is \emph{online coupling}: FGW costs and the Hungarian assignment are recomputed every minibatch. Two amortised variants -- \emph{fixed-pool} and \emph{periodic} coupling -- cache the correspondence across epochs and substantially reduce training cost (protocol details and a scalability study in App.~\ref{sec:scalability}).

We refer to this Hungarian matching procedure on FGW-distances as \emph{FGW-coupling}. It ensures that $(A_0,X_0,F_0)$ and $(A_1,X_1,F_1)$ correspond in a structure-preserving manner, resulting in lower-variance velocity supervision than independent coupling or Euclidean optimal transport coupling.

Finally, in App.~\ref{sec:fgw_proofs} we show that; FGW-distance is permutation invariant (\autoref{thm:fgw_perm_invariance}), FGW-distance is 0 for isomorphic graphs (\autoref{cor:fgw_zero_isomorphic}), and FGW-coupling recovers the optimal permutation-consistent bipartite (\autoref{thm:batch_pairing}).



\subsection{Flow Matching on FGW-Coupled Graphs.}
\label{sec:fm_graph}
For each FGW-coupled graph pair $(G_0,G_1)$, we consider graph representations $(A_0,X_0,F_0)$ and $(A_1,X_1,F_1)$ restricted to a common node budget $n$ under the structure-preserving node correspondence induced by FGW (the soft node correspondence and target-graph transport are given explicitly in App.~\ref{app:node-correspondence}).
In particular, the node budget is handled directly in our flow matching scheme by sampling noise graphs from the graphette prior with node counts explicitly matched to their paired target graphs, ensuring $|V(G_0)| = |V(G_1)|$ for every coupled pair. This avoids any size mismatch and removes the need for padding or truncation during interpolation. In this setting, $G_0 \sim \mathcal{W}$ denotes a noise graph sampled from the graphette prior (which we define in the next section), while $G_1 \sim \mathcal{D}$ denotes a target graph sampled from the data distribution. The prior $\mathcal{W}$ is used exclusively to generate the initial graph topology $A_0$. The $X_0$ and $F_0$ are then sampled independently from fixed categorical prior distributions corresponding to atom-type and bond-type priors, respectively. For $t \sim \mathcal{U}[0,1]$, we construct rectified linear interpolants,%
\begin{equation}
A_t = (1-t)A_0 + tA_1, \quad
X_t = (1-t)X_0 + tX_1, \quad
F_t = (1-t)F_0 + tF_1.
\end{equation}
Under this parameterization, the ideal transport velocities are constant and given by $\Delta A = A_1 - A_0$, $\Delta X = X_1 - X_0$, and $\Delta F = F_1 - F_0$. We model the continuous-time transport dynamics using a time-conditioned velocity field, $v_\theta : (A_t, X_t, F_t, t) \;\mapsto\; \big(v_A,\; v_X,\; v_F\big)$,
where $v_A \in \mathbb{R}^{n\times n}$, $v_X \in \mathbb{R}^{n\times d_x}$, and $v_F \in \mathbb{R}^{n\times n\times d_f}$ predict the instantaneous velocities of the adjacency, node features, and edge features, respectively. We use a novel \emph{permutation equivariant} GNN transformer to implement the velocity field predictor, $v_\theta$. We give details of its architecture in App.~\ref{sec:gnn_transformer} as well as proof of its permutation equivariance. 
Unlike existing graph transformers that operate on discrete graph states and predict categorical outputs -- DiGress \cite{vignacdigress}  and DeFoG \cite{qinmadeira2024defog} parameterise categorical denoising distributions or CTMC transition rates, while VFM/CatFlow \cite{eijkelboom2024variational} induces its vector field from a predicted categorical endpoint -- our transformer directly parameterises and integrates the coupled continuous velocities $(v_A, v_X, v_F)$ over the evolving triple $(A_t, X_t, F_t)$. App.~\ref{subsec:distinction_GNN} details architectural distinctions underlying this difference.
This velocity field predictor is trained via a continuous flow-matching objective, matching $(v_A,v_X,v_F)$ to the rectified displacements $(\Delta A,\Delta X,\Delta F)$ at intermediate times by minimizing a flow matching loss (Eq.~\ref{eq:total_loss}). 
Although the underlying graphs are discrete, all training and integration are performed in continuous space; discreteness is enforced only at generation time via projection (see Appendix \ref{subsec:discretisation}).


\begin{algorithm}[tb]
\caption{Flowette Flow Matching Training}
\label{alg:fgw_fm_training}
\resizebox{\linewidth}{!}{%
\begin{minipage}{1.00\linewidth} 
\begin{algorithmic}[1]
\small
\INPUT
Dataset $\mathcal{Y}$; pretrained encoder $f_{enc}$; graphette prior $\mathcal{W} = (W, \rho_n, f)$; 
batch size $B$; epochs $\mathcal{R}$; weights $(\lambda_x,\lambda_e)$; learning rate $\eta$; FGW tradeoff $\alpha$.

\OUTPUT Velocity Field (GNN Transformer) $v_{\theta}$.\

Initialize $v_{\theta}$ and initialize optimizer $\mathrm{AdamW}(\theta, \eta)$\;

\FOR{$\text{epoch}=1$ to $\mathcal{R}$}
  \FORALL{minibatch of target graphs $\mathcal{G}_1=\{G_1^j\}_{j=1}^{B} \sim \mathcal{D}$}
    \STATE Let $n_j \leftarrow |V(G_1^j)|$\;

    \COMMENTS{\scriptsize Sample a batch of noise graphs from a graphette}
    \STATE $\mathcal{G}_0=\{G_0^i\}_{i=1}^{B} \sim \mathcal{W}$\;

    \COMMENTS{\scriptsize Batch FGW-distances (Algorithm~\ref{alg:batch_fgw})}
    \STATE $(D, T) \leftarrow \textsc{BatchFGW}(\mathcal{G}_0, \mathcal{G}_1, f_{enc}, \alpha)$\;

    \COMMENTS {\scriptsize Hungarian matching on FGW costs to get matched graph pairs}
    \STATE $\pi^\star \leftarrow \textsc{HungarianMatching}(D)$ 
    
    \STATE Coupled pairs $\mathcal{P}\leftarrow\{(G_0^i, G_1^{\pi^\star(i)})\}_{i=1}^{B}$\;

    \COMMENTS{\scriptsize Batch gradient update for $v_\theta$ (Algorithm~\ref{alg:fm_update})}
    \STATE $\theta \leftarrow \textsc{FlowMatchingUpdate}(\mathcal{P}, \theta)$
  \ENDFOR
\ENDFOR
\RETURN{$v_{\theta}$}\;
\end{algorithmic}
\end{minipage}%
}
\end{algorithm}

\subsection{Flow Matching Learning Objective} \label{sec:learning_objective}
We propose a rectified flow-matching objective for graph generation that learns a continuous-time transport from graphette priors to the data distribution while preserving both local and global graph structure.
Since a fixed velocity residual induces a larger predicted endpoint error when incurred early in the trajectory than late (App.~\ref{subsec:lend_derivation}), locally accurate velocity predictions on average do not guarantee low endpoint error under finite-step solvers, an issue exacerbated in graph domains. We therefore couple local velocity matching with an endpoint-aware reweighting of that same residual, together with domain-specific regularization, to ensure stable and structurally valid graph generation.
Let $\beta_{\mathrm{end}}, \beta_{\mathrm{val}}, \beta_{\mathrm{atom}} \ge 0$ be scalar weighting coefficients. The overall training objective is as follows where each term is defined below.
\begin{equation}
\label{eq:total_loss}
\mathcal{L}(\theta) = \mathcal{L}_{\mathrm{vel}} + \beta_{\mathrm{end}} \cdot \mathcal{L}_{\mathrm{end}} + \beta_{\mathrm{val}} \cdot \mathcal{L}_{\mathrm{val}} + \beta_{\mathrm{atom}} \cdot \mathcal{L}_{\mathrm{atom}},
\end{equation}


\paragraph{Local Velocity Matching Loss.}
We enforce local flow correctness by matching predicted velocities to the rectified target transport. Let $\lambda_x, \lambda_e \ge 0$ be scalar weighting coefficients, and let $N_A = n^2$, $N_X = n\,d_x$, $N_F = n^2 d_f$ denote the number of entries in $A$, $X$, and $F$ respectively. Each term below is thus a Mean Squared Error (MSE) over its tensor's entries:
\begin{equation}\label{eq:app_vel_loss}
\mathcal{L}_{\mathrm{vel}} =
\frac{1}{N_A}\|v_A-\Delta A\|_F^2 + \lambda_x\frac{1}{N_X}\|v_X-\Delta X\|_F^2 + \lambda_e\frac{1}{N_F}\|v_F-\Delta F\|_F^2,
\end{equation}
where $v_A$, $v_X$, and $v_F$ denote the predicted velocity components for $A, X$ and $F$, respectively. The targets $\Delta A=A_1-A_0$, $\Delta X=X_1-X_0$, and $\Delta F=F_1-F_0$ correspond to the rectified transport directions induced by the linear interpolation between the prior graph and the target graph. This loss enforces correct local transport behavior at each intermediate state along the rectified flow path.

\paragraph{Endpoint Consistency Loss.}
While $\mathcal{L}_{\mathrm{vel}}$ constrains the instantaneous transport direction at a randomly sampled intermediate time $t$, we additionally supervise the one-step extrapolated endpoint implied by the rectified flow parameterization. Given $A_t$, $X_t$, and $F_t$ at time $t$, a single-step rectified integration predicts the endpoint
\begin{equation}
\widehat{A}_1 = A_t + (1-t)v_A, \quad
\widehat{X}_1 = X_t + (1-t)v_X, \quad
\widehat{F}_1 = F_t + (1-t)v_F,
\label{eq:endpoint-pred}
\end{equation}
where $\widehat{A}_1, \widehat{X}_1$ and $\widehat{F}_1$ denote the predicted adjacency, node features, and edge features at $t=1$. We penalize deviations from the target graph, again as a per-entry MSE using the same normalization constants $N_A, N_X, N_F$ defined above, via
\begin{equation}
\mathcal{L}_{\mathrm{end}} = \frac{1}{N_A}\|\widehat{A}_1 - A_1\|_F^2
+ \lambda_x \frac{1}{N_X}\|\widehat{X}_1 - X_1\|_F^2
+ \lambda_e \frac{1}{N_F}\|\widehat{F}_1 - F_1\|_F^2.
\label{eq:lend}
\end{equation}
As we show in App.~\ref{subsec:lend_derivation}, $\mathcal{L}_{\mathrm{end}}$ is an algebraically induced, non-uniform temporal reweighting of the same velocity residual used in $\mathcal{L}_{\mathrm{vel}}$, with weight $(1-t)^2$: a fixed velocity residual induces a larger predicted endpoint error when incurred earlier in the trajectory ($t$ closer to $0$) than later ($t$ closer to $1$), since it is extrapolated over more of the remaining path. Combining $\mathcal{L}_{\mathrm{vel}}$ and $\beta_{\mathrm{end}}\mathcal{L}_{\mathrm{end}}$ therefore places greater emphasis on velocity residuals with a larger effect on the predicted terminal state, which remains useful under finite-step numerical integration, where errors in early-trajectory velocity predictions have outsized influence on the final generated graph.

\paragraph{Chemistry-aware Regularization.}
To promote chemically valid molecular graphs, we incorporate two lightweight regularization terms applied at the predicted endpoint. First, a \emph{soft valence constraint} ($\mathcal{L}_{\mathrm{val}}$) penalizes violations of atom valence limits by discouraging over-bonding through a soft valence penalty. Second, an \emph{atom-type marginal matching} ($\mathcal{L}_{\mathrm{atom}}$) term aligns the global distribution of predicted atom types with that of the target graph, preventing drift in node-type composition. Full formulations are provided in Appendix ~\ref{sec:chem_regularisation}.

\paragraph{Theoretical Properties of the Learning Objective.}
Theorem~\ref{thm:endpoint_consistency} shows that, under the rectified parameterization, jointly driving $\mathcal{L}_{\mathrm{vel}}$ and $\mathcal{L}_{\mathrm{end}}$ to zero forces $v_\theta$ to recover the ideal constant transport field and guarantees exact endpoint reconstruction under Euler integration; this is consistent with $\mathcal{L}_{\mathrm{end}}$ being a reweighted view of the velocity residual rather than an independent target, since driving that residual to zero at every $t$ trivially drives both losses to zero together. 
Theorem~\ref{thm:stability_app} separately quantifies how deviations from the ideal velocity field control the endpoint error under $K$-step Euler integration; in the rectified setting, the ideal velocity is constant, so no additional discretisation-error term arises from the Euler solver.
For the computational complexity of these components, see Appendix~\ref{sec:complexity}, and for runtime analysis, see Appendix~\ref{sec:runtime_analysis}.

\vspace{-0.2cm}

\section{Structural Bias with Graphette Priors} \label{sec:Graphettes}
\begin{table*}[t]
    \caption{Graphette priors, as used in empirical benchmarks in Section~\ref{sec:results}.}
    \centering
    \small
    \scalebox{0.90}{\begin{tabular}{cl}
        \toprule
         Application  &  Graphette $\mathcal{W} = ( W, \rho_n, f)$  \\
         \midrule
          Community graphs & $f = I$ (GEF \ref{motif:identity}), $\rho_n = 1$, and $W$ is a mixture of stochastic block model graphons \\
          Tree graphs & $f = h$ (GEF \ref{motif:tree}), $\rho_n = (\bar{W} n)^{-1} + \epsilon$ and $W(x,y) = 0.2$ \\
          Egonets & $f = R \circ S$ (GEF \ref{motif:ring3} and \ref{motif:star}), $\rho_n = 1$ and $W(x,y) = \exp\left( -\frac{10}{3}(x + y)\right)$ \\
         Molecular graphs  &  $f = R$ (GEF \ref{motif:ring3}), $\rho_n = (\bar{W} n)^{-1} + \epsilon$, where $\bar{W} = \int_0^1 \int_0^1 W(x,y) dx dy$ and $W(x,y) = 0.2$ \\
          \bottomrule 
    \end{tabular}}
    \label{tab:GraphettePriors}
    \vspace{-1em}
\end{table*}


Inspired by the star function in graphexes \citep{veitch2019sampling} and sparsified graphons we propose a generalized model that includes a graph edit function, which can edit graphs sampled from $W_n$ or $W$. In this section $G = (V,E)$ refers to non-attributed graphs.  
%
 \begin{definition}\label{def:graphette}(\textbf{The graphette})
    We define  $\mathcal{W} = (W, \rho_n, f)$ 
    where $W:[0, 1]^2 \to [0, 1]$ denotes a graphon,  $\{\rho_i\}_i$ denotes a real-valued sequence with $0 \leq  \rho_i \leq 1$ and $f$ denotes a graph edit function. A graph is generated using graphette $\mathcal{W}$ by first sampling $G'$ from $W_n :=  \rho_nW $ and then enacting the function $f$, i.e.,  $G' \sim \mathsf{G}(n, W, \rho_n)$ followed by $G \sim f(G')$.
\end{definition} 
%
%
For the computationally inclined reader, we provide a generative model description of graphettes and graphons in Appendix~\ref{appendix:generativeStories}. Our model has the flexibility to sample dense and sparse graphs. For sparse graphs, the graphette $\mathcal{W}$ is sparsified by $\rho_n \to 0$ (Def.~\ref{def:sparsifiedGraphon}).
As is common in other GNN papers \citep{liu2024graph, vignacdigress} when a sample from the sparsified graphon (Def.~\ref{def:sparsifiedGraphon}) returns graphs with multiple components, we choose the largest connected component as the graph before applying the edits described below. 

\subsection{Rich and recurring subgraph motifs}\label{sec:graphEditFunctions}

We can choose a graph edit function $f$ to suit our purpose. If a graphon $W$ can generate the graphs we require, such as stochastic block model graphs, we can let $f$ be the identity function, making no modifications to the graph. If we are modeling social networks, then $f$ can add stars.  Ring addition is useful to model molecular networks as they have ring structures such as benzene. 

Motivated by these problem domains we use four graph edit functions (GEF) in our experiments: \textbf{GEF 1:} $f = I$ the identity function in instances where the graph does not need to be edited, \textbf{GEF 2:} $f = h$, a cycle deletion function when we know the graphs do not have cycles, \textbf{GEF 3:} $f = R$, a ring addition function when the graphs exhibit ring structures, and \textbf{GEF 4:}  $f = S$, a star addition function for graphs with hub structures such as social networks.

\begin{remark}
    Details of graph edit functions along with some characteristics of graphettes are explored in Appendix \ref{Appendix:graphettes}. In particular, graphettes can recover graphon $W$ and the sparsified graphon $W_n$, have the ability to generate graphs similar to those generated by graphexes and generate dense graphs from a family of graphons. Furthermore, the ring function GEF \ref{motif:ring3} can generate dense or sparse graphs depending on $\{\rho_i\}_i$ (see Lemmas in Appendix~\ref{Appendix:graphettes}).
\end{remark}

We specify $f$ manually per dataset family (Table~\ref{tab:GraphettePriors}); Sec.~\ref{subsec:graphette_sensitivity} discusses learned alternatives and quantifies sensitivity to graphette misspecification (see also App.~\ref{sec:limitations}). 

\subsection{Star and ring functions are well-behaved}\label{sec:wellbehaved}

Graph homomorphisms are widely used to count subgraph motifs and have been used in GNNs \citep{jin2024homomorphism}. 

\begin{definition}\label{def:homomorphism} Let $G_1$ and $G_2$ be graphs. A \textbf{graph homomorphism} from $G_1$ to $G_2$ is a map $f : V(G_1) \to V(G_2)$ such that if $uv \in E(G_1)$ then $f(u)f(v) \in E(G_2)$. (Maps edges to edges.) Let $\Hom(G_1, G_2)$ be the set of all such homomorphisms and let $\hom(G_1, G_2) = |\Hom(G_1, G_2)|$ be the number of homomorphisms.
Then \textbf{homomorphism density} is defined as $t(G_1, G_2) = \frac{\hom(G_1, G_2)}{ |V(G_2)|^{|V(G_1)|} }$.
\end{definition}

\begin{definition}\label{def:triagleCovered} (\textbf{Triangle-covered graphs}.) Let $G = (V, E)$ be a graph. Then $G$ is triangle-covered if for every vertex $u \in V(G)$ there exist vertices $\{v, w\} \in V(G)$ such that edges $\{uv, uw, vw\} \in  E(G)$. We denote the set of triangle-covered graphs by $\mathcal{F}$.
\end{definition}

Roughly stated a graph is triangle-covered if every vertex belongs to a triangle. 
We show that homomorphisms of triangle-covered graphs are preserved by star and ring functions (GEF \ref{motif:ring3}, \ref{motif:star}). 

\begin{restatable}{theorem}{thmtrianglecovered}\label{thm:trianglecovered}
    Let $\mathcal{W} = (W, \rho_n, f)$ denote a graphette and let $G' \sim \mathsf{G}(n, W, \rho_n)$ and $G \sim f(G')$ where $f \in \{R, S\}$ (GEF \ref{motif:ring3}, \ref{motif:star}) with ring size $c > 3$ for $R$. Suppose $|V(G')| = n$ and $|V(G)| = n + m$. Then for any triangle-covered $F \in \mathcal{F}$,  
    $\hom(F, G) =  \hom(F, G')$ and $t(F, G)  = \frac{n^{|V(F)|}}{(n + m)^{|V(F)|}}t(F, G')$. 
\end{restatable}

Theorems \ref{thm:staraddition} and \ref{thm:ringaddition} supplement the above result by computing the effect of star and ring functions on vertex and edge homomorphisms and their densities.

\vspace{-0.2cm}
\section{Numerical Experiments} \label{sec:results}
In this section, we evaluate the proposed flow matching framework on  synthetic graph generation benchmarks and real-world molecular graph generation tasks. The synthetic benchmarks are designed to probe structural fidelity under controlled graph families, while the molecular benchmarks assess chemical validity, diversity, and novelty in realistic settings \footnote{Code: \href{https://github.com/wokas36/Flowette} { https://github.com/wokas36/Flowette}}.

\subsection{Graphette priors}\label{sec:Gpriors}
Graphs generated from the $\mathcal{G}$ operator are used as priors in the flow matching framework. 
It is important to select an appropriate graphette prior that matches the graphs being modeled. Table \ref{tab:GraphettePriors} lists the graphette priors used for each application.

\subsection{Experiments on Synthetic Graphs}\label{sec:resultsSynthetic}
We evaluate Flowette on three standard synthetic graph generation benchmarks: Tree \cite{bergmeisterefficient}, Stochastic Block Model (SBM) \cite{martinkus2022spectre}, and Ego-small graphs \cite{jo2022score}. We adopt the same training, validation, and test splits as in \cite{qinmadeira2024defog}. Generation quality is assessed using seven different metrics, described in more detail in Appendix~\ref{sec:metrics-syn}. We exclude chemistry-specific regularization terms for these synthetic data, by setting $\beta_\mathrm{val} = 0$ and $\beta_\mathrm{atom}=0$. We used $L=3$ attention layers for Tree and Ego-small datasets, and $L=4$ for the SBM dataset. Other hyperparameter settings are provided in Appendix~\ref{sec:hyperparam-syn}.

\begin{figure*}[t]
    \centering
    \includegraphics[width=0.95\linewidth]{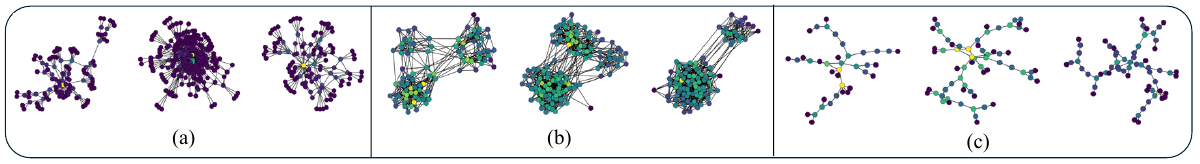}
    \vspace{-1.0em}
    \caption{Flowette-generated synthetic graphs: Ego-small (a), SBM (b), and Tree (c).}
    \label{fig:synthetic_samples}
\end{figure*}




We summarise results on synthetic datasets in Tables~\ref{Tab:ablation_sbm}--\ref{Tab:ego_small}. Across SBM, Tree, and Ego-small, Flowette achieves the strongest performance on higher-order structural metrics (e.g., Orbit) and generative quality measures (e.g., Validity, Uniqueness, Novelty, and V.U.N.), while remaining competitive on Degree and Clustering. This reflects the model's ability to capture global structure and motif-level patterns (e.g., community structure in SBM, acyclicity in Tree, and hub-centric motifs in Ego-small), with minor deviations in lower-order statistics (e.g., Deg. and Clus.) that do not affect overall graph fidelity. Tables of results against baseline models are shown in Appendix~\ref{sec:moreresults-syn}. Figure~\ref{fig:synthetic_samples} presents representative synthetic graphs generated by Flowette on the Ego-small, SBM, and Tree datasets. Additional samples in Appendix~\ref{sec:generated-graphs} further demonstrate the validity of the generative model.

\begin{figure*}[t]
    \centering
    \includegraphics[width=0.95\linewidth]{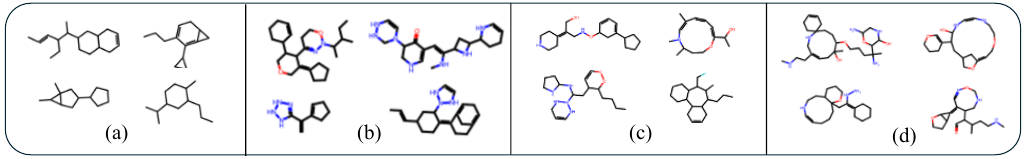}
    \vspace{-1.0em}
    \caption{Flowette-generated molecules: QM9 (a), ZINC250K (b), MOSES (c) and Guacamol (d).}
    \label{fig:molecular_samples}
    \vspace{-0.5em}
\end{figure*}

\subsection{Experiments on Molecular Graphs}
\label{sec:expt-molecules}

We further evaluate our framework on four standard molecular graph generation benchmarks: QM9 \cite{wu2018moleculenet}, ZINC250K \cite{sterling2015zinc}, Guacamol \cite{brown2019guacamol}, and MOSES \cite{polykovskiy2020molecular}, which together span small molecules and larger, drug-like compounds. Experiments on QM9 use the conventional dataset splits and evaluation criteria employed in recent molecular generation work \cite{vignacdigress, qinmadeira2024defog}. For MOSES and Guacamol, we follow the respective benchmark protocols for training and assessment \cite{polykovskiy2020molecular, brown2019guacamol}. Details of the metrics are presented in Appendix~\ref{sec:metrics-molecules}.
As can be seen in Table~\ref{Tab:qm9_zinc} and~\ref{Tab:guacamol_moses}, our approach achieves competitive performance across all metrics, with state-of-the-art performance on several metrics while remaining comparable to strong baselines on others. Figure ~\ref{fig:molecular_samples} shows representative molecular graphs generated by Flowette on QM9, ZINC250K, MOSES, and Guacamol datasets. We used $L=4$ attention layers for Guacamol and MOSES, and $L=5$ for QM9 and ZINC250K. Further details of our experimental protocol are provided in Appendix~\ref{sec:hyperparam-molecules}.




\begin{table*}[t]
   \caption{Generation performance on QM9 and ZINC250K molecular graphs.}
   \label{Tab:qm9_zinc}
   \centering
   \footnotesize
   \scalebox{0.85}{\begin{tabular}{ c l c c c c c c c c c | c c c c c c c c } 
     \toprule
     \multicolumn{2}{c}{} & \multicolumn{3}{c}{QM9}  &\multicolumn{3}{c}{ZINC250K} \\\cmidrule(lr){3-5} \cmidrule(lr){6-8}
     \multicolumn{2}{c}{Model} & Valid $\uparrow$ & Unique $\uparrow$ & NSPDK $\downarrow$ & Valid $\uparrow$ & Unique $\uparrow$ & NSPDK $\downarrow$ \\
     \midrule
     \multirow{7}{*}{} 
     & EDP-GNN \cite{niu2020permutation} & 47.52 & 99.25 & 0.0046 & 82.97 & 99.79 & 0.0485 \\
     & GraphARM \cite{kong2023autoregressive} & 90.25 & 95.65 & 0.002 & 88.23 & 99.46 & 0.055 \\
     & DiGress \cite{vignacdigress} & 98.19 & 96.20 & 0.0003 & 94.99 & 99.97 & 0.0021 \\
     & GDSS \cite{jo2022score} & 95.72 & 98.46 & 0.0033 & 97.01 & 99.64 & 0.0195 \\
     & GDSS + TF \cite{jo2022score} & 99.68 & - & 0.0024 & 96.04 & - &  0.0326 \\
     & GruM \cite{jo2024graph} & 99.69 & 96.90 & 0.0002 & 98.65 & 99.97 & 0.0015 \\
     & GraphBFN \cite{song2025smooth} & 99.73 & - & 0.0002 & 99.22 & - & 0.0013 \\
     & SID \cite{bogetsimple} & 99.67 & 95.66 & \textbf{0.0001} & 99.50 & 99.84 & 0.0021 \\
     & DeFoG \cite{qinmadeira2024defog} & 99.30 & 96.30 & - & 99.22 & 99.99 & 0.0008 \\
     & GBD \cite{liuadvancing} & \textbf{99.88} & - & 0.0002 & 97.87 & - & 0.0018 \\
     & Flowette (Ours) & 99.81 \tiny $\pm$ 0.09 & \textbf{99.30} \tiny $\pm$ 0.05 & 0.0003 \tiny $\pm$ 0.0002 & \textbf{99.90} \tiny $\pm$ 0.10 & \textbf{100.00} \tiny $\pm$ 0.0 & \textbf{0.0006} \tiny $\pm$ 0.0004 \\
     \bottomrule
   \end{tabular}}
\end{table*}
\begin{table*}[t]
   \caption{Generation performance on Guacamol and MOSES molecular graphs.}
   \label{Tab:guacamol_moses}
   \centering
   \footnotesize
   \scalebox{0.90}{\begin{tabular}{ c l c c c c c c c c c | c c c c c c c c } 
     \toprule
     \multicolumn{2}{c}{} & \multicolumn{3}{c}{Guacamol}  &\multicolumn{4}{c}{MOSES} \\\cmidrule(lr){3-5} \cmidrule(lr){6-9}
     \multicolumn{2}{c}{Model} & Valid $\uparrow$ & V.U.N $\uparrow$ & KL div. $\uparrow$ & Unique $\uparrow$ & Novelty $\uparrow$ & SNN $\uparrow$ & Scaff. $\uparrow$ \\
     \midrule
     \multirow{7}{*}{} 
     & DiGress \cite{vignacdigress} & 85.2 & 85.1 & 92.9 & \textbf{100.0} & 95.0 & 0.52 & 14.8 \\
     & DisCo \cite{xu2024discrete} & 86.6 & 86.5 & 92.6 & \textbf{100.0} & 97.7 & 0.50 & 15.1 \\
     & Cometh \cite{siraudincometh} & 98.9 & 97.6 & 96.7 & 99.9 & 92.6 & 0.54 & \textbf{16.0} \\
     & GraphBFN \cite{song2025smooth} & - & - & - & 99.8 & 89.0 & \textbf{0.59} & 10.0 \\
     & G2PT \cite{chen2025graph} & 95.3 & 94.8 & 95.6 & \textbf{100.0} & 79.4 & 0.55 & 2.9 \\
     & GEEL \cite{jangsimple} & 88.2 & 77.1 & 93.1 & \textbf{100.0} & 81.1 & 0.52 & 3.6 \\
     & DeFoG \cite{qinmadeira2024defog} & \textbf{99.0} & 97.9 & 97.7 & 99.9 & 92.1 & 0.55 & 14.4 \\
     & Flowette (Ours)  & 98.6 & \textbf{98.0} & \textbf{97.9} & 99.9 & \textbf{98.1} & 0.58 & 15.3 \\
     \bottomrule
   \end{tabular}}
   \vspace{-1em}
\end{table*}



\subsection{Sensitivity to Graphette Edit-Function Misspecification}
\label{subsec:graphette_sensitivity}
As noted in Sec.~\ref{sec:Graphettes}, $f$ is specified manually per dataset family (Table~\ref{tab:GraphettePriors}) because the graphette is intended only as a coarse structural bias on the flow's source distribution, not as the final generative model: it supplies domain-level topology -- rings, hubs, or acyclicity -- while the velocity field $v_\theta$ learns the remaining target-specific structure. This also lets us isolate the benefit of structural source bias without introducing a second learned generator. The graphette's components need not remain manually specified in general: the base graphon $W$ could be estimated directly from data using implicit graphon-learning methods such as Implicit Graphon Neural Representation~\citep{xia2023implicit} or Scalable Implicit Graphon Learning~\citep{azizpour2025scalable}; the sparsification sequence $\rho_n$ could be fitted to match empirical edge-density scaling in $n$; and the edit function $f$ could be selected or learned from motif statistics or validation performance. We use simple, interpretable priors in this work to isolate the contribution of structural source bias; the remainder of this section quantifies sensitivity to the choice of $f$.

The ablation in App.~\ref{subsec:graphette_vs_graphon} (graphette prior vs. a standard graphon) removes structural bias entirely and therefore represents an extreme mismatch. To more precisely quantify sensitivity to \emph{which} motif the graphette encodes, we hold every other component of Flowette fixed -- FGW coupling, the velocity-field architecture, and the training objective -- and vary only the edit function $f$ used in the molecular graphette (Table~\ref{tab:GraphettePriors}). The results are shown in Table ~\ref{tab:graphette_sensitivity}. The intended prior for molecular graphs uses ring addition $R$, reflecting the prevalence of cyclic structure in molecules. We compare it against two increasingly incompatible choices: star addition $S$, a structurally unrelated motif, and cycle deletion $h$, which is directly contradictory since it actively removes the ring structure the target domain requires.
\begin{table}[h]
\centering
\caption{Sensitivity of Flowette to graphette edit-function misspecification on QM9 and ZINC250K. All components other than $f$ are held fixed.}
\label{tab:graphette_sensitivity}
\scalebox{0.7}{\begin{tabular}{llllll}
\toprule
Dataset & Edit function & Prior variant & Valid & Unique & NSPDK \\
\midrule
\multirow{3}{*}{QM9}
& $R$ (Ring addition) & Intended prior 
& $\mathbf{99.81}$ {\tiny $\pm$ 0.09} 
& $\mathbf{99.30}$ {\tiny $\pm$ 0.05} 
& $\mathbf{0.0003}$ {\tiny $\pm$ 0.0002} \\
& $S$ (Star addition) & Incorrect motif 
& $70.50$ {\tiny $\pm$ 0.06} 
& $69.30$ {\tiny $\pm$ 0.06} 
& $0.0160$ {\tiny $\pm$ 0.0008} \\
& $h$ (Cycle deletion) & Contradictory motif 
& $53.70$ {\tiny $\pm$ 0.04} 
& $54.45$ {\tiny $\pm$ 0.06} 
& $0.0540$ {\tiny $\pm$ 0.0004} \\
\midrule
\multirow{3}{*}{ZINC250K}
& $R$ (Ring addition) & Intended prior 
& $\mathbf{99.90}$ {\tiny $\pm$ 0.10} 
& $\mathbf{100.00}$ {\tiny $\pm$ 0.0} 
& $\mathbf{0.0006}$ {\tiny $\pm$ 0.0004} \\
& $S$ (Star addition) & Incorrect motif 
& $73.40$ {\tiny $\pm$ 0.06} 
& $75.30$ {\tiny $\pm$ 0.06} 
& $0.1105$ {\tiny $\pm$ 0.0006} \\
& $h$ (Cycle deletion) & Contradictory motif 
& $51.58$ {\tiny $\pm$ 0.04} 
& $50.65$ {\tiny $\pm$ 0.05} 
& $0.2050$ {\tiny $\pm$ 0.0007} \\
\bottomrule
\end{tabular}}
\end{table}
Validity degrades consistently as the prior becomes less compatible with the target domain: on QM9, replacing $R$ with $S$ reduces validity from $99.81\%$ to $70.50\%$, and replacing it with the contradictory $h$ reduces validity further to $53.70\%$, with analogous trends on ZINC250K and substantial increases in NSPDK under both misspecified priors. Flowette therefore benefits from -- and is sensitive to -- a coarse but structurally appropriate source prior, rather than being invariant to the choice of graphette edit function. This is consistent with, but a finer-grained probe than, the graphette-vs-graphon ablation of App.~\ref{subsec:graphette_vs_graphon}: here all variants retain a graphette-structured source, differing only in whether the encoded motif matches the target domain.

\subsection{Ablation Analysis of Regularization Terms.}
Table~\ref{tab:ablation_qm9_zinc} isolates the components of the Flowette objective (Eq.~\eqref{eq:total_loss}) on QM9 and ZINC250K. Removing $\mathcal{L}_{\mathrm{end}}$ or $\mathcal{L}_{\mathrm{val}}$ substantially reduces validity, whereas removing $\mathcal{L}_{\mathrm{atom}}$ has a milder effect on validity but consistently reduces uniqueness and novelty, confirming that chemical feasibility does not emerge automatically from generic flow-matching supervision -- without indicating that the backbone itself is molecule-specific, since the same backbone achieves state-of-the-art on our synthetic benchmarks on several metrics (Sec.~\ref{sec:resultsSynthetic}) with $\beta_{\mathrm{val}}=\beta_{\mathrm{atom}}=0$. 
Full discussion, including why $\mathcal{L}_{\mathrm{vel}}$ alone is insufficient on molecular graphs and how the endpoint-aware temporal reweighting places greater emphasis on residuals with a longer remaining extrapolation horizon, is provided in App.~\ref{sec:ablation_reg}; further ablations on the encoder sensitivity, graphette prior, and FGW coupling are in App.~\ref{sec:sensitivity_analysis_encoder}, App. ~\ref{subsec:graphette_vs_graphon}, and App.~\ref{subsec:ablation_fgw_coupling_vs_alternative_pairing}, respectively.

\paragraph{Observed limitations.}
Several aspects of Flowette offer avenues for further improvement, spanning coupling scalability, automated graphette specification, domain-general endpoint regularisation, and tighter train/generation integration; we discuss these limitations in App.~\ref{sec:limitations}.

\vspace{-0.2cm}

\section{Conclusion} \label{sec:conclusion}

We propose Flowette, a generative model for graphs with recurring sub-graph motifs, built around three components: (1) structural alignment of supervision pairs via Fused Gromov-Wasserstein optimal transport; 
(2) a training objective that combines local velocity matching with an endpoint-aware reweighting of the same residual, together with optional chemistry-aware terms for molecule generation;
(3) source distributions that encode structural priors via \emph{graphettes}, a new mathematical object that extends graphons through controlled sub-graph edits (ring or star injection, cycle removal). The velocity field is learned by a GNN-based transformer. Flowette achieves competitive or state-of-the-art performance on several metrics across multiple synthetic and molecular benchmarks, and our ablations confirm that each of the three components contributes independently to performance. 
We hope each of these components; structural alignment, endpoint-aware regularisation, and graphette priors, will be useful in their own right for future work in graph generation.

\bibliography{NeurIPS_Submission/references}
\bibliographystyle{abbrv}

\newpage
\onecolumn
\appendix

\section{Background on Graphons}\label{sec:graphons}

Graphons are symmetric, measurable functions defined on a unit square denoted by $W:[0,1]^2 \to [0,1]$.  Graphons are graph limits, i.e., by mapping the adjacency matrix to the unit square and taking the limit as the number of nodes go to infinity, we can obtain a graphon for converging graph sequences. When the adjacency matrix is mapped to the unit square, we obtain the empirical graphon, defined below.  

\subsection{Graphons are graph limits}
\begin{definition}\label{def:empiricalgraphon}
    Given a graph $G$ with $n$ vertices labeled $\{1, \ldots, n\}$, we define its \textbf{empirical graphon} $W_G: [0, 1]^2 \rightarrow [0, 1]$ as follows: We split the interval $[0, 1]$ into $n$ equal intervals $\{I_1, I_2, \ldots, I_n \}$ (first one closed, all others half open) and for $x \in I_i, y \in I_j$ define
    $$ W_G(x,y) = \begin{cases}
        1 & \, \text{if} \quad  ij \in E(G) \, \\
        0 & \, \text{otherwise} \, ,
    \end{cases}
    $$
    where $E(G)$ denotes the edges of $G$. The empirical graphon replaces the adjacency matrix with a unit square and the $(i,j)$th entry of the adjacency matrix is replaced with a square of size $(1/n) \times (1/n)$. 
\end{definition}

Figure \ref{fig:graphConvergence} shows the empirical graphons for 3 graphs generated from the Erdős–Rényi $G(n,p)$ model where $G$ has $n$ vertices for $n \in \{50, 100, 500 \}$ with edge probability $p = 0.2$. This graph sequence converges to $W = 0.2$. 

\begin{figure}[!ht]
    \centering
    \includegraphics[width=0.9\linewidth]{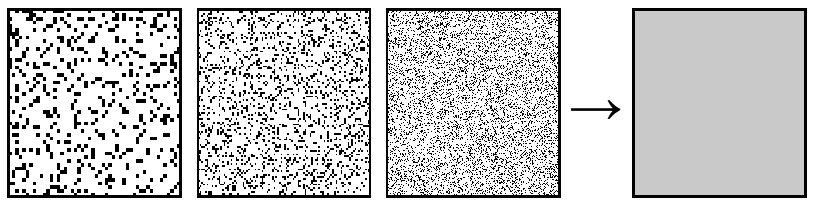}
    \caption{Three $G(n,p)$ graphs for $n \in \{50, 100, 500 \}$ with edge probability $p = 0.2$ converging to $W = 0.2$.}
    \label{fig:graphConvergence}
\end{figure}

Graph convergence is defined in terms of graph homomorphism density and the cut-metric \cite{borgs2008convergent}. Both these notions are equivalent \cite{borgs2011limits}. We define the cut norm and the cut metric below. 

\begin{definition}\label{def:cut1}The \textbf{cut norm} of graphon $W$   \citep{frieze1999quick, borgs2008convergent} is defined as
$$ \lVert W \rVert_\square = \sup_{S, T} \left\vert\int_{S\times T} W(x,y) \, dx dy \right\vert  \, , 
$$
where the supremum is taken over all measurable sets $S$ and $T$ of $[0 ,1]$.  \end{definition}

\begin{definition}\label{def:cut2} Given two graphons $W_1$ and $W_2$ the \textbf{cut metric}  \citep{borgs2008convergent} is defined as 
$$ \delta_{\square}(W_1, W_2) = \inf_\varphi \left\lVert W_1 - W_2^\varphi \right\rVert_\square \, , 
$$
where the infimum is taken over all measure preserving bijections $\varphi:[0,1] \rightarrow [0,1]$.  
\end{definition}

The role of $\varphi$ is to account for graph isomorphisms. Two isomorphic graphs $G$ and $H$ may have different node labels making $\left\lVert W_G - W_H \right\rVert_\square$ non-zero. However, for these two graphs $\inf_\varphi \left\lVert W_G - W_H^\varphi \right\rVert_\square $ is zero as $\varphi$ accounts for bijections.  

As shown in \cite{borgs2011limits} a graph sequence $\{G_n\}_n$ is convergent if and only if it is convergent in the cut metric. 

\begin{theorem}[Graph convergence \cite{borgs2011limits}]\label{thm:BCconvergent2} 
A sequence of graphs $\{G_n\}_n$ is convergent if and only if it is Cauchy in the $\delta_\square$ distance. The sequence  $\{G_n\}_n$ converges to $W$ if and only if $\delta_\square(W_{G_n}, W) \to 0$. Furthermore, if this is the case, and $|V(G_n)| \to \infty$, then there is a way to label the nodes of the graphs $G_n$ such that $\lVert W_{G_n}  - W \rVert_\square \to 0$. 
\end{theorem}

\subsection{Graphons are graph generative models}

Graphons can be used to generate new graphs. Given a graphon $W$ we can sample a graph with $n$ nodes as follows:

\begin{definition}\label{def:wrandomgraphs}Uniformly pick $x_1, x_2, \ldots, x_n$ from $[0,1]$. A \textbf{W-random graph} $\mathsf{G}(n,W)$ has the vertex set $1, 2, \ldots, n$ and vertices $i$ and $j$ are connected with probability $W(x_i, x_j)$. That is, for each node pair $i$ and $j$ we add an edge $A_{ij} \sim \text{Bernoulli}(W(x_i, x_j))$.
\end{definition}

Figure \ref{fig:samplingGraphsFromW} shows two graphons and graphs generated from them for different values of $n$.  The graphon on the top row generates ring type graphs and the graphon on the bottom row generates stochastic block model type graphs.

\begin{figure}[!ht]
    \centering
    \includegraphics[width=0.9\linewidth]{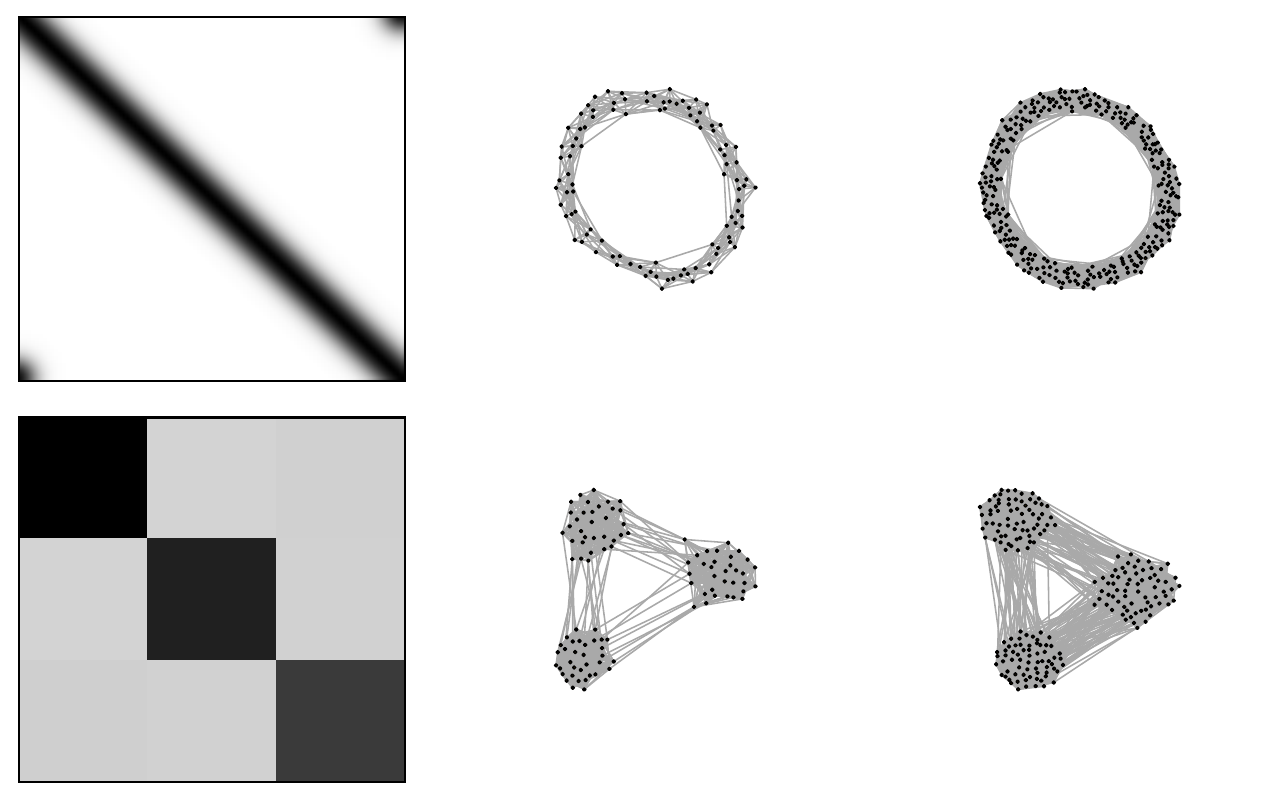}
    \caption{Two graphons on the left and W-random graphs $\mathsf{G}(n,W)$ for different $n$ on the right.}
    \label{fig:samplingGraphsFromW}
\end{figure}

Next we introduce dense and sparse graph sequences. 

\begin{definition}[\bf Dense graph sequences]\label{def:dense}
    A sequence of graphs $\{G_n\}_n$ is dense if the number of edges $m$ grow quadratically with the number of nodes $n$, i.e.,  $m \in \Theta(n^2)$.  
\end{definition}
\begin{definition}[\bf Sparse graph sequences]\label{def:sparse}
     A sequence of graphs $\{G_n\}_n$ is sparse if the number of edges $m$ grow sub-quadratically with the number of nodes $n$, i.e., $m \in o(n^2)$.
\end{definition}

The limitation of graphons is that as a consequence of the Aldous-Hoover theorem \cite{aldous1981representations}, graphs generated from graphons are dense or empty. However, most real-world graphs such as social networks and molecular graphs are sparse. As such, many graphon-based extensions are proposed to mitigate this limitation. 

\subsection{Graphon extensions for sparse graphs}
We discuss two sparse graph graphon extensions: graphon sparsification \citep{klopp2017oracle}, and graphexes \cite{veitch2015class}.  

\begin{definition}\label{def:sparsifiedGraphon}(\textbf{Sparsified graphon}.) 
Let $W:[0,1]^2 \to [0,1]$ denote a graphon and $\{\rho_i\}_i=(\rho_1, \rho_2,\ldots, \rho_n, \ldots)$ denote a real-valued sequence converging to zero such that $0 \leq \rho_i \leq 1$. Then the graphon $W_n  := \rho_n W$ is called the sparsified graphon. 
We denote a graph $G$ with $n$ nodes  sampled from $W_n$ by  $G \sim \mathsf{G}(n,W, \rho_n)$. 
\end{definition}

Note that $\rho_n W \to 0$ with $n$. As such, if $G_n$ is sampled from $\rho_nW $ then the sequence $\{G_n\}_n$ is sparse. Figure \ref{fig:sparsifiedGraphon} shows a stochastic block model graphon $W$ and two sparsified versions $\rho_nW$ for $\rho_n \in \{0.1, 0.05 \}$ along with graphs sampled from these three graphons for $n \in \{100, 200, 300\}$. We see the block structure is diluted for smaller values of $\rho_n$ both in the graphon and in the sampled graph. In addition, the number of isolated nodes increase as $\rho_n$ decreases. 

\begin{figure}[t]
    \centering
    \includegraphics[width=0.9\linewidth]{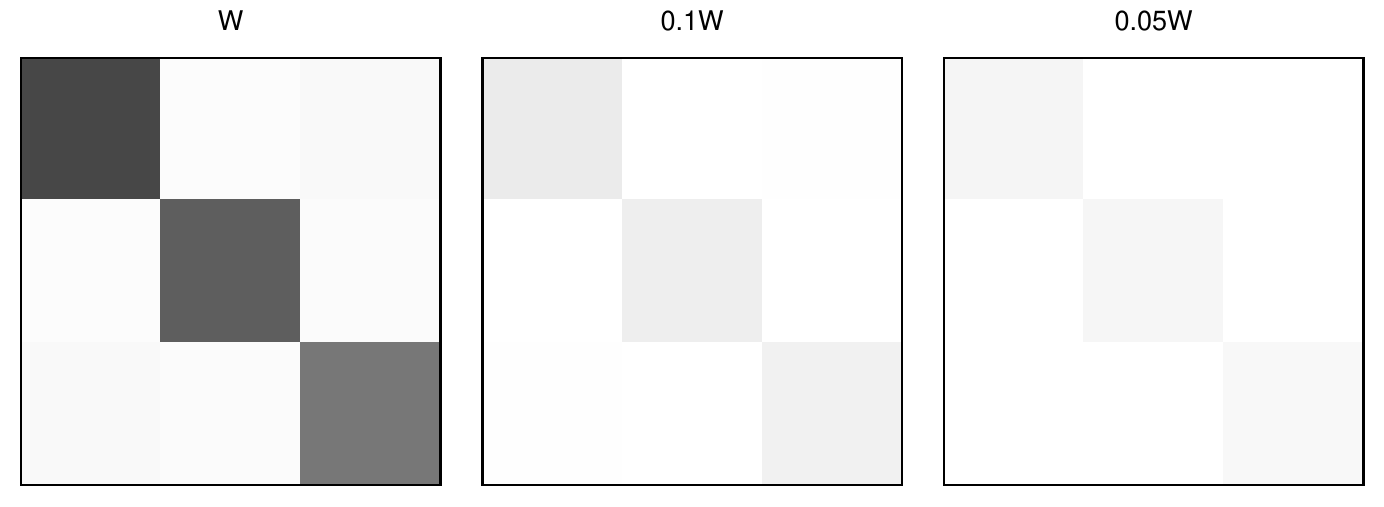}
    \includegraphics[width=0.9\linewidth]{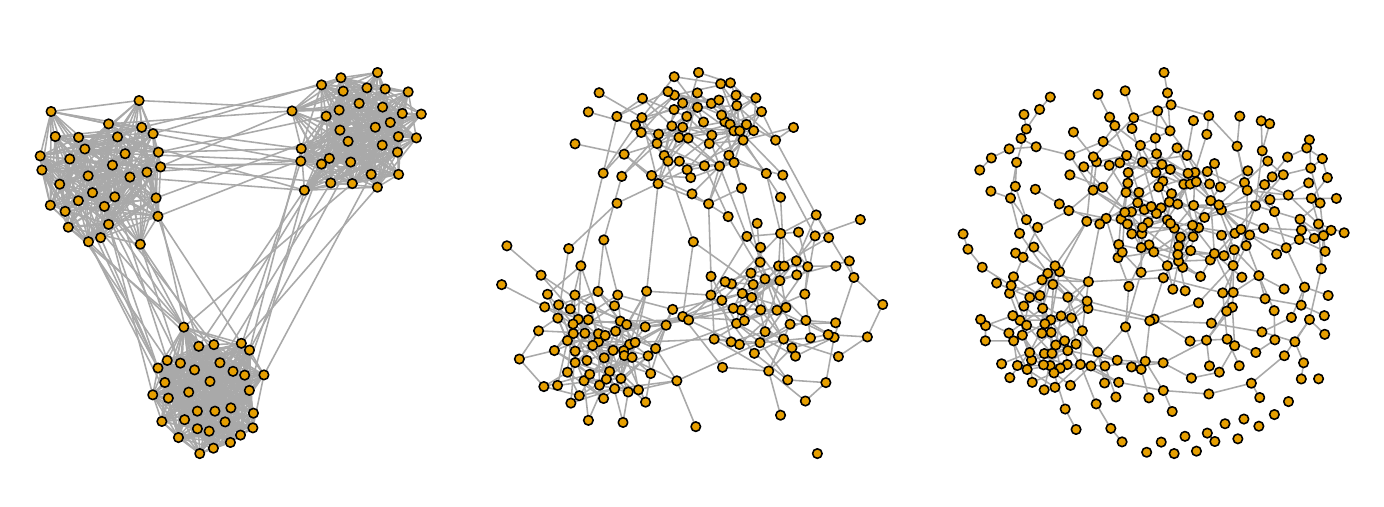}
    \caption{Graphon sparsification where $W_n = \rho_n W$. Top row shows the original graphon $W$ and two instances of $\rho_n W$. The bottom row shows 3 graphs sampled from the graphon directly above. }
    \label{fig:sparsifiedGraphon}
\end{figure}

While graphon sparsification achieves sparsity, it has limitations. The main limitation is that it doesn't create large hub structures prevalent in social networks. Graphexes \cite{veitch2015class} have the capacity to address this issue. 

\paragraph{Graphexes.}
Graphexes were originally introduced by Veitch and Roy \cite{veitch2015class} as a triple $\mathcal{W}=(W,S,I)$ where
\begin{itemize}
\item $W:\mathbb{R}_+^2\to[0,1]$ is a symmetric measurable function (graphon component),
\item $S:\mathbb{R}_+\to\mathbb{R}_+$ is an integrable function (star component),
\item $I\ge 0$ is a constant (dust/isolated edges component).
\end{itemize}
The main difference of graphexes from the original graphon framework is that 
motivated from the work by Caron and Fox \cite{caron2017sparse} 
they sample graphs via a Poisson process construction and use Kallenberg exchangeability in their regularity framework to account for the exchangeability of vertices that are modeled as reals ($\mathbb{R}^+$) instead of integers.  
The Poisson process construction coupled with a graphon $W$ defined on $\mathbb{R}_+^2$ allows them to generate sparse graphs. 
Having a separate star function gives the capability to add stars akin to social networks.

Our proposed structural prior, \textit{graphettes}, takes inspiration from these two sparse graph models. We keep the graphon defined on the unit square, as this simplifies computations, and employ a sparsified graphon in our construction. To account for structures such as rings or stars that are not generated by the sparsified graphon $W_n$, we introduce a graph edit function $f$, similar to the star function $S$ in graphexes. However, the graph edit function $f$ has broader functionality than the star function, as it can modify the graph in multiple ways, for example, by adding rings or deleting edges. As a result, graphettes allow us to generate a diverse range of graphs, from molecular structures to social networks.





\section{Historical Background and Related Work on Flow Matching}
\label{sec:more-related-work}

\subsection{Graph generative models}

Graph generative modeling aims to learn distributions over graphs that capture both local connectivity patterns and global structural properties, with applications in molecular design, combinatorial optimisation, and biological networks \cite{mercado2021graph, sun2023difusco, yi2023graph}. Existing methods can be broadly grouped into autoregressive and one-shot paradigms. Autoregressive models sequentially grow a graph by inserting nodes and edges according to a chosen ordering \citep{you2018graphrnn, liao2019efficient, mahmood2021masked, goyal2020graphgen}. This sequential view allows fine-grained control and the injection of domain-specific constraints (e.g.\ valency checks in molecular generation), but comes at a cost: models must either learn an effective node/edge ordering or commit to a hand-crafted one, which enlarges the hypothesis space and breaks permutation invariance \citep{kong2023autoregressive, han2023autoregressive}. As a result, autoregressive approaches tend to be computationally expensive, difficult to parallelise, and brittle on graphs with complex long-range dependencies.

To mitigate these issues, a large body of work has explored one-shot generation, where all graph variables are predicted in a single pass. One-shot models instantiate graph-adapted versions of VAEs, GANs, and normalising flows \citep{kipf2016variational, de2018molgan, liu2019graph}. While they restore permutation equivariance/ invariance and avoid learning an ordering, purely one-shot models often underfit rich higher-order structure unless equipped with carefully engineered architectures or hierarchical decompositions \citep{dai2020scalable, shirzad2022td, karami2023higen, davies2023size}.

\textbf{Graph diffusion models.}
Diffusion-based generative models have recently become the dominant one-shot paradigm for graphs. Early work adapted continuous diffusion frameworks to graph-structured data by treating adjacency and features as continuous variables \citep{niu2020permutation, jo2022score}, which simplifies optimisation but departs from the discrete nature of graph topology. Subsequent methods introduced discrete-state diffusion over node and edge categories using discrete-time Markov chains \citep{austin2021structured,vignacdigress,haefeli2022diffusion}, achieving strong performance on benchmark graph and molecular datasets. However, these models tie the sampling trajectory to the training schedule: the number of reverse steps must match the (fixed) number of forward noising steps, and sampling cost is therefore tightly coupled to training hyperparameters.

Continuous-time discrete diffusion models based on Continuous Time Markov Chains (CTMCs) address this rigidity by decoupling the sampling grid from the training objective \citep{campbell2022continuous, xu2024discrete}. DisCo \cite{xu2024discrete}, for instance, defines a discrete-state continuous-time diffusion on graphs and inherits the permutation-equivariant/invariant properties of DiGress while enabling flexible numerical integration and theoretical control of approximation error \citep{xu2024discrete}. Other variants exploit degree information or hierarchical decompositions to improve scalability and sample quality \citep{chen2023efficient, bergmeisterefficient, limnios2023sagess}. Despite these advances, diffusion-based graph models still rely on carefully designed rate matrices or noise schedules, and their local node/edge updates make it difficult to preserve global topology and higher-order motifs without additional inductive biases or auxiliary constraints.

\textbf{Graph flow matching models.}
Flow matching (FM) offers an efficient alternative to diffusion-based training for continuous-time generative models. Rather than learning to reverse a fixed stochastic process, FM directly regresses a time-dependent vector field whose trajectories transport a simple source distribution to a complex target distribution along a chosen probability path \citep{lipman2023flowmatching}. Rectified-flow variants further simplify training by adopting linear interpolants between data and noise and regressing the associated constant conditional velocity field \citep{liu2023rectifiedflow}, leading to state-of-the-art results in high-dimensional image and sequence generation \citep{esser2024rectified, ma2024sit}. 

Extending FM to discrete state spaces has given rise to Discrete Flow Matching (DFM), where probability paths are defined by time-inhomogeneous CTMCs and the model parameterises transition rates rather than continuous velocities \citep{campbell2024discretefm, gat2024discrete}. Compared to discrete diffusion, DFM replaces hand-designed forward rate matrices with learnable dynamics and decouples the choice of interpolation path from the training objective. Building on this, DeFoG introduces the first DFM-based framework for graph generation, combining CTMC-based flows with a graph transformer backbone and a training–sampling decoupling that improves flexibility over purely diffusion-based methods \citep{qinmadeira2024defog}.

In parallel, some works have explored continuous FM and flow-based formulations for graphs by relaxing discrete structures into continuous spaces. Graph normalising flows and related models learn invertible transports over continuous embeddings of adjacency and attributes \citep{liu2019graph}, while variational flow matching for graph generation embeds graphs into continuous latent spaces and trains vector fields via a variational objective \citep{eijkelboom2024variational}. These approaches demonstrate that FM-style training can be competitive to diffusion-based graph generators in terms of efficiency and flexibility. However, they typically rely on implicit couplings between source and target graphs (e.g.\ through shared encoders or attention over batched graphs) and do not explicitly enforce structural alignment under permutation and size variability. Consequently, velocity supervision can become inconsistent across permutations or mismatched graph components, making it difficult to guarantee stable preservation of global topology and higher-order motifs during integration. In particular, whereas DiGress \cite{vignacdigress}, DeFoG \cite{qinmadeira2024defog}, and VFM \cite{eijkelboom2024variational} predict categorical outputs or induce their vector field from a categorical endpoint, our GNN transformer directly integrates continuous, topology-conditioned velocities with persistent edge states (App. \ref{subsec:distinction_GNN}).

Our work situates itself in this emerging line of flow-matching graph generators, but departs from prior art in two key ways: (1) we adopt a continuous FM framework over relaxed adjacency and node features while explicitly aligning noise and data graphs via structure-aware matching, and (2) we incorporate principled structural priors into both the source distribution and the training objective to capture local and global structural properties. This design yields coherent, permutation-invariant velocity supervision across variable-sized graphs and enables Flowette to learn stable continuous-time graph transformations that preserve rich combinatorial structure beyond what is typically captured by existing diffusion- or flow-based models.



\subsection{Flow matching training}
FM provides a framework for learning continuous-time generative models by directly regressing a time-dependent velocity field that transports samples from a simple source distribution to a target data distribution along a prescribed probability path \citep{lipman2023flowmatching}. Unlike diffusion-based models, FM does not require simulating or inverting a stochastic corruption process, and avoids likelihood evaluation or score estimation.

Let $\mathcal{S}$ denote a continuous state space and let $p_0$ and $p_1$ be probability distributions over $\mathcal{S}$, representing a source (noise) distribution and a target (data) distribution, respectively. In the graph setting, $\mathcal{S}$ will later correspond to a continuous relaxation of adjacency matrices, node features, and edge features. Flow matching assumes access to paired samples $(s_0, s_1)$ with $s_0 \sim p_0$ and $s_1 \sim p_1$, together with a family of interpolating states $(s_t)_{t\in[0,1]}$ defining a probability path between $p_0$ and $p_1$.

A commonly used choice is the \emph{rectified linear path},
\begin{equation}
s_t = (1-t)s_0 + t s_1,
\qquad t \sim \mathcal{U}[0,1],
\label{eq:rectified_path}
\end{equation}
which yields a constant conditional velocity
\begin{equation}
u^\star(s_t \mid s_0,s_1) = s_1 - s_0.
\end{equation}
Flow matching learns a time-conditioned vector field $v_\theta : \mathcal{S} \times [0,1] \to \mathcal{S}$, parameterised by $\theta$, that approximates the conditional velocity along the path.

The FM objective is defined as a regression loss:
\begin{equation}
\mathcal{L}_{\mathrm{FM}}(\theta)
=
\mathbb{E}_{\substack{s_0 \sim p_0,\, s_1 \sim p_1 \\ t \sim \mathcal{U}[0,1]}}
\left[
\left\|
v_\theta(s_t,t) - u^\star(s_t \mid s_0,s_1)
\right\|_2^2
\right].
\end{equation}
Minimising Eq.~\eqref{eq:fm_loss} encourages the learned vector field to induce trajectories that transport mass from $p_0$ to $p_1$.

At generation time, new samples are obtained by solving the ordinary differential equation
\begin{equation}
\frac{d s_t}{dt} = v_\theta(s_t,t),
\qquad s_{t=0} \sim p_0,
\label{eq:fm_ode_vel}
\end{equation}
using a numerical integrator (e.g.\ Euler or Runge-Kutta), and taking $s_{t=1}$ as a sample from the learned model distribution.

\section{GNN Transformer for Learning Velocity Fields}\label{sec:gnn_transformer}
At time $t\in[0,1]$, we represent the continuous graph state as $(A_t,\, X_t,\, F_t)$, where $A_t\in\mathbb{R}^{n\times n}$ is a soft adjacency matrix, $X_t\in\mathbb{R}^{n\times d_x}$ denotes node features, and $F_t\in\mathbb{R}^{n\times n\times d_f}$ denotes edge features. The adjacency $A_t$ is treated as a continuous variable during flow matching and serves as a structural bias rather than a discrete connectivity constraint.



\textbf{Input projections and time conditioning.}
A scalar time $t \in [0,1]$ is embedded via
\begin{equation}
c_t = \psi_t(t) \in \mathbb{R}^{H},
\end{equation}
where $\psi_t$ is a Multi-Layer Perceptron (MLP). To enable joint node–edge reasoning, we map node and edge features into a common latent space of dimension $H$ using separate learnable transformations. Specifically, node features are embedded via a function
\[
\psi_x:\mathbb{R}^{d_x}\rightarrow\mathbb{R}^{H},
\qquad
h_i^{(0)} = \psi_x\!\big(X_t(i)\big) + c_t,
\] 
while edge features are embedded via a distinct function
\[
\psi_f:\mathbb{R}^{d_f+1}\rightarrow\mathbb{R}^{H},\;\;
b_{ij}^{(0)} = \psi_f\!\big([\,F_t(i,j)\;\|\;A_t(i,j)\,]\big) + c_t,
\]
where $\|\,$ denotes concatenation. In practice, $\psi_x$ and $\psi_f$ are implemented as learnable linear layers or MLPs.

We denote the node and edge hidden states at layer $\ell$ by $h_i^{(\ell)} \in \mathbb{R}^{H}$ and $b_{ij}^{(\ell)} \in \mathbb{R}^{H}$. 
Given node features $X_t \in \mathbb{R}^{n \times d_x}$, edge features $F_t \in \mathbb{R}^{n \times n \times d_f}$, and adjacency $A_t \in \mathbb{R}^{n \times n}$, the GNN transformer predicts velocity fields $(v_A, v_X, v_F)$ used for flow.




\paragraph{Stacked attention layers.}
We apply $L$ layers of the edge-aware velocity attention:
\begin{equation}
h_i^{(\ell)}, b_{ij}^{(\ell)} =
\mathrm{AttnLayer}\big(
h_i^{(\ell-1)}, b_{ij}^{(\ell-1)}, A_t^{(\ell-1)} \big),
\quad \ell=1,\dots,L.
\end{equation}

\paragraph{Velocity prediction heads.}
Final velocity fields are obtained via linear projections:
\begin{align}
v_X(i) &= h_i^{(L)} W_{\text{out}}^X \in \mathbb{R}^{n \times d_x}, \\
v_F(i,j) &= b_{ij}^{(L)} W_{\text{out}}^E \in \mathbb{R}^{d_f}, \\
v_A(i,j) &= w_A^\top b_{ij}^{(L)} \in \mathbb{R}. 
\end{align}

\paragraph{Symmetry enforcement.}
To ensure undirected graph consistency, we symmetrize outputs as
\begin{equation}
v_A \leftarrow \tfrac{1}{2}(v_A + v_A^\top), 
\qquad
v_F(i,j) \leftarrow \tfrac{1}{2}\big(v_F(i,j) + v_F(j,i)\big).
\end{equation}

This architecture jointly models node, edge, and adjacency dynamics under a unified time-conditioned attention mechanism. By explicitly propagating dense edge states and topology-aware attention biases, the model captures higher-order structural evolution, which is critical for flow matching on molecular graphs.

\subsection{Edge-Aware Velocity Attention Layer}

\paragraph{Structural-biased attention.} 
We first compute query and key projections
\begin{equation}
Q_i^{(\ell)} = h_i^{(\ell)} W_Q^{(\ell)}, 
\qquad 
K_j^{(\ell)} = h_j^{(\ell)} W_K^{(\ell)},
\end{equation}
where $W_Q^{(\ell)}, W_K^{(\ell)} \in \mathbb{R}^{H \times H}$.
The attention logits are defined as
\begin{equation} \label{eq:sij}
S_{ij}^{(\ell)} 
= 
\frac{\langle Q_i^{(\ell)}, K_j^{(\ell)} \rangle}{\sqrt{H}} 
+ A^{(\ell)}_t(i,j),
\end{equation}
where the adjacency $A^{(\ell)}_t(i,j)$ acts as an additive structural bias. 
Row-wise normalized attention weights are then given by
\begin{equation}
\tilde{\alpha}_{ij}^{(\ell)} = \mathrm{softmax}_j(S_{ij}) =
\frac{\exp(S_{ij}^{(\ell)})}{\sum_{k=1}^{N} \exp(S_{ik}^{(\ell)})}.
\end{equation}

\paragraph{Edge-conditioned message construction.}
For each ordered node pair $(i,j)$, we form a joint edge-conditioned concatenated representation
\begin{equation}\label{eq:uij}
u_{ij}^{(\ell)} = [\, h_i^{(\ell)} \,\|\, h_j^{(\ell)} \,\|\, \psi(b_{ij}^{(\ell)}) \,] \in \mathbb{R}^{3H},
\end{equation}
where $\psi(\cdot)$ is a MLP, and then $u_{ij}$ is processed by an edge MLP
\begin{equation}
m_{ij}^{(\ell)} = \phi_e(u_{ij}^{(\ell)}) \in \mathbb{R}^{H}.
\end{equation}

\paragraph{Node aggregation and updates.}
Node-wise aggregated messages are computed as
\begin{equation}
\tilde{m}_i^{(\ell)} = \sum_{j=1}^{N} \tilde{\alpha}_{ij}^{(\ell)} m_{ij}^{(\ell)},
\end{equation}
and node states are updated via a MLP
\begin{equation}
h_i^{(\ell+1)} = h_i^{(\ell)} + \phi_n(\tilde{m}_i^{(\ell)}).
\end{equation}
Edge states are updated symmetrically using
\begin{equation}\label{eq:bij}
b_{ij}^{(\ell+1)} = b_{ij}^{(\ell)} + \phi_u(m_{ij}^{(\ell)}),
\end{equation}
where $\phi_n$ and $\phi_u$ are learnable MLPs with SiLU activations.

Unlike standard attention layers that aggregate only node features, our formulation explicitly maintains and updates pairwise edge states. The adjacency bias $A^{(\ell)}_t(i,j)$ allows the model to softly favor plausible bonds while retaining the flexibility to modify connectivity during flow integration, which is essential for preserving cycles and long-range structural motifs in molecular graphs.

\subsection{Distinction from Existing Graph Transformer Architectures} \label{subsec:distinction_GNN} 
The GNN transformer above differs architecturally and semantically from existing graph transformers used in diffusion- and flow-based graph generation. DiGress \cite{vignacdigress} and DeFoG \cite{qinmadeira2024defog} operate on discrete noisy graphs and parameterise categorical denoising distributions or CTMC-based transition rates over node/edge states. VFM/CatFlow \cite{eijkelboom2024variational} uses relaxed categorical states but predicts a categorical endpoint distribution from which its vector field is induced. In contrast, our architecture makes four design choices absent from these models:

\begin{itemize}
    \item \textbf{Joint continuous graph state and direct velocity prediction.}  Flowette operates on the continuously relaxed graph state $G_t=(A_t,X_t,F_t)$ and directly predicts the coupled velocity field $v_\theta(G_t,t)=(v_A,v_X,v_F)$, where adjacency, node attributes, and edge attributes evolve simultaneously. These outputs are tangent directions that are repeatedly evaluated and integrated by the ODE solver, rather than categorical clean-state logits, endpoint probabilities, or discrete transition rates. The network must therefore remain meaningful throughout the continuum of fractional graph states encountered during generation, not only at discrete valid graphs.
    

    \item \textbf{Explicit separation of evolving topology and edge attributes.}  Flowette treats soft connectivity $A_t(i,j)$ and edge attributes $F_t(i,j)$ as related but distinct evolving variables. The current soft adjacency is inserted directly into the attention logits, 
    as in Eq. \ref{eq:sij}, so the evolving topology continuously modulates information exchange. In parallel, a separate pairwise hidden state $b_{ij}^{(\ell)}$, initialised from $F_t$, participates in message construction and is updated at every layer (Eqs. \ref{eq:uij}–\ref{eq:bij}). Thus, $A_t$ controls where and how strongly nodes communicate, while the pairwise state represents what edge-specific information is exchanged. This separation is particularly relevant for molecular graphs because bond existence and bond type are strongly related but are not identical variables.
    

    \item \textbf{Time conditions the complete topology–attribute dynamics.}  The time embedding is injected into both node and pairwise representations, allowing the interaction between topology and attributes to change throughout transport. Near $t=0$, the field performs coarse transformations from a graphette-prior sample, whereas near $t=1$, it performs finer structural and attribute corrections. Time therefore does not merely serve as a denoising-step index; it conditions a single continuous vector field whose outputs $(v_A,v_X,v_F)$ form the right-hand side of the graph ODE in Eq. \ref{eq:fm_ode_vel}.
    

    \item \textbf{Coupled and graph-consistent infinitesimal updates.}  The final node states predict $v_X$, while the evolving pairwise states predict both $v_A$ and $v_F$. Consequently, changes in atom representations, bond attributes, and connectivity are inferred from a shared evolving representation rather than from unrelated classification heads. The pairwise outputs are symmetrised for undirected graphs, and the complete vector field is permutation equivariant: $v_A(PAP^T, PX,PFP^T,t)=P\,v_A(A,X,F,t)\,P^T,$ with the corresponding transformations for $v_X$ and $v_F$. Therefore, node-relabelling consistency is respected at every infinitesimal ODE update, rather than being imposed only during final graph decoding.

\end{itemize}

\begin{restatable}{proposition}{propvelequivariance}[Permutation equivariance of the velocity field]
\label{prop:vel_equivariance}
Let $v_\theta : (A_t, X_t, F_t, t) \;\mapsto\; \bigl(v_A,\; v_X,\; v_F\bigr)$ denote the time-conditioned GNN transformer described in Section~\ref{sec:methods}, with $L$ layers of edge-aware velocity attention, shared learnable linear maps and MLPs, and adjacency bias $A_t$ inside the attention logits. For any permutation $\pi \in \mathfrak{S}_n$ with permutation matrix $P_\pi$, define the permuted graph state $(A_t', X_t', F_t') = (P_\pi A_t P_\pi^\top,\; P_\pi X_t,\; P_\pi F_t P_\pi^\top).$ Then there exists a parameterization of $v_\theta$ such that $v_\theta(A_t', X_t', F_t', t) = \bigl(P_\pi v_A P_\pi^\top,\; P_\pi v_X,\; P_\pi v_F P_\pi^\top\bigr),$ i.e., the predicted adjacency, node, and edge velocities transform equivariantly under any node relabeling.
\end{restatable}
See Sec.~\ref{sec:predictor_proof} for the proof.

\section{Chemistry-Aware Regularization}\label{sec:chem_regularisation}
While the endpoint consistency loss $\mathcal{L}_{\mathrm{end}}$ provides an endpoint-aware temporal reweighting of the velocity residual, it does not explicitly enforce chemical feasibility. We therefore introduce differentiable, chemistry-aware regularizers applied exclusively to the predicted endpoint $\hat{G}_1 = (\hat{A}_1, \hat{X}_1, \hat{F}_1)$, encouraging the model to produce globally valid molecular graphs at generation time.

\emph{Soft Valence Constraint.}
Let $A^{w}
\;=\;
\Tilde{\sigma}\Big(\frac{1}{2}\Big(\hat{A}_1+\hat{A}_1^\top\Big)\Big) \in \mathbb{R}_+^{n\times n}$ be a symmetric, non-negative weighted adjacency matrix, where $\Tilde{\sigma}(\cdot)=\operatorname{clip}(\cdot,0,1)$ is applied entrywise.
Let $\hat{F}_1{(ij)}\in\mathbb{R}^{\Tilde{B}}$ be the bond-type logits for pair $(i,j)$ and define bond-type probabilities
\begin{equation}
p_{ij} \;=\; \mathrm{softmax}(\hat{F}_1{(ij)}) \in \Delta^{\Tilde{B}-1},
\end{equation}
where $\Delta^{\Tilde{B}-1} = \{p \in \mathbb{R}^{\Tilde{B}} \mid p \ge 0,\ \mathbf{1}^\top p = 1\}$. Let $b\in\mathbb{R}^{\Tilde{B}}$ be the vector of bond orders (e.g., single/double/triple/aromatic),
and define the expected bond order
\begin{equation}
\widehat{\mathrm{ord}}_{ij}
\;=\;
\langle p_{ij}, b\rangle
\;=\;
\sum_{k=1}^{\Tilde{B}} p_{ij,k}\, b_k .
\end{equation}
The expected valence of node $i$ is then
\begin{equation}
\widehat{\mathrm{val}}_{i}
\;=\;
\sum_{j=1}^{n} A^{w}_{ij}\,\widehat{\mathrm{ord}}_{ij}.
\end{equation}

Given the maximum allowable valence $\mathrm{val}_i^{\max}$ determined by the predicted atom type, we penalize soft violations via
\begin{equation}
\label{eq:loss_val}
\mathcal{L}_{\mathrm{val}}
=
\frac{1}{n}\sum_{i=1}^n
\max \bigl\{0,\widehat{\mathrm{val}}_i - \mathrm{val}_i^{\max}\bigr\}.
\end{equation}
This loss couples adjacency and bond-type predictions in a fully differentiable manner, discouraging chemically implausible bonding patterns without imposing hard constraints.

\emph{Atom-type Marginal Matching.}
To prevent global drift in node-type composition (e.g., overproduction of common atoms), we align the predicted atom-type marginal with that of the target graph at the endpoint. Let $d_a$ denote the number of atom-type categories, and let $X_{1,i,:d_a}$ denote the atom-type entries of the FGW-aligned target
representation of node $i$. Define the predicted and target marginals 
\begin{equation}
\hat{p}=\frac{1}{n}\sum_{i=1}^{n}\operatorname{softmax}\!\big(\hat{X}_{1,i,:d_a}\big),
\qquad
\bar{p}=\frac{\bar{x}}{\mathbf{1}^\top\bar{x}},\quad
\bar{x}=\frac{1}{n}\sum_{i=1}^{n}X_{1,i,:d_a},
\end{equation}
and penalise their discrepancy as a per-entry MSE,
\begin{equation}
\mathcal{L}_{\mathrm{atom}}=\frac{1}{d_a}\,\big\|\hat{p}-\bar{p}\big\|_2^2.
\end{equation}
Unlike per-node classification losses, this term is permutation-invariant and operates purely at the graph level, enforcing global semantic consistency of the node-type composition. As a result, it complements the endpoint-aware velocity reweighting and improves overall chemical validity.

\section{Proofs for Structure Preserving Flows (Section \ref{sec:methods})}
\label{sec:proofs}

\counterwithin{theorem}{section}

\subsection{Theoretical Properties of FGW-Coupling}
\label{sec:fgw_proofs}


\begin{restatable}{lemma}{lemmaperminvariance}[Permutation invariance of FGW-distance]
\label{thm:fgw_perm_invariance}
Let $G_0, G_1$ be two graphs as above, and let $\pi \in \mathfrak{S}_{n_0}$, $\tau \in \mathfrak{S}_{n_1}$ be arbitrary node permutations with permutation matrices $P_\pi$, $P_\tau$. Define permuted graphs $\widetilde{G}_0 = (\pi \cdot A_0,\ \pi \cdot X_0,\ \pi \cdot F_0)$ and $\widetilde{G}_1 = (\tau \cdot A_1,\ \tau \cdot X_1,\ \tau \cdot F_1)$. Then, (1) The FGW-distance is invariant: $\mathrm{FGW}_\alpha(G_0,G_1) = \mathrm{FGW}_\alpha(\widetilde{G}_0,\widetilde{G}_1)$; (2) If $T^\star$ is an optimal coupling for $\mathrm{FGW}_\alpha(G_0,G_1)$, then $\widetilde{T}^\star := P_\pi T^\star P_\tau^\top$ is an optimal coupling for $\mathrm{FGW}_\alpha(\widetilde{G}_0,\widetilde{G}_1)$. Conversely, any optimal coupling $\widetilde{T}^\star$ for the permuted problem corresponds to an optimal coupling $T^\star = P_\pi^\top \widetilde{T}^\star P_\tau$ for the original problem.
\end{restatable}
%
\begin{proof}
We first express how the costs transform under permutations.

For the appearance term, since $\pi$ only reorders rows of $X_0$ and
$\tau$ only reorders rows of $X_1$, we have
\[
M(\pi \cdot X_0,\ \tau \cdot X_1)
= P_\pi\, M(X_0,X_1)\, P_\tau^\top.
\]
Similarly, the intra-graph distance matrices satisfy
\[
C(\pi \cdot Z_0) = P_\pi\, C(Z_0)\, P_\pi^\top, \qquad
C(\tau \cdot Z_1) = P_\tau\, C(Z_1)\, P_\tau^\top.
\]
We now show that any feasible coupling $T$ for the original problem
corresponds to a feasible coupling $\widetilde{T}$ for the permuted
problem with the same objective value.

Let $T \in \Pi(p,q)$, so $T\mathbf{1}=p$ and $T^\top \mathbf{1} = q$.
Define $\widetilde{T} = P_\pi T P_\tau^\top$.
Since $P_\pi$ and $P_\tau$ are permutation matrices,
\[
\widetilde{T}\mathbf{1}
= P_\pi T P_\tau^\top \mathbf{1}
= P_\pi T \mathbf{1}
= P_\pi p
= p,
\]
because $p$ is uniform and invariant under permutation.
Analogously,
\[
\widetilde{T}^\top \mathbf{1}
= P_\tau T^\top P_\pi^\top \mathbf{1}
= P_\tau T^\top \mathbf{1}
= P_\tau q
= q.
\]
Non-negativity is preserved since $P_\pi$ and $P_\tau$ just reorder
entries. Thus $\widetilde{T} \in \Pi(p,q)$.

\paragraph{Feature cost.}
The linear appearance term for the permuted problem is
\[
\begin{aligned}
\langle \widetilde{T}, M(\pi \cdot X_0,\ \tau \cdot X_1)\rangle
&= \langle P_\pi T P_\tau^\top,\; P_\pi M(X_0,X_1) P_\tau^\top \rangle \\
&= \langle T,\ M(X_0,X_1)\rangle,
\end{aligned}
\]
where we used invariance of the Frobenius inner product under
simultaneous left- and right-multiplication by orthogonal matrices.

\paragraph{GW term.}
Let $C_0 = C(Z_0)$ and $C_1 = C(Z_1)$.
For the permuted graphs, the structural costs are
$\widetilde{C}_0 = P_\pi C_0 P_\pi^\top$ and
$\widetilde{C}_1 = P_\tau C_1 P_\tau^\top$.
The GW term of Eq.~\eqref{eq:app_fgw} for $(\widetilde{G}_0,\widetilde{G}_1)$
and coupling $\widetilde{T}$ is
\[
\begin{aligned}
\mathcal{L}_{\mathrm{GW}}(\widetilde{T};\widetilde{C}_0,\widetilde{C}_1)
&=
\sum_{i,k}\sum_{j,\ell}
\bigl(\widetilde{C}_0(i,k) - \widetilde{C}_1(j,\ell)\bigr)^2
\widetilde{T}_{ij} \widetilde{T}_{k\ell}.
\end{aligned}
\]
Using the definitions of $\widetilde{C}_0,\widetilde{C}_1,\widetilde{T}$
and reindexing by $i' = \pi^{-1}(i)$, $k' = \pi^{-1}(k)$,
$j' = \tau^{-1}(j)$, $\ell' = \tau^{-1}(\ell)$, we obtain
\[
\mathcal{L}_{\mathrm{GW}}(\widetilde{T};\widetilde{C}_0,\widetilde{C}_1)
=
\mathcal{L}_{\mathrm{GW}}(T;C_0,C_1).
\]
Indeed, each term in the sum is just a relabeling of the indices of the
corresponding term for $(T;C_0,C_1)$.

\paragraph{FGW Objective.}
Combining the two parts,
\[
(1-\alpha)\,\langle \widetilde{T}, M(\pi \cdot X_0,\ \tau \cdot X_1)\rangle
+
\alpha\,\mathcal{L}_{\mathrm{GW}}(\widetilde{T};\widetilde{C}_0,\widetilde{C}_1)
\]
equals
\[
(1-\alpha)\,\langle T, M(X_0,X_1)\rangle
+
\alpha\,\mathcal{L}_{\mathrm{GW}}(T;C_0,C_1).
\]
Thus every feasible $T$ for the original problem corresponds to a
feasible $\widetilde{T}$ for the permuted problem with identical
objective value, and conversely via
$T = P_\pi^\top \widetilde{T} P_\tau$.
Therefore the minimal values of the two optimization problems coincide,
proving (1), and optimal solutions map to each other, proving (2).
\end{proof}

As a useful special case, if two graphs are identical up to a node
relabeling, the FGW distance is zero and the optimal coupling is the
corresponding permutation coupling.


\begin{restatable}{corollary}{corfgwzeroisomorphic}[Zero FGW-distance for isomorphic graphs]
\label{cor:fgw_zero_isomorphic}
Assume $n_0 = n_1 = n$ and suppose there exists a permutation $\sigma \in \mathfrak{S}_n$ such that $x^0_i = x^1_{\sigma(i)}$ and $C(Z_0)_{ik} = C(Z_1)_{\sigma(i),\sigma(k)}$ for all $i,k$. Let $T_\sigma = \frac{1}{n} P_\sigma$. Then $T_\sigma$ is feasible, achieves objective value $0$ in Eq. \eqref{eq:app_fgw}, and hence $\mathrm{FGW}_\alpha(G_0,G_1) = 0$.
\end{restatable}
%
\begin{proof}
By construction $T_\sigma \in \Pi(p,q)$ with uniform node measures. Under the stated assumptions, $M(X_0,X_1)$ and $C(Z_0),C(Z_1)$ align exactly under the permutation $\sigma$, so each term in both the appearance cost and the GW term of Eq. \eqref{eq:app_fgw} vanishes when evaluated at $T_\sigma$. Non-negativity of the integrand implies that the minimal value of Eq. \eqref{eq:app_fgw} is zero; hence $T_\sigma$ is optimal and the distance is zero.
\end{proof}

\begin{restatable}{theorem}{theorembatchpairing}[Permutation-consistent FGW-coupling]
\label{thm:batch_pairing}
Let $\{G_0^i\}_{i=1}^B$ and $\{G_1^j\}_{j=1}^B$ be two collections of graphs. Form the FGW-distance matrix $D \in \mathbb{R}_{\ge 0}^{B \times B}$ with $D_{ij} := \mathrm{FGW}_\alpha(G_0^i, G_1^j)$ as in Eq.~\eqref{eq:batch_cost}. Suppose there exists a permutation $\pi^\star \in \mathfrak{S}_B$ such that: 
\begin{itemize}
    \item For each $i$, the pair $(G_0^i,G_1^{\pi^\star(i)})$ satisfies the conditions of Corollary~\ref{cor:fgw_zero_isomorphic}, hence $D_{i,\pi^\star(i)} = 0$.
    \item For every $i$ and every $j \neq \pi^\star(i)$ we have $D_{ij} \ge 0$, and the assignment $\pi^\star$ is the unique permutation achieving total cost zero: for any $\pi \neq \pi^\star$, $\sum_{i=1}^B D_{i,\pi(i)} > 0$.
\end{itemize} 
Then any solution of the Hungarian assignment problem
\begin{equation}
\label{eq:app_hungarian}
\widehat{\pi}
\;\in\;
\argmin_{\pi \in \mathfrak{S}_B}
\sum_{i=1}^B D_{i,\pi(i)}
\end{equation}
recovers $\pi^\star$, i.e., $\widehat{\pi} = \pi^\star$. More generally, if there are several zero-cost permutations, the set of minimizers of Eq. ~\eqref{eq:app_hungarian} coincides with the set of such zero-cost permutations.
\end{restatable}
\begin{proof}
By construction of $D$, all entries are non-negative because both the appearance and GW terms in Eq. \eqref{eq:app_fgw} are non-negative. Assumption (i) and Corollary~\ref{cor:fgw_zero_isomorphic} imply that $D_{i,\pi^\star(i)} = 0$ for each $i$, hence
\[
\sum_{i=1}^B D_{i,\pi^\star(i)} = 0.
\]
Consider any other permutation $\pi \neq \pi^\star$. By assumption (ii), the corresponding total cost is strictly positive: $\sum_{i=1}^B D_{i,\pi(i)} > 0$. Therefore $\pi^\star$ is the unique minimizer of Eq. \eqref{eq:app_hungarian}, and the Hungarian algorithm, which solves this linear assignment problem exactly, must return $\widehat{\pi} = \pi^\star$.

If multiple permutations achieve total cost zero, then any such permutation is a minimizer of Eq. ~\eqref{eq:app_hungarian}, and no permutation with positive total cost can be optimal. Thus the set of minimizers is precisely the set of zero-cost permutations.

This theorem shows that, whenever the FGW-distance separates the correct pairings from all others, our Hungarian coupling step recovers a batch-wise one-to-one matching that is consistent with feature- and structure-preserving node permutations.
\end{proof}

\subsection{Theoretical Analysis of the Velocity Field Predictor}
\label{sec:predictor_proof}

\propvelequivariance*
\begin{proof}
We prove the claim by induction on the $L$ attention layers, using the
fact that all linear maps have shared weights and that permutations act
by row/column reordering.

\textbf{Initialization.}
At time $t$, initial node and edge embeddings are
\[
h_i^{(0)} = \psi_x(X_t(i)) + c_t,
\qquad
b_{ij}^{(0)} = \psi_f\bigl([F_t(i,j), A_t(i,j)]\bigr) + c_t,
\]
where $\psi_x$ and $\psi_f$ are shared MLPs applied row-wise and
$c_t$ depends only on $t$.
For the permuted state $(A_t',X_t',F_t')$ we obtain
\[
h_i^{\prime (0)} = \psi_x(X_t' (i)) + c_t = \psi_x(X_t(\pi^{-1}(i))) + c_t = h_{\pi^{-1}(i)}^{(0)},
\]
so in matrix form $H^{\prime (0)} = P_\pi H^{(0)}$. Similarly,
\[
b_{ij}^{\prime (0)} = \psi_f([F_t' (i,j), A_t' (i,j)]) + c_t = \psi_f([F_t(\pi^{-1}(i),\pi^{-1}(j)), A_t(\pi^{-1}(i),\pi^{-1}(j))]) + c_t = b_{\pi^{-1}(i),\pi^{-1}(j)}^{(0)},
\]
i.e., $B^{\prime (0)} = P_\pi B^{(0)} P_\pi^\top$.

\paragraph{Inductive step.}
Assume that at layer $\ell$ we have $H^{(\ell)} \in \mathbb{R}^{N \times H}$, $B^{(\ell)} \in \mathbb{R}^{N \times N \times H}$ and
\[
H^{\prime (\ell)} = P_\pi H^{(\ell)}, \qquad
B^{\prime (\ell)} = P_\pi B^{(\ell)} P_\pi^\top.
\]

\textbf{Attention logits.} We compute
\[
Q^{(\ell)} = H^{(\ell)} W_Q^{(\ell)},
\qquad
K^{(\ell)} = H^{(\ell)} W_K^{(\ell)},
\]
with shared weight matrices $W_Q^{(\ell)}, W_K^{(\ell)}$. For the permuted state,
\[
Q^{\prime (\ell)} = H^{\prime (\ell)} W_Q^{(\ell)}
= P_\pi H^{(\ell)} W_Q^{(\ell)}
= P_\pi Q^{(\ell)},
\]
and similarly $K^{\prime (\ell)} = P_\pi K^{(\ell)}$. The attention logits are
\[
S^{(\ell)}_{ij} = \frac{\langle Q^{(\ell)}_i, K^{(\ell)}_j \rangle}{\sqrt{H}} + A_t(i,j),
\]
and
\[
S^{\prime (\ell)}_{ij} = \frac{\langle Q^{\prime (\ell)}_i, K^{\prime (\ell)}_j \rangle}{\sqrt{H}} + A_t'(i,j).
\]
Using $Q^{\prime (\ell)} = P_\pi Q^{(\ell)}$, $K^{\prime (\ell)} = P_\pi K^{(\ell)}$, and $A_t' = P_\pi A_t P_\pi^\top$, we obtain
\[
S^{\prime (\ell)} = P_\pi S^{(\ell)} P_\pi^\top.
\]
Row-wise softmax is equivariant under row permutations, so the attention weights satisfy
\[
\alpha^{\prime (\ell)} = P_\pi \alpha^{(\ell)} P_\pi^\top.
\]

\textbf{Edge messages.}
For each pair $(i,j)$, we compute
\[
u_{ij}^{(\ell)} =
\bigl[h_i^{(\ell)} \,\|\, h_j^{(\ell)}
      \,\|\, \psi(b_{ij}^{(\ell)})\bigr],
\qquad
m_{ij}^{(\ell)} = \phi_e(u_{ij}^{(\ell)}),
\]
where $\psi$ and $\phi_e$ are shared MLPs. By the inductive hypothesis and shared weights,
\[
u_{ij}^{\prime (\ell)} = u_{\pi^{-1}(i),\pi^{-1}(j)}^{(\ell)},
\quad m_{ij}^{\prime (\ell)} = m_{\pi^{-1}(i),\pi^{-1}(j)}^{(\ell)}.
\]

\textbf{Node aggregation.}
Node-wise aggregated messages are
\[
\widetilde{m}_i^{(\ell)} =
\sum_{j=1}^N \alpha_{ij}^{(\ell)} m_{ij}^{(\ell)}.
\]
Using the permutation relations for $\alpha^{(\ell)}$ and $m^{(\ell)}$, one checks that $\widetilde{m}_i^{\prime (\ell)} = \widetilde{m}_{\pi^{-1}(i)}^{(\ell)}$, i.e., $\widetilde{M}^{\prime (\ell)} = P_\pi \widetilde{M}^{(\ell)}$.

\textbf{Node and edge updates.}
Node states are updated via a shared MLP $\phi_n$:
\[
h_i^{(\ell+1)} = h_i^{(\ell)} + \phi_n(\widetilde{m}_i^{(\ell)}).
\]
Thus
\[
h_i^{\prime (\ell+1)} = h_{\pi^{-1}(i)}^{(\ell)} + \phi_n(\widetilde{m}_{\pi^{-1}(i)}^{(\ell)}) = h_{\pi^{-1}(i)}^{(\ell+1)},
\]
so $H^{\prime (\ell+1)} = P_\pi H^{(\ell+1)}$. Edge states are updated symmetrically via a shared MLP $\phi_u$:
\[
b_{ij}^{(\ell+1)} =
b_{ij}^{(\ell)} + \phi_u(m_{ij}^{(\ell)}),
\]
which yields $B^{\prime (\ell+1)} = P_\pi B^{(\ell+1)} P_\pi^\top$.

By induction, the equivariance relations hold for all $\ell = 0,\dots,L$.

\paragraph{Output velocities.}
Finally, the velocity heads are linear maps applied row-wise:
\[
v_X(i) = h_i^{(L)} W_{\mathrm{out}}^X,\quad
v_F(i,j) = b_{ij}^{(L)} W_{\mathrm{out}}^E,\quad
v_A(i,j) = w_A^\top b_{ij}^{(L)}. 
\]
Therefore, for the permuted state, we have
\[
v_X'(i) = v_X(\pi^{-1}(i)), \quad
v_F'(i,j) = v_F(\pi^{-1}(i),\pi^{-1}(j)), \quad
v_A'(i,j) = v_A(\pi^{-1}(i),\pi^{-1}(j)),
\]
or in matrix form
\[
v_X' = P_\pi v_X,\qquad
v_F' = P_\pi v_F P_\pi^\top,\qquad
v_A' = P_\pi v_A P_\pi^\top.
\]
This is precisely the desired permutation equivariance.
\end{proof}

Together, Lemma ~\ref{thm:fgw_perm_invariance} and Theorem \ref{thm:batch_pairing} and Proposition~\ref{prop:vel_equivariance} show that our FGW coupling, Hungarian pairing, and velocity field respect node permutations, ensuring that the overall flow-matching objective is well-defined on graphs up to isomorphism.

\subsection{Theoretical Properties of the Learning Objective}
\label{sec:learning_proofs}
\subsubsection{Derivation: $\mathcal{L}_{\mathrm{end}}$ as a Reweighted Velocity Residual} \label{subsec:lend_derivation}
We show explicitly that $\mathcal{L}_{\mathrm{end}}$ (Eq.~\eqref{eq:lend}) is not an independent supervision signal but an algebraically induced, non-uniform temporal reweighting of the velocity residual already penalised by $\mathcal{L}_{\mathrm{vel}}$.

Recall the rectified interpolant $A_t = (1-t)A_0 + tA_1$ and the ideal constant velocity target $\Delta A = A_1 - A_0$. Substituting the interpolant into the one-step endpoint prediction (Eq.~\eqref{eq:endpoint-pred}):
\begin{align}
\widehat{A}_1 &= A_t + (1-t)v_A \notag \\
&= \big[(1-t)A_0 + tA_1\big] + (1-t)v_A.
\end{align}
Subtracting the true target $A_1$:
\begin{align}
\widehat{A}_1 - A_1
&= (1-t)A_0 + tA_1 - A_1 + (1-t)v_A \notag \\
&= (1-t)A_0 - (1-t)A_1 + (1-t)v_A \notag \\
&= (1-t)\big(A_0 - A_1\big) + (1-t)v_A \notag \\
&= (1-t)\big(v_A - (A_1 - A_0)\big) \notag \\
&= (1-t)(v_A - \Delta A). \label{eq:lend-residual}
\end{align}
The same relation holds identically for $X$ and $F$. Squaring Eq.~\eqref{eq:lend-residual} gives
\begin{equation}
\|\widehat{A}_1 - A_1\|_F^2 = (1-t)^2 \|v_A - \Delta A\|_F^2,
\end{equation}
Since the same per-entry normalisations and component weights are used in Eqs.~\eqref{eq:app_vel_loss} and~\eqref{eq:lend}, the relation holds for the complete velocity and endpoint losses:
\begin{equation}
\mathcal{L}_{\mathrm{end}}(t)
=
(1-t)^2\mathcal{L}_{\mathrm{vel}}(t).
\label{eq:end-reweighted-vel}
\end{equation}
Consequently, their combined contribution to the training objective is
\begin{equation}
\mathcal{L}_{\mathrm{vel}}(t)
+
\beta_{\mathrm{end}}\mathcal{L}_{\mathrm{end}}(t)
=
\left[1+\beta_{\mathrm{end}}(1-t)^2\right]
\mathcal{L}_{\mathrm{vel}}(t).
\label{eq:combined-reweighted-vel}
\end{equation}
Thus, $\mathcal{L}_{\mathrm{end}}$ introduces no additional transport target; it acts as an explicit time-dependent reweighting of the same velocity residual.

The weight $(1-t)^2$ has a direct interpretation: $(1-t)$ is the remaining extrapolation horizon from time $t$ to $t=1$ in the one-step endpoint estimate. Consequently, the same velocity residual induces a larger predicted endpoint error when evaluated earlier in the trajectory ($t$ close to $0$) than later ($t$ close to $1$). Combining $\mathcal{L}_{\mathrm{vel}}$ with $\beta_{\mathrm{end}}\mathcal{L}_{\mathrm{end}}$ therefore provides an endpoint-aware temporal reweighting that places greater emphasis on velocity residuals with a larger effect on the predicted terminal state, without introducing any new transport target beyond $\Delta A$, $\Delta X$, and $\Delta F$. Under the assumptions of Theorem~\ref{thm:endpoint_consistency}, driving the velocity residual to zero also drives both $\mathcal{L}_{\mathrm{vel}}$ and $\mathcal{L}_{\mathrm{end}}$ to zero, since the two losses differ only by the multiplicative factor $(1-t)^2$.

\begin{restatable}{theorem}{theormendpointconsistency}
[Endpoint Consistency of Rectified Transport]
\label{thm:endpoint_consistency}
Fix an FGW-coupled pair of graphs $(A_0,X_0,F_0)$ and $(A_1,X_1,F_1)$, and let $(A_t,X_t,F_t)_{t\in[0,1]}$ denote the corresponding rectified path. Assume $\lambda_x,\lambda_e>0$ and that $t\mapsto v_\theta(A_t,X_t,F_t,t)$ is continuous along this path. If $\mathbb{E}_{t\sim\mathcal{U}[0,1]}\big[\mathcal{L}_{\mathrm{vel}}(t)\big]=0$, then
\begin{equation}
v_\theta(A_t,X_t,F_t,t)
=
(\Delta A,\Delta X,\Delta F),
\qquad
\forall\,t\in[0,1].
\end{equation}
Consequently, the rectified path satisfies
\begin{equation}
\frac{d}{dt}(A_t,X_t,F_t)
=
v_\theta(A_t,X_t,F_t,t),
\qquad
(A_0,X_0,F_0)\ \text{given},
\end{equation}
and reaches the target endpoint
\begin{equation}
(A_1,X_1,F_1)
=
(A_0+\Delta A,\,
 X_0+\Delta X,\,
 F_0+\Delta F).
\end{equation}
Moreover, for any $K\in\mathbb{N}$, $K$-step explicit Euler integration with step size $h=1/K$, initialized at $(A_0,X_0,F_0)$, recovers the exact endpoint $(A_1,X_1,F_1)$.
\end{restatable}

\begin{proof}
We prove the result in two steps.

\paragraph{Step 1: Identification of the velocity field.}
Let
\[
s_t=(A_t,X_t,F_t),
\qquad
\Delta=(\Delta A,\Delta X,\Delta F),
\]
where, by definition of the rectified path,
\[
A_t=(1-t)A_0+tA_1,\qquad
X_t=(1-t)X_0+tX_1,\qquad
F_t=(1-t)F_0+tF_1.
\]
Hence
\[
\Delta A=A_1-A_0,\qquad
\Delta X=X_1-X_0,\qquad
\Delta F=F_1-F_0.
\]

By Eq.~\eqref{eq:app_vel_loss}, the velocity-matching loss at time $t$ is
\begin{equation}
\begin{split}
\mathcal{L}_{\mathrm{vel}}(t)
&=
\frac{1}{N_A}
\bigl\|
v_A(A_t,X_t,F_t,t)-\Delta A
\bigr\|_F^2
+
\lambda_x\frac{1}{N_X}
\bigl\|
v_X(A_t,X_t,F_t,t)-\Delta X
\bigr\|_F^2
\\
&+ 
\lambda_e\frac{1}{N_F}
\bigl\|
v_F(A_t,X_t,F_t,t)-\Delta F
\bigr\|_F^2 .
\end{split}
\label{eq:endpoint-proof-vel}
\end{equation}
Since $N_A,N_X,N_F>0$ and, by assumption, $\lambda_x,\lambda_e>0$, every term in Eq.~\eqref{eq:endpoint-proof-vel} is nonnegative and has a strictly positive coefficient.

Moreover, by assumption, $t\mapsto v_\theta(A_t,X_t,F_t,t)$ is continuous along the rectified path. Since the squared Frobenius norm is continuous, $\mathcal{L}_{\mathrm{vel}}(t)$ is therefore a continuous, nonnegative function on $[0,1]$.

For $t\sim\mathcal{U}[0,1]$, the hypothesis
\[
\mathbb{E}_{t\sim\mathcal{U}[0,1]}
\bigl[\mathcal{L}_{\mathrm{vel}}(t)\bigr]=0
\]
is equivalent to
\[
\int_0^1 \mathcal{L}_{\mathrm{vel}}(t)\,dt=0.
\]
A continuous nonnegative function on $[0,1]$ can have zero integral only if it vanishes everywhere. Hence
\begin{equation}
\mathcal{L}_{\mathrm{vel}}(t)=0,
\qquad
\forall\,t\in[0,1].
\label{eq:endpoint-proof-zero}
\end{equation}

Because all coefficients in Eq.~\eqref{eq:endpoint-proof-vel} are strictly positive, Eq.~\eqref{eq:endpoint-proof-zero} implies that each squared residual vanishes individually. Therefore, for every $t\in[0,1]$,
\begin{equation}
v_A(A_t,X_t,F_t,t)=\Delta A,
\qquad
v_X(A_t,X_t,F_t,t)=\Delta X,
\qquad
v_F(A_t,X_t,F_t,t)=\Delta F.
\label{eq:endpoint-proof-velocity}
\end{equation}
Thus, along the entire rectified path, the learned velocity field coincides with the ideal constant transport $(\Delta A,\Delta X,\Delta F)$.

\paragraph{Step 2: Exact reconstruction by the ODE and Euler integration.}
By construction of the rectified path,
\begin{equation}
\frac{dA_t}{dt}=\Delta A,\qquad
\frac{dX_t}{dt}=\Delta X,\qquad
\frac{dF_t}{dt}=\Delta F.
\end{equation}
Combining these identities with
Eq.~\eqref{eq:endpoint-proof-velocity} gives
\[
\frac{d}{dt}(A_t,X_t,F_t)
=
v_\theta(A_t,X_t,F_t,t)
\]
for every $t\in[0,1]$. Hence the rectified path satisfies the learned ODE with initial condition $(A_0,X_0,F_0)$.

Integrating from $t=0$ to $t=1$ yields
\begin{align}
A_1
&=
A_0+\int_0^1 \Delta A\,dt
=
A_0+\Delta A,\\
X_1
&=
X_0+\int_0^1 \Delta X\,dt
=
X_0+\Delta X,\\
F_1
&=
F_0+\int_0^1 \Delta F\,dt
=
F_0+\Delta F.
\end{align}
Thus the rectified ODE trajectory reaches exactly the target endpoint $(A_1,X_1,F_1)$.

It remains to show that explicit Euler integration is also exact. Let
\[
t_k=\frac{k}{K},
\qquad
h=\frac{1}{K},
\qquad
k=0,\ldots,K,
\]
and initialize
\[
(A^{(0)},X^{(0)},F^{(0)})
=
(A_0,X_0,F_0).
\]
The Euler updates are
\begin{align}
A^{(k+1)}
&=
A^{(k)}
+
h\,v_A(A^{(k)},X^{(k)},F^{(k)},t_k),\\
X^{(k+1)}
&=
X^{(k)}
+
h\,v_X(A^{(k)},X^{(k)},F^{(k)},t_k),\\
F^{(k+1)}
&=
F^{(k)}
+
h\,v_F(A^{(k)},X^{(k)},F^{(k)},t_k).
\end{align}

We show by induction that the Euler iterate coincides with the rectified state at every grid point:
\begin{equation}
(A^{(k)},X^{(k)},F^{(k)})
=
(A_{t_k},X_{t_k},F_{t_k}),
\qquad
k=0,\ldots,K.
\label{eq:euler-induction}
\end{equation}

For $k=0$, Eq.~\eqref{eq:euler-induction} holds by initialization. Suppose it holds for some $k<K$. Then
\[
(A^{(k)},X^{(k)},F^{(k)})
=
(A_{t_k},X_{t_k},F_{t_k}),
\]
and therefore, by Eq.~\eqref{eq:endpoint-proof-velocity},
\[
v_\theta(A^{(k)},X^{(k)},F^{(k)},t_k)
=
(\Delta A,\Delta X,\Delta F).
\]
The Euler update consequently gives
\begin{align}
A^{(k+1)}
&=
A_{t_k}+h\Delta A
=
A_0+(t_k+h)\Delta A
=
A_{t_{k+1}},\\
X^{(k+1)}
&=
X_{t_k}+h\Delta X
=
X_0+(t_k+h)\Delta X
=
X_{t_{k+1}},\\
F^{(k+1)}
&=
F_{t_k}+h\Delta F
=
F_0+(t_k+h)\Delta F
=
F_{t_{k+1}},
\end{align}
where $t_{k+1}=t_k+h$. Hence Eq.~\eqref{eq:euler-induction} holds at step $k+1$, completing the induction.

Taking $k=K$ gives $t_K=1$, and therefore
\[
(A^{(K)},X^{(K)},F^{(K)})
=
(A_{t_K},X_{t_K},F_{t_K})
=
(A_1,X_1,F_1).
\]
Thus $K$-step explicit Euler integration with $h=1/K$ recovers the exact endpoint for every $K\in\mathbb{N}$.
\end{proof}

By Eq.~\eqref{eq:end-reweighted-vel}, the conclusion also implies $\mathcal{L}_{\mathrm{end}}(t)=0$ for every $t\in[0,1]$; thus the endpoint term introduces no additional zero-loss condition beyond $\mathcal{L}_{\mathrm{vel}}$.

\begin{restatable}{theorem}{theoremstability}[Finite-Step Stability]
\label{thm:stability_app}
Let $s_t=(A_t,X_t,F_t)$ denote the rectified path between $s_0$ and $s_1$, and let $v^\star(s_t,t)=(\Delta A,\Delta X,\Delta F)$ be the ideal constant velocity. Suppose $v_\theta$ satisfies
\begin{equation}
\label{eq:eps_bound}
\|v_\theta(s_t,t)-v^\star(s_t,t)\|_F \le \varepsilon
\quad\text{for all }t\in[0,1],
\end{equation}
and is $L$-Lipschitz in $s$, i.e., $\|v_\theta(s,t)-v_\theta(s',t)\|_F \le L\|s-s'\|_F, \forall s,s',\ \forall t\in[0,1]$. Let $s^{(K)}$ denote the endpoint obtained by $K$-step explicit Euler integration of $v_\theta$ from $s_0$ with step size $h=1/K$. Then there exists a constant $C>0$, depending only on $L$, such that
\begin{equation}
\|s^{(K)}-s_1\|_F
\le
C\varepsilon.
\end{equation}
\end{restatable}

\begin{proof}
Let $s^\star_t$ denote the exact rectified path driven by $v^\star$, i.e., $s^\star_t=s_0+t(s_1-s_0)$, so that $s^\star_1=s_1$ and $d s^\star_t/dt=v^\star(s^\star_t,t)$. Let $s^{(k)}$ be the Euler iterates
with step size $h=1/K$, starting at $s^{(0)}=s_0$:
\begin{equation}
s^{(k+1)}
=
s^{(k)}
+
h\,v_\theta(s^{(k)},t_k),
\qquad
t_k=kh.
\end{equation}
Define the error at step $k$ as
\begin{equation}
e_k:=s^{(k)}-s^\star_{t_k}.
\end{equation}
We derive a recursion for $e_k$. Using the definition of $s^{(k+1)}$ and the exact flow of $s^\star_t$,
\begin{align}
e_{k+1}
&=
s^{(k+1)}-s^\star_{t_{k+1}}
\\
&=
s^{(k)}
+
h\,v_\theta(s^{(k)},t_k)
-
\left(
s^\star_{t_k}
+
\int_{t_k}^{t_{k+1}}
v^\star(s^\star_u,u)\,du
\right)
\\
&=
e_k
+
h\bigl(
v_\theta(s^{(k)},t_k)
-
v^\star(s^\star_{t_k},t_k)
\bigr)
\nonumber\\
&\qquad
-
\int_{t_k}^{t_{k+1}}
\bigl(
v^\star(s^\star_u,u)
-
v^\star(s^\star_{t_k},t_k)
\bigr)\,du.
\end{align}
Taking Frobenius norms and applying the triangle inequality,
\begin{align}
\|e_{k+1}\|_F
&\le
\|e_k\|_F
+
h\,
\|v_\theta(s^{(k)},t_k)
-
v^\star(s^\star_{t_k},t_k)\|_F
\nonumber\\
&\qquad
+
\int_{t_k}^{t_{k+1}}
\|v^\star(s^\star_u,u)
-
v^\star(s^\star_{t_k},t_k)\|_F\,du.
\label{eq:ek_recursion}
\end{align}

\paragraph{Velocity mismatch term.}
Add and subtract $v_\theta(s^\star_{t_k},t_k)$, and use the Lipschitz property and Eq.~\eqref{eq:eps_bound}:
\begin{align}
&\|v_\theta(s^{(k)},t_k)
-
v^\star(s^\star_{t_k},t_k)\|_F
\nonumber\\
&\qquad\le
\|v_\theta(s^{(k)},t_k)
-
v_\theta(s^\star_{t_k},t_k)\|_F
+
\|v_\theta(s^\star_{t_k},t_k)
-
v^\star(s^\star_{t_k},t_k)\|_F
\\
&\qquad\le
L\|s^{(k)}-s^\star_{t_k}\|_F+\varepsilon
=
L\|e_k\|_F+\varepsilon.
\end{align}

\paragraph{Local truncation term.}
Since $v^\star$ is constant along the rectified path (indeed, $v^\star(s^\star_u,u)=\Delta s:=s_1-s_0$ for all $u$), we have
\begin{equation}
v^\star(s^\star_u,u)
-
v^\star(s^\star_{t_k},t_k)
=
0
\quad
\text{for all }
u\in[t_k,t_{k+1}],
\end{equation}
and therefore the integral term in
Eq.~\eqref{eq:ek_recursion} vanishes exactly.

\paragraph{Putting it together.}
Substituting the above bounds into Eq.~\eqref{eq:ek_recursion} yields
\begin{equation}
\|e_{k+1}\|_F
\le
\|e_k\|_F
+
h\bigl(L\|e_k\|_F+\varepsilon\bigr)
=
(1+Lh)\|e_k\|_F+h\varepsilon.
\end{equation}
Since $e_0=s^{(0)}-s^\star_0=0$, for $L>0$ we obtain
\begin{align}
\|e_K\|_F
&\le
h\varepsilon
\sum_{j=0}^{K-1}(1+Lh)^j
\\
&=
h\varepsilon
\frac{(1+Lh)^K-1}{Lh}
\\
&=
\frac{\varepsilon}{L}
\bigl((1+Lh)^K-1\bigr).
\end{align}
Since $K=1/h$ and
$(1+Lh)^{1/h}\le e^L$ for $h\in(0,1]$,
\begin{equation}
(1+Lh)^K
=
(1+Lh)^{1/h}
\le
e^L,
\end{equation}
and hence
\begin{equation}
\|e_K\|_F
\le
\frac{e^L-1}{L}\,\varepsilon.
\end{equation}
For $L=0$, the recursion reduces to
\[
\|e_{k+1}\|_F
\le
\|e_k\|_F+h\varepsilon,
\]
which gives
\[
\|e_K\|_F
\le
Kh\varepsilon
=
\varepsilon.
\]
Thus, with
\[
C=
\begin{cases}
\dfrac{e^L-1}{L}, & L>0,\\[6pt]
1, & L=0,
\end{cases}
\]
we obtain
\begin{equation}
\|s^{(K)}-s_1\|_F
=
\|e_K\|_F
\le
C\varepsilon,
\end{equation}
as claimed.
\end{proof}

Theorem~\ref{thm:stability_app} formalizes how a uniform velocity approximation error controls the endpoint error under finite-step integration. We now relate this raw Frobenius error to the per-entry MSE used in our training objective. For $\lambda_x,\lambda_e>0$, let
\begin{equation}
N_{\max}
=
\max\left\{
N_A,\,
\frac{N_X}{\lambda_x},\,
\frac{N_F}{\lambda_e}
\right\}.
\end{equation}
From Eq.~\eqref{eq:app_vel_loss},
\begin{equation}
\|v_\theta(s_t,t)-v^\star(s_t,t)\|_F^2
\le
N_{\max}\,
\mathcal{L}_{\mathrm{vel}}(t).
\end{equation}
Hence, if $\mathcal{L}_{\mathrm{vel}}(t)\le\delta$ for all $t\in[0,1]$, then Eq.~\eqref{eq:eps_bound} holds with \begin{equation} \varepsilon=\sqrt{N_{\max}\,\delta}, \end{equation} and Theorem~\ref{thm:stability_app} gives $\|s^{(K)}-s_1\|_F\le C\sqrt{N_{\max}\delta}$.
As shown in Sec.~\ref{subsec:lend_derivation}, $\mathcal{L}_{\mathrm{end}}$ provides an endpoint-aware $(1-t)^2$ temporal reweighting of the same velocity residual rather than an independent error signal.

\section{Computational Complexity Analysis}
\label{sec:complexity}
Let $B$ be the batch size, $n$ the graph size used for flow matching, $m$ the number of edges $|E|$ in the graph, $L$ the number of transformer layers, $H$ the hidden dimension, and $K$ the number of Euler steps. 

\textbf{FGW coupling.} Computing the FGW cost matrix between $B$ source and $B$ target graphs in a single minibatch requires $B^2$ FGW solves. Using a square-loss FGW solver with $I$ iterations, this costs $O(B^2 I n^3)$, plus $O(B^2 m (d_x+d_z))$ to form feature and structure costs; Hungarian matching adds $O(B^3)$. Under the default \emph{online} protocol (Sec.~\ref{sec:fgw_fm}), this cost is paid at every minibatch of every epoch. Under the \emph{fixed-pool} protocol, it is instead a one-time preprocessing cost, computed once in $B \times B$ blocks over the training set and reused across all subsequent epochs; \emph{periodic} coupling amortises this cost by recomputing it only every $\Tilde{\mathcal{R}}$ epochs. All three variants are embarrassingly parallel across blocks. Sec.~\ref{sec:scalability} reports empirical runtimes for all three protocols.

\textbf{Flow-matching training.} After FGW precomputation, each minibatch contains $B$ matched graph pairs. Edge-aware flow matching costs $O(LmH^2 + m d_f)$ per pair, yielding a per-minibatch cost $O(B(LmH^2 + m d_f))$. Training over $\mathcal{R}$ epochs and $M$ minibatches therefore costs $O(\mathcal{R}M\,B(LmH^2 + m d_f))$. 

\textbf{Sampling.} Euler integration for $K$ steps requires $K$ forward passes of the transformer, with total cost $O(Km(LH^2 + d_f))$. 

\section{Generative stories}\label{appendix:generativeStories}

\subsection{Generative story for graphons, sparsified graphons and graphettes }

\begin{algorithm}[H]
\caption{Sampling from a graphon $W$}
\label{alg:genStoryGraphon}
\begin{algorithmic}[1]
\INPUT
$n$, $W$

\OUTPUT Adjacency matrix $A$ of sampled graph.

\STATE Sample $x_1,\dots,x_n \overset{iid}{\sim}\mathrm{Unif}[0,1] $

\FOR {$i \in \{1, \ldots, n\}$}
    \FOR {$j \in \{(i+1), \ldots, n\}$}
  \STATE Sample $A_{ij} \sim \textrm{Bernoulli}( W(x_i, x_j))$\;
  \STATE Set $A_{ji} = A_{ij}$
\ENDFOR
\ENDFOR
\RETURN{$A$}\;
\end{algorithmic}
\end{algorithm}

\begin{algorithm}[H]
\caption{Sampling from a sparsified graphon $W_n = \rho_n W$}
\label{alg:genStorySparsifiedGraphon}
\begin{algorithmic}[1]
\INPUT
$n$, $W$, sequence $\{\rho_i\}_i = (\rho_1, \rho_2, \ldots, \rho_n, \ldots)$

\OUTPUT Adjacency matrix $A$ of sampled graph.

\STATE Get the $n$th value of  $\{\rho_i\}_i$ denoted by $\rho_n$
\STATE Sample $x_1,\dots,x_n \overset{iid}{\sim}\mathrm{Unif}[0,1] $

\FOR {$i \in \{1, \ldots, n\}$}
    \FOR {$j \in \{(i+1), \ldots, n\}$}
  \STATE Sample $A_{ij} \sim \textrm{Bernoulli}(\rho_n W(x_i, x_j))$\;
  \STATE Set $A_{ji} = A_{ij}$
\ENDFOR
\ENDFOR
\RETURN{$A$}\;
\end{algorithmic}
\end{algorithm}


\begin{algorithm}[H]
\caption{Sampling from a graphette $\mathcal{W} = (W, \rho_n, f)$}
\label{alg:genStorygraphette}
\begin{algorithmic}[1]
\INPUT
$n$, $W$, sequence $\{\rho_i\}_i = (\rho_1, \rho_2, \ldots, \rho_n, \ldots)$, graph edit function $f(G, \cdots)$ where $\cdots$ denote other function specific parameters

\OUTPUT Adjacency matrix $A$ of sampled graph.

\COMMENTS{ \scriptsize First sample a graph from the sparsified graphon}
\STATE Get the $n$th value of  $\{\rho_i\}_i$ denoted by $\rho_n$
\STATE Sample $x_1,\dots,x_n \overset{iid}{\sim}\mathrm{Unif}[0,1] $

\FOR {$i \in \{1, \ldots, n\}$}
    \FOR {$j \in \{(i+1), \ldots, n\}$}
  \STATE Sample $A'_{ij} \sim \textrm{Bernoulli}(\rho_n W(x_i, x_j))$\;
  \STATE Set $A'_{ji} = A'_{ij}$
\ENDFOR
\ENDFOR
\STATE Get the graph $G'$ from the adjacency matrix $A'$;

\COMMENTS{ \scriptsize Then enact the graph edit function on it}
\STATE $G = f(G',\cdots)$;
\STATE Get the adjacency matrix $A$ of $G$;
\RETURN{$A$}\;
\end{algorithmic}
\end{algorithm}
\section{Graph edit functions and theorem proofs in Section \ref{sec:Graphettes}}\label{Appendix:graphettes}

\subsection{Graph edit functions GEF 1 - GEF 4}

We use the following four graph edit functions in our experiments. 

\begin{motif}\label{motif:identity} \textbf{Identity}. Let $f = I$ be the identity function with $I(G) = G$, which does not edit the graph. 
\end{motif}
The identity graph edit function is used in Lemmas \ref{lemma:Wfamily} and \ref{lemma:sparsifiedW} to show that graphettes can recover graphons and sparsified graphons. Furthermore, for stochastic block models, as illustrated in Section \ref{sec:resultsSynthetic}, 
 graphettes are equivalent to graphons and our Flowette model achieves
competitive performance, including state-of-the-art results on several metrics across multiple benchmarks.

\begin{motif}\label{motif:tree}\textbf{Cycle deletion}. 
Let $f = h$ denote a global function that removes cycles from the graph.
\end{motif}
When the application requires that graphs are trees (no cycles), we choose $f = h$. For a connected graph the resulting graph is a tree. 

The class of graphettes also includes graph edit functions which may be parametrised, for example we may wish to model different ring sizes in molecular graphs in Sec. ~\ref{sec:expt-molecules}. 

\begin{motif}\label{motif:ring3}\textbf{Ring addition}.
Let $R(p,c)$ denote a function where a ring of size $c$ is added and connected to a node $i$ with probability $p$, i.e.,  $R(p,c) = \mathrm{Bernoulli}(p) \times c$.
\end{motif}

For social network graphs, we model hubs, which are highly connected individuals. 

\begin{motif}\label{motif:star}\textbf{Star addition}. Let $S\left(a, b, \{u_i\}_{i=1}^n \right)$  denote a function that adds stars given by a Poisson random variable $\Poi(an \exp(u_i + b))$  to a node $i$ where $ \{u_i\}_{i=1}^n = (u_1, \ldots, u_n)$, $0 \leq u_i \leq 1$, $a >0$ and $|b| < 1$.   
\end{motif}

This models a non-uniform probability of any particular node $i$ (indexed by $u_i$) , having a star added, with star sizes varying with $u_i$. 

\subsection{Theorem proofs in Section \ref{sec:Graphettes} }
\begin{lemma}\label{lemma:recoverW}
The graphette $\mathcal{W}=(W, \rho_n, f)$  recovers $W$ when $\rho_n = 1$ and $f$ is the identity function. 
\end{lemma}
\begin{proof}
    This results in $G$ being sampled from $W_n = W$ and as $f$ is the identity, no edits are done. For  increasing  $n$ the resulting graph sequence $\{G_n\}_n$  is dense. 
\end{proof}

\begin{lemma}\label{lemma:Wfamily}The graphette $\mathcal{W} =(W, \rho_n, f)$  generates dense graphs from a family of graphons $\{W_n\}_n$  when $\rho_n \to c > 0$ and $f$ is the identity function. 
\end{lemma}
\begin{proof}
    When $\rho_n \to c >0 $, the graphon family $\{W_n\}_n$ where $W_n = \rho_n W \not\to 0$. As graphs sampled from any non-zero graphon are dense we have the result. 
\end{proof}

\begin{lemma}\label{lemma:sparsifiedW} The graphette $\mathcal{W} =(W, \rho_n, f)$  recovers the sparsified graphon $W_n$ when $\rho_n \to 0$ and $f$ is the identity function.
\end{lemma}
\begin{proof}
    When $\rho_n \to 0$ and $f$ is the identity, $\mathcal{W}$ generates graphs from $W_n = \rho_n W$ and $W_n \to 0$ as in the sparsified graphon setting. 
\end{proof}

\begin{lemma}\label{lemma:graphex}The graphette $\mathcal{W} =(W, \rho_n, f)$  can generate graphs with stars and isolated edges similar to the graphex construction. 
\end{lemma}
\begin{proof}
    Let $\rho_n = 1$. As illustrated in \cite{veitch2019sampling}, a star function such as $S(x,n)  = \Poi(n \exp(x + 0.2))$, where $S(x,n)$ edges are added to the node at latent position $x$ can generate stars.  An isolated edge function such as $I(\ell,n) = \Poi(n \ell)$ can add isolated edges at a rate $\ell$ to the graph. By letting $f = I \circ S$, both stars and isolated edges can be added.  
\end{proof}

\begin{lemma}\label{lemma:sparsestars1} The graphette $\mathcal{W} =(W, \rho_n, f)$  with $\rho_n = 1$  and  an appropriate  star function can produce sparse graphs.    
\end{lemma}
\begin{proof}
    Consider a simple deterministic star function $S(x,n)  = \lfloor{\alpha n} \rfloor$, where $\alpha >0$.  As $\rho_n = 1$, the graph is first sampled from $W$. Note $\{G'_n\}_n$ sampled from $W$ is dense. The star function will add $\sum_{i = 1}^n \lfloor{\alpha n} \rfloor \approx \alpha n^2 $ new nodes and edges to $G_n$.  After star addition, the edge density of $G_n$
    \begin{equation*}
        \text{density}(G_n) \approx \frac{m + \alpha n^2 }{\left( n + \alpha n^2 \right)^2} \to 0 \, , 
    \end{equation*}
making the sequence sparse, where $m$ denotes the edges in $G'_n$ before enacting the star function. For stochastic star functions such as $S(x,n) = \Poi(\alpha n)$ or $S(x,n) = \Poi(\alpha n\exp(bx - c))$ the result holds with high probability for $\alpha, b >0$ and can be shown using  concentration inequalities. 
\end{proof}

\begin{lemma}\label{lemma:ringdense} The graphette $\mathcal{W} =(W, \rho_n, f)$  with $\rho_n = 1$  and  ring function $R(p, c)$   (GEF \ref{motif:ring3}) produces dense graphs with high probability.
\end{lemma}
\begin{proof}
    The graph sequence $\{G'_n\}_n$ is sampled from $W_n = W$ and is dense. Suppose $G'_n$ has $n$ nodes and $m$ edges. Then the expectation of the number of nodes to be replaced by rings is $np$. These nodes will be replaced by $npc$ nodes resulting in an increase in node count from $n$ to $n + np(c-1)$ in expectation. As a ring of size $c$ has $c$ edges, the edge count increases from $m$ to $m + npc$ in expectation. Using the first order Taylor approximation of edge density we obtain 
    \[
    \mathbb{E}\left( \text{density}(G_n) \right)  \approx \frac{m + npc}{\frac{1}{2}(n + np(c-1))(n + np(c-1) -1 )}  \, .  \]
    As graphs sampled from $W$ are dense, $2m/n(n-1) \to \rho \neq 0$. This implies that $\mathbb{E}\left( \text{density}(G_n) \right)$ converges to a non-zero constant. In expectation the graphs $\{G_n\}_n$ are dense. Using concentration bounds we can show that the probability of the number of nodes and edges in $G_n$ being away from their expectations is small. Thus, we can show that $P(\text{density}(G_n) < \mathbb{E}\left( \text{density}(G_n) \right) - \epsilon)$ is small giving us the result. 
\end{proof}

\begin{lemma}\label{lemma:ringsparse} The graphette $\mathcal{W} =(W, \rho_n, f)$  with $\rho_n \to 0$ and ring function $R(x_i, p, c)$ (graph edit function \ref{sec:graphEditFunctions}.\ref{motif:ring3}) with $\mathcal{C} =  \{c\}$ produces sparse graphs for $c \geq 3$.
\end{lemma}
\begin{proof}
Consider a connected $G'_n$. Suppose $G'_n$ has $n$ nodes and $m$ edges and suppose a graph edit replaces $k$ nodes in a graph $G'_n$ with rings of size $c$ resulting in graph $G_n$. Then $G_n$, the edited graph has $n - k + kc$ nodes and $m + kc$ edges. Thus, the new edge density is
\[
 \text{density}(G_n) = \frac{m + kc}{\frac{1}{2}(n + k(c-1))(n + k(c-1)-1)} \, .
\]
By computing $\text{density}(G'_n) - \text{density}(G_n)$ we can show that
\[
\text{density}(G_n) \leq \text{density}(G'_n)
\]
for $c \geq 3$ showing that the edge density decreases by the addition of rings of size $c$. As graphs sampled from $W_n = \rho_n W$ are sparse we obtain the result. 
\end{proof}

\thmtrianglecovered*
\begin{proof}
    As $G' \subseteq G$ we have $\Hom(F, G') \subseteq \Hom(F, G)$. Next we show that $\Hom(F, G) \subseteq \Hom(F, G')$. Let $h \in \Hom(F, G)$. 
    
    As $F \in \mathcal{F}$, for every $u \in V(F)$ there exists $\{v, w\} \in V(F)$ such that $\{uv, vw, wu \} \in E(F)$.  As $h:F \to G$ is a homomorphism, we have $\{h(u)h(v), h(v)h(w), h(w)h(u) \} \in E(G)$. This forms a triangle in $G$. However, the star function $S$ only adds degree-1 vertices to $G'$, resulting in no new triangles.  Similarly, for the ring function $R$ with $c > 3$, no new triangles are added. 
    Therefore, this triangle is present in $G'$ making $h \in \Hom(F, G')$. Thus, $\Hom(F, G) \subseteq \Hom(F, G')$ making $\Hom(F, G) = \Hom(F, G')$ and $\hom(F, G') = \hom(F, G)$. As 
    \[
     t(F, G) = \frac{\hom(F, G)}{(n + m)^{|V(F)|}} \, \,   \text{and}  \, \,   t(F, G') = \frac{\hom(F, G)}{n ^{|V(F)|}}
    \]
    we get the result.       
\end{proof}

\begin{theorem}\label{thm:staraddition} Let $\mathcal{W} = (W, \rho_n, f)$ denote a graphette and let $G' \sim \mathsf{G}(n, W, \rho_n)$ and $G \sim f(G')$ where $f = S$ denotes the star function (GEF \ref{motif:star}). Suppose $|V(G')| = n$ and $|V(G)| = n + m$. Then we have the following:
\begin{enumerate}
    \item $\hom(\vertex, G) =  \hom(\vertex, G') + m$ and $t(\vertex, G) = 1$
    \item $\hom(\edge, G) =  \hom(\edge, G') + 2m$ and $t(\edge, G) =  \frac{2(|E(G')| + m)}{(n+m)^2} = \frac{n^2}{(n+m)^2}\left(1 + \frac{m}{|E(G')|} \right)t(\edge, G') $
\end{enumerate}
  
\end{theorem}
\begin{proof}
\begin{enumerate}
    \item A vertex can be mapped to any vertex in graph $G$ making  $\hom(\vertex, G) = |V(G)| = n + m$ and $\hom(\vertex, G') = |V(G')| = n$ giving the first equality. As $ t(F, G) = \frac{\hom(F, G)}{ |V(G)|^{|V(F)|} } $ we get
    \[
    t(\vertex, G) = \frac{n+m}{|V(G)|} = 1 = \frac{n}{|V(G')|} = t(\vertex, G') \, . 
    \]
    \item Adding stars using $S$ adds $m$ vertices and $m$ edges to $G'$.  Thus, $|E(G)| = |E(G')| + m$.
    An edge can be mapped to any edge in graph $G$ in two ways giving $\hom(\edge, G) = 2|E(G)| = 2 (|E(G')| + m) =  \hom(\edge, G') + 2m$.
    The homomorphism density 
    \begin{align}
        t(\edge, G') & = \frac{2 |E(G')| }{n^2}  \, ,  \label{eq:homEdgeG'}\\
         t(\edge, G) & = \frac{2 |E(G)| }{(n + m)^2} = \frac{2 (|E(G')| + m)}{(n + m)^2} \label{eq:homEdgeG}
    \end{align}
   By dividing  equation \eqref{eq:homEdgeG'} from equation \eqref{eq:homEdgeG} we get the result. 
    
\end{enumerate}    
\end{proof}

\begin{theorem}\label{thm:ringaddition} Let $\mathcal{W} = (W, \rho_n, f)$ denote a graphette and let $G' \sim \mathsf{G}(n, W, \rho_n)$ and $G \sim f(G')$ where $f = R$ denotes the ring function (GEF \ref{motif:ring3}) with ring size $c > 3$ for $c \in \mathcal{C}$. Suppose $R$ adds $k$ rings and $|V(G')| = n$ and $|V(G)| = n + m$. Then we have the following:
\begin{enumerate}
    \item $\hom(\vertex, G) =  \hom(\vertex, G') + m$ and $t(\vertex, G) = 1$
    \item $\hom(\edge, G) =  \hom(\edge, G') + 2(m + k)$ and $t(\edge, G) =  \frac{2(|E(G')| + m+k)}{(n+m)^2} = \frac{n^2}{(n+m)^2}\left(1 + \frac{m+k}{|E(G')|} \right)t(\edge, G') $
\end{enumerate}
\end{theorem}
\begin{proof}
Part 1 is the same as in Theorem \ref{thm:staraddition}. For part 2 note that a ring has the same number of vertices and edges. If $k$ rings are added to $G'$ and the number of vertices increased by $m$ then the edges would increase by $m + k$ as $R$ connects each ring to a node in $G'$. The rest is the same as in Theorem \ref{thm:staraddition}.
\end{proof}
\section{Algorithms}\label{app:algorithms}


\begin{algorithm}[H]
\caption{\textsc{BatchFGW}} 
\label{alg:batch_fgw}
\begin{algorithmic}[1]
    
\INPUT
Noise graphs $\mathcal{G}^{(0)}=\{G_0^i\}_{i=1}^{B}$; target graphs $\mathcal{G}^{(1)}=\{G_1^j\}_{j=1}^{B}$;
pretrained encoder $f_{enc}$; FGW tradeoff $\alpha\in[0,1]$.

\OUTPUT
Cost matrix $D \in \mathbb{R}^{B \times B}$; couplings $T = \{T^\star_{ij}\}_{i,j=1}^{B}$, one transport plan per pair, indexed identically to $D$.

\STATE Initialize $D \leftarrow \mathbf{0}_{B\times B}$\;
\STATE Initialize $T \leftarrow \emptyset$ 

\FOR{$i=1$ to $B$}
  \FOR{$j=1$ to $B$}
    
    \STATE $(D_{ij}, T_{ij}) \leftarrow \mathrm{FGW}_{\alpha}\!\left(G_0^i, G_1^j; f_{enc}\right)$\;
  \ENDFOR
\ENDFOR
\RETURN{$D$ and $T$}\;
\end{algorithmic}
\end{algorithm}

\begin{algorithm}[H]
\caption{\textsc{FlowMatchingUpdate}}
\label{alg:fm_update}
\resizebox{\linewidth}{!}{%
\begin{minipage}{1.00\linewidth}
\begin{algorithmic}[1]
\INPUT
Coupled pairs $\mathcal{P} = \{(G_0^i, G_1^{\pi^\star(i)})\}_{i=1}^{B}$;
velocity field parameters $\theta$;
regularization parameters $\beta_{\mathrm{end}}, \beta_{\mathrm{val}}, \beta_{\mathrm{atom}}$.
\OUTPUT Updated velocity Field (GNN Transformer) parameters $\theta$.
\STATE $\mathcal{L}\leftarrow 0$, $m\leftarrow 0$
\FORALL{$(G_0,G_1)\in \mathcal{P}$}
  \IF{$|V(G_0)|=0$ or $|V(G_1)|=0$}
    \STATE continue
  \ENDIF
  \STATE $n\gets |V(G_0)|=|V(G_1)|$
  \STATE $(A_0,X_0,F_0)\leftarrow \textsc{Graph}(G_0,n)$
  \STATE $(A_1,X_1,F_1)\leftarrow \textsc{Graph}(G_1,n)$
  \STATE $\Tilde{P} \gets \textsc{RowNormalize}(T[G_0, G_1])$ \COMMENTS{\scriptsize soft correspondence from Alg.~\ref{alg:batch_fgw}'s coupling}
  \STATE $A_1 \gets \Tilde{P} A_1 \Tilde{P}^\top$;\quad $X_1 \gets \Tilde{P} X_1$;\quad $F_{1,ab:} \gets \sum_{jk} \Tilde{P}_{aj} F_{1,jk:} \Tilde{P}_{bk}$ \COMMENTS{\scriptsize move target into source-node order}
  \STATE $t\sim\mathcal{U}[0,1]$
  \STATE $A_t\leftarrow (1-t)A_0 + tA_1$
  \STATE $X_t\leftarrow (1-t)X_0 + tX_1$
  \STATE $F_t\leftarrow (1-t)F_0 + tF_1$
  \STATE $\Delta A \leftarrow (A_1-A_0)$
  \STATE $\Delta X \leftarrow (X_1-X_0)$
  \STATE $\Delta F \leftarrow (F_1-F_0)$
  \STATE $(v_A,v_X,v_F)\leftarrow v_{\theta}(A_t,X_t,F_t,t)$
  \STATE $\mathcal{L}_{\mathrm{vel}} =
  \frac{1}{N_A}\|v_A-\Delta A\|_F^2 + \lambda_x\frac{1}{N_X}\|v_X-\Delta X\|_F^2 + \lambda_e\frac{1}{N_F}\|v_F-\Delta F\|_F^2$
  \STATE $\widehat{A}_1\leftarrow A_t + (1-t)v_A$
  \STATE $\widehat{X}_1\leftarrow X_t + (1-t)v_X$
  \STATE $\widehat{F}_1\leftarrow F_t + (1-t)v_F$
  \STATE $\mathcal{L}_{\mathrm{end}} = \frac{1}{N_A}\|\widehat{A}_1 - A_1\|_F^2
  + \lambda_x \frac{1}{N_X}\|\widehat{X}_1 - X_1\|_F^2
  + \lambda_e \frac{1}{N_F}\|\widehat{F}_1 - F_1\|_F^2$
  \STATE $\Tilde{\mathcal{L}} \leftarrow \mathcal{L}_{\mathrm{vel}} + \beta_{\mathrm{end}}\mathcal{L}_{\mathrm{end}}
    + \beta_{\mathrm{val}}\mathcal{L}_{\mathrm{val}} + \beta_{\mathrm{atom}}\mathcal{L}_{\mathrm{atom}}$
  \STATE $\mathcal{L}\leftarrow \mathcal{L}+\Tilde{\mathcal{L}}$
  \STATE $m\leftarrow m+1$
\ENDFOR
\IF{$m>0$}
  \STATE $\mathcal{L}\leftarrow \mathcal{L}/m$
  \STATE Update $\theta$ via AdamW step using $\nabla_{\theta}\mathcal{L}$ and $\eta$
\ENDIF
\RETURN{$\theta$}
\end{algorithmic}
\end{minipage}%
}
\end{algorithm}

\begin{algorithm}[H]
\caption{Euler Integration of the Learned Flow}
\label{alg:integrate_flow}
\begin{algorithmic}[1]
\INPUT
Velocity field $v_{\theta}$; initial state $(A_0,X_0,F_0)$; number of steps $K$; non-negativity activation function $\sigma(\cdot)$.

\OUTPUT Final state $(A_1,X_1,F_1)$.

\STATE Set $(A \leftarrow A_0,\; X \leftarrow X_0,\; F \leftarrow F_0)$\;
\STATE $\Delta t \leftarrow 1/K$\;

\FOR {$k=0$ to $K-1$}
  \STATE $t \leftarrow k/K$\;
  \STATE $(v_A,v_X,v_F) \leftarrow v_{\theta}(A,X,F,t)$\;
  \STATE $A \leftarrow A + \Delta t \, v_A$\;
  \STATE $X \leftarrow X + \Delta t \, v_X$\;
  \STATE $F \leftarrow F + \Delta t \, v_F$\;
\ENDFOR
\STATE $A \leftarrow \tfrac{1}{2}(A + A^\top)$
\STATE $A \leftarrow \sigma(A)$ 
\STATE $F \leftarrow \tfrac{1}{2}(F + F^\top)$
\RETURN{$(A,X,F)$}\;
\end{algorithmic}
\end{algorithm}


\begin{algorithm}[H]
\caption{Add Disjoint Rings}
\label{alg:attach_disjoint_rings}
\begin{algorithmic}[1]
    
\INPUT
Graph $G=(V,E)$; node budget $n$; ring-size counts $\mathcal{R}'=\{(s,c_s)\}$, where $c_s$ denotes the requested number of cycles of length $s$.

\OUTPUT{Graph $\widetilde{G}$ with $|V(\widetilde{G})|\le n$.}

\STATE $\widetilde{G}\leftarrow G$\;
\IF{$\mathcal{R}'=\emptyset$}
 \RETURN{$\widetilde{G}$}
\ENDIF
\FORALL{$(s,c_s)\in\mathcal{R}'$}
  \FOR{$r=1$ to $c_s$}
    \STATE $\texttt{rem}\leftarrow n-|V(\widetilde{G})|$\;
    \IF{$\texttt{rem}\le 0$}
     \STATE \textbf{break}
    \ENDIF 
    \STATE $s' \leftarrow \min(s,\texttt{rem})$\;
    \IF{$s'<3$}
     \STATE \textbf{continue}
    \ENDIF
    \STATE $H \leftarrow \mathrm{Cycle}(s')$ \COMMENTS{\scriptsize fresh cycle component}
    \STATE $\widetilde{G} \leftarrow \widetilde{G} \uplus H$\;

    \STATE Sample $v \sim \mathrm{Unif}\!\big(V(\widetilde{G})\setminus V(H)\big)$, \;
           $w \sim \mathrm{Unif}\!\big(V(H)\big)$\;
    \STATE $E(\widetilde{G}) \leftarrow E(\widetilde{G}) \cup \{(v,w)\}$ \COMMENTS{\scriptsize attach by one edge}

  \ENDFOR
\ENDFOR
\RETURN{$\widetilde{G}$}\;
\end{algorithmic}
\end{algorithm}

\clearpage
\begin{algorithm}[t]
\caption{Graphette Sampling with Ring Injection}
\label{alg:graphex_sampling}
\begin{algorithmic}[1]
\INPUT
Graphon matrix $W\in[0,1]^{m\times m}$; node budget $n$; sparsity factor $\rho$ (or \texttt{None}); stability $\varepsilon>0$; ring-size counts $\mathcal{R}'=\{(s,c_s)\}$, where $c_s$ is the number of cycles of length $s$ to inject (e.g., $\mathcal{R}'=\{(5,2),(6,1)\}$ requests two 5-cycles and one 6-cycle); max attempts $T$.

\OUTPUT Undirected graph $G=(V,E)$ with $|V|\le n$.

\COMMENTS{\scriptsize Sparsified-graphon sampling (graphex base)}
\STATE $\bar{W} \leftarrow \frac{1}{m^2}\sum_{a,b=1}^{m} W_{ab}$\;
\IF{$\rho = \texttt{None}$}
  \STATE $\rho \leftarrow \frac{1}{\bar{W}\,n} + \varepsilon$\;
\ENDIF
\STATE $W' \leftarrow \rho \, W$\;

\STATE Sample $u_1,\dots,u_n \overset{iid}{\sim}\mathrm{Unif}[0,1]$\;
\STATE Map to indices $\ell_i \leftarrow \min(m-1,\lfloor m u_i\rfloor)$ for $i\in[n]$\;
\STATE Initialize adjacency $A\leftarrow \mathbf{0}\in\{0,1\}^{n\times n}$\;
\FOR{$1\le i<j\le n$}
  \STATE $p_{ij}\leftarrow W'_{\ell_i,\ell_j}$\;
  \STATE Sample $A_{ij}\sim \mathrm{Bernoulli}(p_{ij})$; set $A_{ji}\leftarrow A_{ij}$\;
\ENDFOR
\STATE $G\leftarrow \textsc{GraphFromAdj}(A)$\;

\COMMENTS{\scriptsize Keep largest connected component, where $CC(G)$ denotes the set of connected components of $G$}
\STATE $C^\star \leftarrow \arg\max_{C\in \mathrm{CC}(G)} |C|$\;
\STATE $G \leftarrow G[C^\star]$\;
\IF{$|V(G)| > n$}
  \STATE Sample $S\subseteq V(G)$ uniformly with $|S|=n$; $G\leftarrow G[S]$\;
\ENDIF

\COMMENTS{\scriptsize Budgeted disjoint ring injection (attach new components)}
\STATE $G \leftarrow \textsc{AddDisjointRings}(G, n, \mathcal{R}')$\;


\COMMENTS{\scriptsize Final safety check}
\IF{$|V(G)|>n$}
  \STATE Sample $S\subseteq V(G)$ uniformly  with $|S|=n$; $G\leftarrow G[S]$\;
\ENDIF
\RETURN{$G$}\;
\end{algorithmic}
\end{algorithm}

\subsection{Soft Node Correspondence from FGW Coupling}
\label{app:node-correspondence}
For each FGW-coupled graph pair $(G_0, G_1)$, let $T^\star \in \mathbb{R}^{n_0 \times n_1}$ denote the optimal transport plan obtained from Eq.~\eqref{eq:fgw_objective} for this pair (Algorithm~\ref{alg:batch_fgw} returns both the distance $D_{ij}$, used for Hungarian graph-level pairing, and the corresponding coupling $T^\star$). We row-normalise $T^\star$ to obtain a soft node correspondence $\Tilde{P}_{ab} = \frac{T^\star_{ab}}{\sum_{b'} T^\star_{ab'}}$, which is a right-stochastic matrix mapping each target node $b$'s contribution onto source node $a$; this recovers the exact hard permutation when $T^\star$ is a permutation coupling and degrades gracefully to a soft alignment otherwise. Before constructing any interpolant, the target graph is transported into the source-node ordering:
\begin{equation}
{X}_1 = \Tilde{P} X_1, \qquad
{A}_1 = \Tilde{P} A_1 \Tilde{P}^\top, \qquad
{F}_{1,ab:} = \sum_{jk} \Tilde{P}_{aj} F_{1,jk:} \Tilde{P}_{bk},
\label{eq:transport}
\end{equation}
where the last expression applies $\Tilde{P}$ along both node indices of the edge tensor. We then restrict $(A_0,X_0,F_0)$ and $({A}_1,{X}_1,{F}_1)$ to the common node budget $n$, ensuring $|V_0| = |V_1| = n$ for every coupled pair, and construct the rectified interpolants in source-node coordinates (Sec.~\ref{sec:fm_graph}). This transport step is implemented explicitly in lines 9--10 of Algorithm~\ref{alg:fm_update}.

\section{Details of Empirical Evaluation}

\subsection{Discretisation and chemical validity.}
\label{subsec:discretisation}

At sampling time, the continuous outputs $(\hat{A}_1,\hat{X}_1,\hat{F}_1)$ obtained after $K$ Euler steps are produced as in Algorithm~\ref{alg:integrate_flow}, and are then mapped to a discrete molecular graph as described below. Atom and bond types are decoded by argmax over their categorical logits, with formal charges decoded analogously for datasets where charge features are used. The adjacency prediction $\hat{A}_1$ is symmetrised and transformed into edge-confidence scores using a sigmoid. These scores are binarised using an adaptive threshold based on a task-specific expected edge budget. Chemical feasibility is then promoted through valence-aware pruning, which iteratively removes the highest-bond-order incident edge from atoms whose current valence exceeds the permitted value, until no such violation remains. Finally, we optionally retain the largest connected component and sanitise the resulting molecule using RDKit. 

Post-generation validity correction is also used by existing one-shot molecular graph generators. MoFlow~\citep{zang2020moflow} applies an iterative post-hoc valency correction that reduces offending bond orders until the valency constraints are satisfied and retains the largest connected component, while GraphEBM~\citep{liu2021graphebm} adopts the same correction after Langevin sampling and discretisation. In contrast, the autoregressive GraphAF~\citep{shi2020graphaf} checks valency during each bond-generation step, rejecting and resampling bond proposals that violate the corresponding constraints. Flowette applies its deterministic terminal projection only after continuous flow integration, leaving the learned velocity field and continuous sampling dynamics unchanged.

\subsection{Evaluation Metrics for Synthetic Graph Evaluation}
\label{sec:metrics-syn}

Generation quality is assessed using different metrics: (1) degree distribution (Deg.), (2) clustering coefficient (Clus.), and (3) orbit count (Orbit) statistics. 
For all metrics, we quantify the discrepancy between the generated and test distributions using Maximum Mean Discrepancy (MMD) \cite{gretton2012kernel, martinkus2022spectre, bergmeisterefficient, qinmadeira2024defog}. These discrepancies are summarized using the Ratio metric (lower is better). 
We also report generation quality using validity (V), uniqueness (U), novelty (N), and their combined V.U.N.\ metric. Following prior work \cite{qinmadeira2024defog}, we generate the same number of graphs as in the test set (40 samples) to enable a fair comparison of graph statistic distributions.

We evaluate generative quality by comparing the distributions of key graph statistics computed on generated samples against those of the test set. Specifically, we consider degree distributions, clustering coefficients, and orbit counts. Distributional discrepancies are quantified using the Maximum Mean Discrepancy (MMD).
Specifically, we first compute MMD values between the training and test sets for the same statistics, establishing a reference level of distributional shift. We then compute MMD values between the generated and test sets. The Ratio is defined as the average generated:test MMD normalized by the average training:test MMD, providing a scale-free measure of how closely the generated graphs match the test distribution.

Rather than reporting absolute MMD values, we follow a normalised evaluation protocol that reports relative error ratios. For a given statistic, we compute
\begin{equation}
r \;=\; 
\frac{\mathrm{MMD}^2\!\big(\mathcal{G}_{\mathrm{gen}},\, \mathcal{G}_{\mathrm{test}}\big)}
{\mathrm{MMD}^2\!\big(\mathcal{G}_{\mathrm{train}},\, \mathcal{G}_{\mathrm{test}}\big)},
\end{equation}
where $\mathcal{G}_{\mathrm{gen}}$, $\mathcal{G}_{\mathrm{train}}$, and $\mathcal{G}_{\mathrm{test}}$ denote the generated, training, and test graph sets, respectively. The denominator corresponds to the reference training--test discrepancy reported in prior work~\citep{vignacdigress}. Note that throughout, we use $\mathrm{MMD}^2$, consistent with the convention adopted in DiGress, where the squared MMD is reported.

\subsection{Hyperparameter settings for Synthetic Graph Evaluation}
\label{sec:hyperparam-syn}

For synthetic graph benchmarks, we use a simplified training objective by setting $\beta_{\mathrm{val}}=0$ and $\beta_{\mathrm{atom}}=0$, thereby excluding chemistry-specific regularization terms that are not applicable to purely structural graph data. We optimize all models using the AdamW optimizer~\cite{loshchilovdecoupled}. For synthetic graph benchmarks, models are trained for 100 epochs with a learning rate of $1e^{-3}$, hidden dimension $H=128$, and batch size $B=32$. We set the $\beta_{\mathrm{end}}=1.0$ and $\lambda_x=\lambda_e=0.5$, and use an FGW trade-off parameter $\alpha=0.5$. Euler integration is performed with $K=50$ steps at sampling time. For the network architecture, we use $L=3$ attention layers for the Tree and Ego-small datasets, and $L=4$ attention layers for the SBM dataset.

\subsection{Tables of results for synthetic datasets}
\label{sec:moreresults-syn}
We compare Flowette against a diverse set of strong baselines across synthetic benchmarks. On the SBM dataset (Table~\ref{Tab:ablation_sbm}), we evaluate against ten methods, including SPECTRE \cite{martinkus2022spectre}, DiGress \cite{vignacdigress}, BwR \cite{diamant2023improving}, HSpectre \cite{bergmeisterefficient}, GruM \cite{jo2024graph}, CatFlow \cite{eijkelboom2024variational}, DisCo \cite{xu2024discrete}, Cometh \cite{siraudincometh}, DeFoG \cite{qinmadeira2024defog}, and G2PT \cite{chen2025graph}. On the Tree dataset (Table~\ref{Tab:ablation_tree}), we compare against nine baselines, including GRAN \cite{liao2019efficient}, DiGress \cite{vignacdigress}, BwR \cite{diamant2023improving}, EDGE \cite{chen2023efficient}, BiGG \cite{dai2020scalable}, GraphGen \cite{goyal2020graphgen}, HSpectre \cite{bergmeisterefficient}, DeFoG \cite{qinmadeira2024defog}, and G2PT \cite{chen2025graph}. On the Ego-small dataset (Table~\ref{Tab:ego_small}), we compare against seven baselines: GraphRNN \cite{you2018graphrnn}, GRAN \cite{liao2019efficient}, GraphCNF \cite{lippecategorical}, GDSS \cite{jo2022score}, DiscDDPM \cite{haefeli2022diffusion}, DiGress \cite{vignacdigress}, and EDGE \cite{chen2023efficient}.

\textbf{SBM Dataset.} For SBM graphs in Table \ref{Tab:ablation_sbm}, our approach outperforms all baselines on Orbit, Validity, Uniqueness, Novelty, and V.U.N., while Degree and Clustering are marginally weaker than the strongest competitors. SBMs are defined by mesoscopic block structure rather than precise local statistics, which is well captured by the graphette prior and reinforced by Flowette. This yields accurate recovery of block-level connectivity and higher-order substructure patterns. Small discrepancies in degree and clustering arise because these second-order marginals can vary within blocks without affecting the overall community structure and are not explicitly constrained by our objective.

\textbf{Tree Dataset.} On the Tree dataset in Table \ref{Tab:ablation_tree}, our method achieves the best performance on Orbit, Validity, Uniqueness, Novelty, and the combined V.U.N. score, while slightly underperforming on Degree. This behaviour is consistent with our use of a graphette structural prior and Flowette, which strongly constrain acyclicity and higher-order subgraph structure. As a result, tree-specific substructure patterns and validity are preserved extremely well. In contrast, degree distributions in trees are highly sensitive to structural prior, and our continuous flow prioritizes global structural alignment over exact degree matching, leading to minor deviations in Degree despite structurally correct samples.

\textbf{Ego-small Dataset.} On Ego-small graphs in Table \ref{Tab:ego_small}, our method achieves the best Orbit score and competitive Degree performance, with Clustering slightly below the top baseline. Ego graphs are dominated by hub-and-spoke and star-like motifs, which align naturally with graphex constructions and are effectively preserved by our flow model. While orbit statistics and degree distributions are recovered, clustering depends on fine-grained triangle closure, which may vary under continuous flow integration even when the overall ego-centric structure is correctly modeled.

\begin{table*}[ht]
\caption{Generation performance on SBM synthetic graphs.}
\centering 
\scalebox{0.65}{\begin{tabular}{l c c c c c c c}
\specialrule{.1em}{.05em}{.05em} 
Model \hspace*{1.5cm}&\hspace*{0.3cm}Deg. $\downarrow$ \hspace*{0.3cm} & \hspace*{0.3cm}Clus. $\downarrow$ \hspace*{0.3cm} & \hspace*{0.3cm}Orbit $\downarrow$ \hspace*{0.3cm} & \hspace*{0.3cm}Valid $\uparrow$ \hspace*{0.3cm} & \hspace*{0.3cm}Unique $\uparrow$ \hspace*{0.3cm} & \hspace*{0.3cm}Novel $\uparrow$ \hspace*{0.3cm} & \hspace*{0.3cm}V.U.N $\uparrow$ \hspace*{0.3cm} \\ [0.5ex]
\hline
SPECTRE \cite{martinkus2022spectre} & 0.0015 & 0.0521 & 0.0412 & 52.5 & \textbf{100.0} & \textbf{100.0} & 52.5 \\
DiGress \cite{vignacdigress} & 0.0018 & 0.0485 & 0.0415 & 60.0 & \textbf{100.0} & \textbf{100.0} & 60.0 \\
BwR \cite{diamant2023improving} & 0.0478 & 0.0638 & 0.1139 & 7.5 & \textbf{100.0} & \textbf{100.0} & 7.5 \\
HSpectre \cite{bergmeisterefficient} & 0.0141 & 0.0528 & 0.0809 & 75.0 & \textbf{100.0} & \textbf{100.0} & 75.0 \\
GruM \cite{jo2024graph} & 0.0015 & 0.0589 & 0.0450 & - & - & - & 85.0 \\
CatFlow \cite{eijkelboom2024variational} & 0.0012 & 0.0498 & 0.0357 & - & - & - & 85.0 \\
DisCo \cite{xu2024discrete} & 0.0006 & 0.0266 & 0.0510 & 66.2 & \textbf{100.0} & \textbf{100.0} & 66.2 \\
Cometh \cite{siraudincometh} & 0.0020 & 0.0498 & 0.0383 & 75.0 & \textbf{100.0} & \textbf{100.0} & 75.0 \\
DeFoG \cite{qinmadeira2024defog} & \textbf{0.0006} & 0.0517 & 0.0556 & 90.0 & 90.0 & 90.0 & 90.0 \\
G2PT \cite{chen2025graph} & 0.0035 & \textbf{0.0120} & 0.0007 & - & - & - & \textbf{100.0} \\
Flowette (Ours)  & 0.0076 \tiny $\pm$ 0.0035 & 0.0508 \tiny $\pm$ 0.01 & \textbf{0.0000} \tiny $\pm$ 0.0 & \textbf{100.0} \tiny $\pm$ 0.0 & \textbf{100.0} \tiny $\pm$ 0.0 & \textbf{100.0} \tiny $\pm$ 0.0 & \textbf{100.0} \tiny $\pm$ 0.0 \\
\specialrule{.1em}{.05em}{.05em}
\end{tabular}} 
\label{Tab:ablation_sbm}
\end{table*}

\begin{table*}[ht]
\caption{Generation performance on Tree synthetic graphs.}
\centering 
\scalebox{0.65}{\begin{tabular}{l c c c c c c c}
\specialrule{.1em}{.05em}{.05em} 
Model \hspace*{1.5cm}&\hspace*{0.3cm}Deg. $\downarrow$ \hspace*{0.3cm} & \hspace*{0.3cm}Clus. $\downarrow$ \hspace*{0.3cm} & \hspace*{0.3cm}Orbit $\downarrow$ \hspace*{0.3cm} & \hspace*{0.3cm}Valid $\uparrow$ \hspace*{0.3cm} & \hspace*{0.3cm}Unique $\uparrow$ \hspace*{0.3cm} & \hspace*{0.3cm}Novel $\uparrow$ \hspace*{0.3cm} & \hspace*{0.3cm}V.U.N $\uparrow$ \hspace*{0.3cm} \\ [0.5ex]
\hline
GRAN \cite{liao2019efficient} & 0.1884 & 0.0080 & 0.0199 & 0.0 & \textbf{100.0} & \textbf{100.0} & 0.0 \\
DiGress \cite{vignacdigress} & 0.0002 & \textbf{0.0000} & \textbf{0.0000} & 90.0 & \textbf{100.0} & \textbf{100.0} & 90.0 \\
BwR \cite{diamant2023improving} & 0.0016 & 0.1239 & 0.0003 & 0.0 & \textbf{100.0} & \textbf{100.0} & 0.0 \\
EDGE \cite{chen2023efficient} & 0.2678 & \textbf{0.0000} & 0.7357 & 0.0 & 7.5 & \textbf{100.0} & 0.0 \\
BiGG \cite{dai2020scalable} & 0.0014 & \textbf{0.0000} & \textbf{0.0000} & \textbf{100.0} & 87.5 & 50.0 & 75.0 \\
GraphGen \cite{goyal2020graphgen} & 0.0105 & \textbf{0.0000} & \textbf{0.0000} & 95.0 & \textbf{100.0} & \textbf{100.0} & 95.0 \\
HSpectre \cite{bergmeisterefficient} & 0.0004 & \textbf{0.0000} & \textbf{0.0000} & 82.5 & \textbf{100.0} & \textbf{100.0} & 82.5 \\
DeFoG \cite{qinmadeira2024defog} & \textbf{0.0002} & \textbf{0.0000} & \textbf{0.0000} & 96.5 & \textbf{100.0} & \textbf{100.0} & 96.5 \\
G2PT \cite{chen2025graph} & 0.0020 & \textbf{0.0000} & \textbf{0.0000} & - & - & - & 99.0 \\
Flowette (Ours)  & 0.0005 \tiny $\pm$ 0.002 & \textbf{0.0000} \tiny $\pm$ 0.0 & \textbf{0.0000} \tiny $\pm$ 0.0 & \textbf{100.0} \tiny $\pm$ 0.0 & \textbf{100.0} \tiny $\pm$ 0.0 & \textbf{100.0} \tiny $\pm$ 0.0 & \textbf{100.0} \tiny $\pm$ 0.0 \\
\specialrule{.1em}{.05em}{.05em}
\end{tabular}} 
\label{Tab:ablation_tree}
\end{table*}

\begin{table}[ht]
\caption{Generation performance on Ego-small synthetic graphs.}
\centering 
\scalebox{0.7}{\begin{tabular}{l c c c c}
\specialrule{.1em}{.05em}{.05em} 
Model \hspace*{1.5cm}&\hspace*{0.3cm}Deg. $\downarrow$ \hspace*{0.3cm} & \hspace*{0.3cm}Clus. $\downarrow$ \hspace*{0.3cm} & \hspace*{0.3cm}Orbit $\downarrow$ \hspace*{0.3cm} \\ [0.5ex]
\hline
GraphRNN \cite{you2018graphrnn} & 0.0768 & 1.1456 & 0.1087 \\ 
GRAN \cite{liao2019efficient} & 0.5778 & 0.3360 & 0.0406 \\ 
GraphCNF \cite{lippecategorical} & 0.1010 & 0.7654 & 0.0820 \\
GDSS \cite{jo2022score} & 0.8189 & 0.6032 & 0.3315 \\
DiscDDPM \cite{haefeli2022diffusion} & 0.4613 & 0.1681 & 0.0633 \\
DiGress \cite{vignacdigress} & 0.0708 & 0.0092 & 0.1205 \\
EDGE \cite{chen2023efficient} & \textbf{0.0040} & \textbf{0.0040} & 0.0008 \\
Flowette (Ours)  & 0.0435 \tiny $\pm$ 0.004 & 0.0098 \tiny $\pm$ 0.0025 & \textbf{0.0000} \tiny $\pm$ 0.0 \\
\specialrule{.1em}{.05em}{.05em}
\end{tabular}} 
\label{Tab:ego_small}
\end{table}


\subsection{Evaluation Metrics for Molecular Graph Evaluation}
\label{sec:metrics-molecules}
We evaluate generative performance using a collection of complementary molecular metrics \cite{vignacdigress, qinmadeira2024defog}. Validity reports the fraction of generated molecules that satisfy basic chemical constraints such as allowable valences. Uniqueness quantifies how many generated molecules are structurally distinct, as determined by non-isomorphic SMILES representations. Novelty measures the proportion of samples that do not appear in the training set. To evaluate distributional similarity, we report the Similarity to Nearest Neighbor (SNN), which captures how close each generated molecule is to its most similar training example using Tanimoto similarity. Scaffold similarity (Scaf.) compares the distributions of Bemis–Murcko core structures. While KL divergence (KL div.) measures discrepancies between generated and reference distributions of selected physicochemical descriptors. Neighborhood Subgraph Pairwise Distance Kernel (NSPDK) evaluates similarity at the level of local graph substructures by comparing distributions of rooted subgraphs and their pairwise distances, thereby capturing higher-order structural consistency between generated and reference molecular graphs.

Let $\mathcal{M}_{\mathrm{gen}} = \{m_i\}_{i=1}^{N}$ denote the set of generated molecules and $\mathcal{M}_{\mathrm{train}}$ the training set.

\textbf{Validity.}
Validity measures the fraction of generated molecules that satisfy basic chemical constraints, such as permissible atom valences and bond types. Formally,
\begin{equation}
\mathrm{Validity}
=
\frac{1}{N}
\sum_{i=1}^{N}
\mathbf{1}\!\left[m_i \text{ is chemically valid}\right],
\end{equation}
where $\mathbf{1}[\cdot]$ denotes the indicator function and chemical validity is assessed using RDKit sanitization rules.

\textbf{Uniqueness.}
Uniqueness quantifies the proportion of structurally distinct molecules among valid samples, determined via canonical SMILES representations:
\begin{equation}
\mathrm{Uniqueness}
=
\frac{|\mathrm{Unique}(\mathcal{M}_{\mathrm{gen}}^{\mathrm{valid}})|}
{|\mathcal{M}_{\mathrm{gen}}^{\mathrm{valid}}|}.
\end{equation}

\textbf{Novelty.}
Novelty measures the fraction of generated molecules that do not appear in the training set:
\begin{equation}
\mathrm{Novelty}
=
\frac{1}{|\mathcal{M}_{\mathrm{gen}}^{\mathrm{valid}}|}
\sum_{m \in \mathcal{M}_{\mathrm{gen}}^{\mathrm{valid}}}
\mathbf{1}\!\left[m \notin \mathcal{M}_{\mathrm{train}}\right].
\end{equation}

\textbf{Similarity to Nearest Neighbour (SNN).}
SNN measures how close each generated molecule is to its most similar reference molecule in fingerprint space. Let $\Tilde{\phi}(\cdot)$ denote a fixed molecular fingerprint map, $\mathcal{M}_{\mathrm{gen}}$ the set of generated molecules, and $\mathcal{M}_{\mathrm{ref}}$ the reference (training) set. We define
\begin{equation}
\mathrm{SNN}
=
\frac{1}{|\mathcal{M}_{\mathrm{gen}}|}
\sum_{m \in \mathcal{M}_{\mathrm{gen}}}
\max_{m' \in \mathcal{M}_{\mathrm{ref}}}
\mathrm{Tanimoto}\big(\Tilde{\phi}(m), \Tilde{\phi}(m')\big),
\end{equation}
where the Tanimoto similarity is computed over binary fingerprints and the maximum is taken over all reference molecules.

\textbf{Scaffold Similarity}
\paragraph{Scaffold similarity (Scaf.).}
Scaffold similarity evaluates how well the distribution of molecular scaffolds is preserved in generated samples. Let $\mathcal{S}$ denote the union of Bemis--Murcko scaffolds observed in the reference and generated sets. We construct scaffold count vectors
\[
\mathbf{r} = (r_s)_{s \in \mathcal{S}}, \qquad
\mathbf{g} = (g_s)_{s \in \mathcal{S}},
\]
where $r_s$ and $g_s$ denote the number of molecules containing scaffold $s$ in the reference and generated sets, respectively (with missing entries treated as zero) \cite{polykovskiy2020molecular}. Scaffold similarity is defined as the cosine similarity between these count vectors:
\begin{equation}
\mathrm{Scaf.}
=
\frac{\langle \mathbf{r}, \mathbf{g} \rangle}
{\|\mathbf{r}\|_2 \, \|\mathbf{g}\|_2}.
\end{equation}
This metric captures agreement between scaffold frequency distributions while being invariant to absolute dataset size.

\textbf{Physicochemical KL divergence.}
To assess distributional fidelity at the level of physicochemical properties, we follow standard GuacaMol-style benchmarks and compare descriptor distributions between training and generated molecules.
Let $\mathcal{D}$ denote a fixed set of scalar descriptors, e.g.\ Bertz complexity, logP, molecular weight, topological polar surface area, and several integer-valued counts (H-bond donors/acceptors, rotatable bonds, aliphatic and aromatic rings) \cite{brown2019guacamol}. For each descriptor $j \in \mathcal{D}$ and molecule $x$, we write $d_j(x) \in \mathbb{R}$ for its value.

Given a reference (training) set $\mathcal{X}_{\mathrm{ref}}$ and a generated set $\mathcal{X}_{\mathrm{gen}}$, we construct empirical distributions of $d_j$ under each set.
For integer-valued descriptors, we estimate discrete histograms
\[
p_{\mathrm{ref}}^{(j)}(k)
\;=\;
\frac{1}{|\mathcal{X}_{\mathrm{ref}}|}
\sum_{x \in \mathcal{X}_{\mathrm{ref}}}
\mathbf{1}\{d_j(x)=k\},
\qquad
p_{\mathrm{gen}}^{(j)}(k)
\;=\;
\frac{1}{|\mathcal{X}_{\mathrm{gen}}|}
\sum_{x \in \mathcal{X}_{\mathrm{gen}}}
\mathbf{1}\{d_j(x)=k\},
\]
where $k$ ranges over a finite set of bins.
The KL divergence for descriptor $j$ is then
\begin{equation}
\label{eq:kl_discrete_desc}
D_{\mathrm{KL}}^{(j)}
\;=\;
\sum_{k}
p_{\mathrm{ref}}^{(j)}(k)
\log
\frac{p_{\mathrm{ref}}^{(j)}(k)}{p_{\mathrm{gen}}^{(j)}(k)}.
\end{equation}

For continuous-valued descriptors, we estimate densities
$p_{\mathrm{ref}}^{(j)}(z)$ and $p_{\mathrm{gen}}^{(j)}(z)$ (e.g.\ via kernel
density estimation or smoothed histograms) and define
\begin{equation}
\label{eq:kl_continuous_desc}
D_{\mathrm{KL}}^{(j)}
\;=\;
\int
p_{\mathrm{ref}}^{(j)}(z)
\log
\frac{p_{\mathrm{ref}}^{(j)}(z)}{p_{\mathrm{gen}}^{(j)}(z)}
\, dz,
\end{equation}
which in practice is approximated numerically.


\textbf{Aggregate KL score (KL div.).}
The overall physicochemical KL metric aggregates these divergences across all descriptors and the internal-similarity distribution.
Let $\mathcal{J}$ denote the index set consisting of all descriptor dimensions together with the internal-similarity term.
We report either the average KL,
\begin{equation}
\label{eq:kl_physchem_avg}
\mathrm{KL}_{\mathrm{physchem}}
\;=\;
\frac{1}{|\mathcal{J}|}
\sum_{j \in \mathcal{J}}
D_{\mathrm{KL}}^{(j)},
\end{equation}
or the corresponding similarity score
\begin{equation}
\label{eq:kl_physchem_score}
\mathrm{Score}_{\mathrm{physchem}}
\;=\;
\frac{1}{|\mathcal{J}|}
\sum_{j \in \mathcal{J}}
\exp\bigl(-D_{\mathrm{KL}}^{(j)}\bigr),
\end{equation}
where smaller KL values (or larger scores) indicate closer agreement between the generated and training distributions of physicochemical descriptors and internal similarities.

\textbf{Neighborhood Subgraph Pairwise Distance Kernel (NSPDK).}
NSPDK evaluates similarity at the level of local graph substructures by comparing distributions of rooted subgraphs and their pairwise distances. Given feature maps $\phi(\cdot)$ induced by the kernel, the similarity between generated and reference sets is measured via a maximum mean discrepancy (MMD):
\begin{equation}
\mathrm{NSPDK}
=
\left\|
\frac{1}{|\mathcal{M}_{\mathrm{gen}}|}
\sum_{m \in \mathcal{M}_{\mathrm{gen}}} \phi(m)
-
\frac{1}{|\mathcal{M}_{\mathrm{train}}|}
\sum_{m \in \mathcal{M}_{\mathrm{train}}} \phi(m)
\right\|_2^2.
\end{equation}

NSPDK captures higher-order structural consistency beyond local degree statistics, making it particularly sensitive to discrepancies in functional groups and subgraph arrangements.

\subsection{Hyperparameter setup for molecular graph generation}
\label{sec:hyperparam-molecules}
We compare our approach against seven baseline models on the Guacamol and MOSES benchmarks: Digress \cite{vignacdigress}, DisCo \cite{xu2024discrete}, Cometh \cite{siraudincometh}, GraphBFN \cite{song2025smooth}, G2PT \cite{chen2025graph}, GEEL \cite{jangsimple}, and DeFoG \cite{qinmadeira2024defog}. We further evaluate our method on the QM9 and ZINC250K datasets using ten comparable baselines: EDP-GNN \cite{niu2020permutation}, GraphARM \cite{kong2023autoregressive}, DiGress \cite{vignacdigress}, GDSS \cite{jo2022score}, GDSS + TF \cite{jo2022score}, GruM \cite{jo2024graph}, SID \cite{bogetsimple}, DeFoG \cite{qinmadeira2024defog}, GraphBFN \cite{song2025smooth}, and GBD \cite{liuadvancing}.

For molecular graph benchmarks, we optimize all models using the training objective defined in Eq.~\eqref{eq:total_loss} and employ the AdamW optimizer~\cite{loshchilovdecoupled}. All models are trained for 500 epochs with a learning rate of $1e^{-3}$, hidden dimension $H=128$, batch size $B=32$, and loss weights $\lambda_x=\lambda_e=0.5$.  We set the FGW trade-off parameter to $\alpha=0.5$. For the QM9, ZINC250K, and Guacamol datasets, we use $\beta_{\mathrm{val}}=0.5$, $\beta_{\mathrm{atom}}=0.5$, and $\beta_{\mathrm{end}}=1.0$. For the MOSES dataset, we set $\beta_{\mathrm{val}}=0.3$, $\beta_{\mathrm{atom}}=0.4$, and $\beta_{\mathrm{end}}=0.8$. During sampling, Euler integration is performed with $K=200$ steps. Regarding model architecture, we use $L=5$ attention layers for QM9 and ZINC250K, and $L=4$ attention layers for GuacaMol and MOSES. Following existing evaluation protocols \cite{qinmadeira2024defog}, we generate 10000 graphs for QM9, ZINC250K, and Guacamol, and 25000 graphs for MOSES at inference time to ensure fair comparison.

\subsection{Runtime Analysis} \label{sec:runtime_analysis}
We evaluate the computational efficiency of Flowette under the same experimental setup as DeFoG~\cite{qinmadeira2024defog}, using batch size $B=512$ and $1000$ training epochs. As shown in Table~\ref{tab:runtime}, Flowette remains competitive in total training time (i.e., FGW + Hungarian pairing + model training) across all benchmarks. The additional FGW + Hungarian pairing overhead is small, ranging from 0.016h (Tree) to 1.68h (MOSES), and constitutes only a minor fraction of the overall training cost. Moreover, this overhead can be further reduced, as pairings can be pre-computed offline once before training and reused across epochs. However, Flowette is strictly faster than DeFoG at inference across all benchmarks. Since FGW coupling and Hungarian matching are training-only operations, inference requires only standard Euler integration of the learned velocity field, with no CTMC simulation and no rate matrix computation, that DeFoG requires at sampling time.


\begin{table}[t]
\centering
\caption{Runtime comparison between Flowette (under the default online coupling protocol; see Sec.~\ref{sec:scalability} for amortised alternatives) and DeFoG. Flowette training time and pairing time (FGW + Hungarian) are reported separately to highlight the marginal overhead of coupling.}
\scalebox{0.62}{
\begin{tabular}{l c c c c c c}
\specialrule{.1em}{.05em}{.05em}
Dataset & DeFoG Model Train (h) & Flowette Pairing (h) & Flowette Model Train (h) & DeFoG Sample (h) & Flowette Sample (h) & Graphs Sampled \\
\hline
Tree & 8.0 & 0.016 & 8.50 & 0.07 & 0.04 & 40 \\
QM9 & 6.5 & 0.560 & 6.90 & 0.20 & 0.16 & 10,000 \\
ZINC250K & 14.0 & 0.750 & 14.60 & 4.80 & 3.70 & 10,000 \\
MOSES & 46.0 & 1.680 & 46.50 & 5.00 & 4.30 & 25,000 \\
\specialrule{.1em}{.05em}{.05em}
\end{tabular}}
\label{tab:runtime}
\end{table}

\subsection{Scalability of Coupling Protocols}
\label{sec:scalability}

\paragraph{Coupling protocols.} The default protocol, used throughout Algorithm~\ref{alg:fgw_fm_training} and all main results unless stated otherwise, is \emph{online coupling}: FGW costs and the Hungarian assignment are recomputed for each minibatch at every epoch, so no fixed precomputed source pool is assumed. We additionally support two optional amortised variants for settings where repeated FGW computation is a bottleneck: \emph{fixed-pool} coupling, which performs one blockwise pass over the training set before the first epoch, samples one source graph per target, and caches the resulting source graph and transport plan keyed by the target's dataset index for reuse across all subsequent epochs; and \emph{periodic} coupling, which refreshes this cache every $\Tilde{\mathcal{R}}$ epochs. In both amortised variants, precomputation is performed in minibatch-sized $B \times B$ blocks rather than as a single dataset-wide matching problem, and the cached correspondence remains valid under shuffled training since it is indexed by target graph identity rather than minibatch position.

To isolate the cost of the coupling step itself and evaluate scalability with graph size, we conduct a controlled scaling study on Tree and SBM graphs with $n \in \{100, 500, 1000\}$ nodes, using 2000 training, 500 validation, and 500 test graphs per setting. Trees are sampled from random Pr\"ufer sequences; four-community SBMs use $p_{\mathrm{in}} = 30/n$ and $p_{\mathrm{out}} = 4/n$ to hold expected degree approximately constant across $n$. We compare the three coupling protocols -- \emph{online, fixed-pool, and periodic} ($\Tilde{\mathcal{R}}=10$), under identical hardware, model settings, and a 1000-epoch training budget, reporting total end-to-end training time (FGW coupling + Hungarian matching + model training).
\begin{table}[h]
\centering
\caption{Total end-to-end training time (hours) under online, fixed-pool, and periodic
($\Tilde{\mathcal{R}}=10$) coupling, across graph size $n$ and dataset family.}
\label{tab:scalability}
\begin{tabular}{llrrr}
\toprule
Dataset & $n$ & Online (h) & Fixed-pool (h) & Periodic, $\Tilde{\mathcal{R}}{=}10$ (h) \\
\midrule
\multirow{3}{*}{Tree} & 100  & 8.90  & 3.50  & 5.15  \\
                       & 500  & 16.25 & 5.86  & 9.45  \\
                       & 1000 & 26.30 & 8.67  & 12.36 \\
\midrule
\multirow{3}{*}{SBM}  & 100  & 12.25 & 5.65  & 7.91  \\
                       & 500  & 22.05 & 8.82  & 12.75 \\
                       & 1000 & 34.56 & 11.58 & 14.95 \\
\bottomrule
\end{tabular}
\end{table}

Repeatedly recomputing FGW under the online protocol introduces substantial training overhead as $n$ grows, whereas amortising the coupling considerably improves practical scalability. Across both Tree and SBM graphs, fixed-pool coupling reduces end-to-end training time by a large margin relative to online coupling, while periodic coupling with $\Tilde{\mathcal{R}}=10$ offers a smaller but still meaningful reduction, consistent with its role as an intermediate trade-off between full re-coupling and a static correspondence. This consistent trend across sparse, acyclic Trees and community-structured SBMs indicates that the benefit is not specific to a particular graph family. Fixed-pool coupling gives the lowest runtime, while periodic coupling offers an intermediate trade-off by periodically refreshing the alignment during training. Neither FGW nor Hungarian matching is required at inference under any of the three protocols, so this overhead is confined entirely to training.

\subsection{Ablation Analysis}\label{sec:ablation-moreresults}

\subsubsection{Ablation Analysis of Regularization Terms}\label{sec:ablation_reg}


The comparatively poor molecular performance obtained with $\mathcal{L}_{\mathrm{vel}}$ alone (Table~\ref{tab:ablation_qm9_zinc}) should not be interpreted as evidence of insufficient backbone capacity. Rather, it exposes a mismatch between per-entry MSE-based continuous regression and discrete chemical feasibility: after endpoint discretization, a small error in an atom- or bond-type prediction can change a categorical assignment and violate valence, invalidating the entire molecule even when its contribution to the continuous regression loss is small.

The endpoint and chemistry-aware terms address complementary aspects of this mismatch. As shown in Sec.~3.3 and App.~\ref{subsec:lend_derivation}, $\mathcal{L}_{\mathrm{end}}$ does not introduce an independent supervision target; under the rectified parameterization it is exactly a $(1-t)^2$ temporal reweighting of the same velocity residual penalized by $\mathcal{L}_{\mathrm{vel}}$. The factor follows directly from the remaining extrapolation horizon in the one-step endpoint estimate, so the same velocity residual induces a larger predicted endpoint error when evaluated earlier in the trajectory. Consequently, combining $\mathcal{L}_{\mathrm{vel}}$ with $\beta_{\mathrm{end}}\mathcal{L}_{\mathrm{end}}$ provides an endpoint-aware temporal reweighting that places greater emphasis on residuals with a larger effect on the predicted terminal state. In contrast, $\mathcal{L}_{\mathrm{val}}$ directly discourages chemically infeasible bonding, while $\mathcal{L}_{\mathrm{atom}}$ stabilizes the predicted atom-type composition. These terms provide complementary, low-overhead supervision without changing the FGW coupling, velocity-field architecture, or underlying continuous transport process.

To evaluate the role of each component in the Flowette training objective, we conduct an ablation study by selectively removing its key elements as shown in Table~\ref{tab:ablation_qm9_zinc}.
\begin{itemize}
    \item \textbf{NoReg/ only $\mathcal{L}_{vel}$}: This variant removes all regularization terms and only keeps $\mathcal{L}_{vel}$.
    \item \textbf{NoValenceReg/ no $\mathcal{L}_{val}$}: This variant removes only the regularization term $\mathcal{L}_{val}$.
    \item \textbf{NoAtomReg/ no $\mathcal{L}_{atom}$}: This variant removes only the regularization term $\mathcal{L}_{atom}$.
    \item \textbf{NoEndReg/ no $\mathcal{L}_{end}$}: This variant removes only the regularization term $\mathcal{L}_{end}$.
    \item \textbf{NoChemReg/ only $\mathcal{L}_{vel}$ and $\mathcal{L}_{end}$}: This variant removes both chemistry-aware regularization terms $\mathcal{L}_{val}$ and $\mathcal{L}_{atom}$.
\end{itemize}

\begin{table*}[t]
   \caption{Ablation study of graph generation performance of QM9 and ZINC250K. Recall from Equation~\eqref{eq:total_loss}, the components in the training objective:
   the velocity loss $\mathcal{L}_{\mathrm{vel}}$, endpoint consistency loss $\mathcal{L}{\mathrm{end}}$, soft valence constraint $\mathcal{L}_{\mathrm{val}}$, and atom-type matching $\mathcal{L}_{\mathrm{atom}}$.}
   \label{tab:ablation_qm9_zinc}
   \centering
   \footnotesize
   \scalebox{0.95}{\begin{tabular}{ c l c c c c c c c c c | c c c c c c c c } 
     \toprule
     \multicolumn{2}{c}{} & \multicolumn{3}{c}{QM9}  &\multicolumn{3}{c}{ZINC250K} \\\cmidrule(lr){3-5} \cmidrule(lr){6-8}
     \multicolumn{2}{c}{Variant} & Valid & Unique & Novel & Valid & Unique & Novel \\
     \midrule
     \multirow{5}{*}{} 
     & only $\mathcal{L}_{\mathrm{vel}}$ & 69.0 \tiny $\pm$ 0.10 & 65.0 \tiny $\pm$ 0.02 & 65.0 \tiny $\pm$ 0.06 & 75.0 \tiny $\pm$ 0.05 & 83.1 \tiny $\pm$ 0.10 & 83.1 \tiny $\pm$ 0.03 \\
     & no $\mathcal{L}_{\mathrm{val}}$ & 76.0 \tiny $\pm$ 0.07 & 90.7 \tiny $\pm$ 0.04 & 90.7 \tiny $\pm$ 0.06 & 88.0 \tiny $\pm$ 0.10 & 91.0 \tiny $\pm$ 0.05 & 91.0 \tiny $\pm$ 0.04 \\
     & no $\mathcal{L}_{\mathrm{atom}}$ & 79.0 \tiny $\pm$ 0.05 & 93.6 \tiny $\pm$ 0.06 & 93.6 \tiny $\pm$ 0.03 & 92.1 \tiny $\pm$ 0.02 & 95.0 \tiny $\pm$ 0.08 & 95.0 \tiny $\pm$ 0.06 \\
     & no $\mathcal{L}_{\mathrm{end}}$ & 71.0 \tiny $\pm$ 0.06 & 79.0 \tiny $\pm$ 0.05 & 79.0 \tiny $\pm$ 0.05 & 79.0 \tiny $\pm$ 0.11 & 85.0 \tiny $\pm$ 0.04 & 85.0 \tiny $\pm$ 0.07 \\
     & only $\mathcal{L}_{\mathrm{vel}}$ and $\mathcal{L}_{\mathrm{end}}$ & 73.1 \tiny $\pm$ 0.05 & 85.0 \tiny $\pm$ 0.06 & 85.0 \tiny $\pm$ 0.05 & 82.6 \tiny $\pm$ 0.03 & 90.5 \tiny $\pm$ 0.04 & 90.5 \tiny $\pm$ 0.04 \\
     & Full model & \textbf{99.8} \tiny $\pm$ 0.09 & \textbf{99.3} \tiny $\pm$ 0.05 & \textbf{99.4} \tiny $\pm$ 0.06 & \textbf{99.9} \tiny $\pm$ 0.10 & \textbf{100.0} \tiny $\pm$ 0.0 & \textbf{100.0} \tiny $\pm$ 0.0 \\
     \bottomrule
   \end{tabular}}
\end{table*}
The chemistry-aware terms $\mathcal{L}_\mathrm{val}$ and $\mathcal{L}_\mathrm{atom}$ primarily govern feasibility and diversity. Following endpoint consistency, removing valence supervision (\textbf{no $\mathcal{L}_\mathrm{val}$}) causes a pronounced validity drop on both datasets, demonstrating the importance of soft valence constraints for preventing over-bonding. In contrast, removing atom-type marginal matching (\textbf{no $\mathcal{L}_\mathrm{atom}$}) has a milder effect on validity, uniqueness and novelty. Removing both terms (\textbf{only $\mathcal{L}_\mathrm{vel}$ and $\mathcal{L}_\mathrm{end}$}) yields intermediate performance, confirming that the two regularizers provide complementary signals. Velocity matching alone is insufficient, and strong performance emerges only when all components are jointly enforced.

\paragraph{Why velocity matching alone is insufficient on molecules.} \label{par:velocity_matching}
The performance of the $\mathcal{L}_{\mathrm{vel}}$-only variant reflects the interaction between continuous regression and discrete molecular validity, rather than an architectural inability to represent the transport field. The velocity objective penalizes numerical residuals in the continuous graph representation, whereas chemical validity is determined after categorical decoding and is highly non-smooth. For example, a small perturbation of bond-type logits may change the decoded bond from single to double and cause an atom to exceed its allowed valence. Such an error can invalidate the entire molecule even when it contributes only weakly to the aggregate regression loss. Conversely, a residual of comparable magnitude can leave the decoded molecular graph unchanged. Thus, continuous velocity error is not uniformly aligned with the downstream chemical consequence of that error.


This issue is specific to domains with hard combinatorial constraints, such as valence, and does not indicate that the overall Flowette backbone is molecule-specific: the same backbone -- FGW-coupled supervision, the permutation-equivariant GNN transformer, and the rectified flow-matching objective -- achieves state-of-the-art results on our synthetic benchmarks (Tree, SBM, Ego-small; Sec.~\ref{sec:resultsSynthetic}) with $\beta_{\mathrm{val}} = \beta_{\mathrm{atom}} = 0$ and no chemistry terms whatsoever. The backbone is therefore already domain-general and demonstrably strong on its own; it is only when the target domain imposes hard combinatorial constraints that additional supervision is needed to align the training signal with what validity means in that domain.

\paragraph{Endpoint consistency.}
Removing $\mathcal{L}_{\mathrm{end}}$ substantially degrades validity and uniqueness on both molecular datasets. As shown in App.~\ref{subsec:lend_derivation}, $\mathcal{L}_{\mathrm{end}}$ is not an independent endpoint supervision target; it is exactly a $(1-t)^2$ reweighting of the velocity residual already penalized by $\mathcal{L}_{\mathrm{vel}}$. This endpoint-aware temporal reweighting places greater emphasis on velocity residuals with a larger effect on the predicted terminal state, corresponding to earlier points along the trajectory with a longer remaining extrapolation horizon. The ablation therefore indicates that this non-uniform temporal weighting contributes materially to molecular generation quality, even though it introduces no additional transport target.

\paragraph{Chemistry-aware endpoint regularization.}

The terms $\mathcal{L}_{\mathrm{val}}$ and $\mathcal{L}_{\mathrm{atom}}$ supply targeted supervision for chemical consequences that are weakly represented by the per-entry MSE-based continuous regression objective. Removing $\mathcal{L}_{\mathrm{val}}$ causes a pronounced validity drop on both QM9 and ZINC250K, confirming that explicit soft valence supervision helps prevent over-bonding. Removing $\mathcal{L}_{\mathrm{atom}}$ produces a milder degradation, consistent with its role in stabilizing global atom-type composition rather than directly imposing a hard local chemical constraint. Removing both chemistry-aware terms yields intermediate performance. These regularizers require only coarse domain information--allowable valences and target atom-type marginals--and operate on the predicted endpoint. They do not alter the FGW coupling, the velocity-field architecture, or the continuous transport process. Table~\ref{tab:ablation_qm9_zinc} therefore shows that they are targeted corrections for the mismatch between continuous MSE-based training errors and discrete chemical feasibility, rather than components of a molecule-specific redesign.


\subsubsection{Sensitivity Analysis of the Pretrained Encoder $f_{enc}$} \label{sec:sensitivity_analysis_encoder}

\textbf{Pre-trained encoder.} We use a GIN encoder (3-layer GINEConv, hidden size 128, output size 64, learning rate 1e-3, batch-size 32, and epochs 50) trained once in a self-supervised manner as a graph autoencoder. It is optimized using cross-entropy losses to reconstruct atom and bond types from node and edge features, without external labels. The encoder is trained only on the training split of each dataset and is never exposed to validation or test data, preventing data leakage. After training, it is frozen and used as a fixed structural embedding extractor, receiving no gradient updates during FM training. Its role is limited to computing node embeddings for FGW distance computation within mini-batches; it does not influence the velocity field, integration process, or generated outputs.

To evaluate sensitivity to encoder quality, we conduct a four-condition ablation, as shown in Table~\ref{tab:sensitivity_encoder}: (1) GIN encoder (50 epochs, fully pretrained), (2) training-free 1-WL (Weisfeiler-Lehman \cite{shervashidze2011weisfeiler}) features (3 iterations of WL, zero parameters), (3) Early stopped GIN encoder (5 epochs, partially trained), and (4) a random encoder (weights fixed with default initialisation, zero training). Notably, the gap between the pretrained GIN and the training-free 1-WL baseline is consistently the smallest across all conditions, metrics, and datasets, while performance degrades progressively as encoder quality decreases from 1-WL to early-stopped to random. This confirms that structural information is far more critical than learned representations in our framework.

\begin{table*}[t]
   \caption{Sensitivity analysis of the pre-trained encoder on QM9 and ZINC250K datasets.}
   \label{tab:sensitivity_encoder}
   \centering
   \footnotesize
   \scalebox{0.75}{\begin{tabular}{ c l c c c c c c c c c | c c c c c c c c } 
     \toprule
     \multicolumn{2}{c}{} & \multicolumn{3}{c}{QM9}  &\multicolumn{3}{c}{ZINC250K} \\\cmidrule(lr){3-5} \cmidrule(lr){6-8}
     \multicolumn{2}{c}{Variant} & Valid & Unique & NSPDK & Valid & Unique & NSPDK \\
     \midrule
     \multirow{4}{*}{}
     & Flowette (pretrained GIN encoder) & \textbf{99.81} \tiny $\pm$ \textbf{0.09} & \textbf{99.30} \tiny $\pm$ \textbf{0.05} & \textbf{0.0003} \tiny $\pm$ \textbf{0.0002} & \textbf{99.90} \tiny $\pm$ \textbf{0.10} & \textbf{100.00} \tiny $\pm$ \textbf{0.0} & \textbf{0.0006} \tiny $\pm$ \textbf{0.0004} \\
     & Flowette (1-WL features) & 98.30 \tiny $\pm$ 0.05 & 98.00 \tiny $\pm$ 0.06 & 0.0008 \tiny $\pm$ 0.0003 & 98.70 \tiny $\pm$ 0.030 & 99.00 \tiny $\pm$ 0.050 & 0.0018 \tiny $\pm$ 0.0007 \\
     & Flowette (Early stopped) & 94.30 \tiny $\pm$ 0.06 & 93.00 \tiny $\pm$ 0.04 & 0.0153 \tiny $\pm$ 0.0008 & 93.80 \tiny $\pm$ 0.060 & 92.60 \tiny $\pm$ 0.030 & 0.0316 \tiny $\pm$ 0.0003 \\
     & Flowette (Random encoder) & 88.30 \tiny $\pm$ 0.08 & 87.00 \tiny $\pm$ 0.06 & 0.0656 \tiny $\pm$ 0.0006 & 88.80 \tiny $\pm$ 0.070 & 88.60 \tiny $\pm$ 0.050 & 0.0819 \tiny $\pm$ 0.0006 \\
     \bottomrule
   \end{tabular}}
\end{table*}

\subsubsection{Ablation of Graphette Prior vs. Standard Graphon Prior}\label{subsec:graphette_vs_graphon}
We address this via a controlled ablation comparing Flowette with the graphette prior to a variant using a standard graphon prior, with FGW coupling and all loss terms fixed. As shown in Tables~\ref{tab:graphette_synthetic} and \ref{tab:graphette_molecular}, replacing the graphette prior causes large and consistent degradation. On Tree, Deg increases by $0.04715$ and V.U.N decreases by $32.05$. On SBM, Clus. increases by $0.8007$ and V.U.N decreases by $38.10$. On QM9, Valid decreases by $24.31$, Unique by $30.40$, and NSPDK increases by $0.0757$. On ZINC250K, Valid decreases by $19.90$, Unique by $15.00$, and NSPDK increases by $0.0554$. These consistent changes confirm that the graphette prior is a critical independent component for capturing the dominant topological properties of each graph family.
\begin{table}[t]
\centering
\caption{Ablation study: Graphette prior vs. standard graphon on synthetic datasets.}
\scalebox{0.60}{
\begin{tabular}{l l c c c c c c c}
\specialrule{.1em}{.05em}{.05em}
Dataset & Model Variant & Deg $\downarrow$ & Clus. $\downarrow$ & Orbit $\downarrow$ & Valid $\uparrow$ & Unique $\uparrow$ & Novel $\uparrow$ & V.U.N $\uparrow$ \\
\hline
\multirow{2}{*}{Tree}
& Flowette (graphette prior) & \textbf{0.00005 $\pm$ 0.002} & \textbf{0.0000 $\pm$ 0.0} & \textbf{0.0000 $\pm$ 0.0} & \textbf{100.0 $\pm$ 0.0} & \textbf{100.0 $\pm$ 0.0} & \textbf{100.0 $\pm$ 0.0} & \textbf{100.0 $\pm$ 0.0} \\
& Flowette (standard graphon) & 0.0472 $\pm$ 0.008 & 0.0136 $\pm$ 0.001 & 0.0012 $\pm$ 0.005 & 75.5 $\pm$ 0.03 & 90.0 $\pm$ 0.02 & 100.0 $\pm$ 0.0 & 67.95 $\pm$ 0.06 \\
\hline
\multirow{2}{*}{SBM}
& Flowette (graphette prior) & \textbf{0.0076 $\pm$ 0.003} & \textbf{0.0508 $\pm$ 0.01} & \textbf{0.0000 $\pm$ 0.0} & \textbf{100.0 $\pm$ 0.0} & \textbf{100.0 $\pm$ 0.0} & \textbf{100.0 $\pm$ 0.0} & \textbf{100.0 $\pm$ 0.0} \\
& Flowette (standard graphon) & 0.1207 $\pm$ 0.002 & 0.8515 $\pm$ 0.003 & 0.0105 $\pm$ 0.004 & 80.0 $\pm$ 0.02 & 90.5 $\pm$ 0.03 & 85.5 $\pm$ 0.04 & 61.90 $\pm$ 0.03 \\
\specialrule{.1em}{.05em}{.05em}
\end{tabular}}
\label{tab:graphette_synthetic}
\end{table}

\begin{table}[t]
\centering
\caption{Ablation study: Graphette prior vs. standard graphon on molecular datasets.}
\scalebox{0.8}{
\begin{tabular}{l l c c c}
\specialrule{.1em}{.05em}{.05em}
Dataset & Model Variant & Valid $\uparrow$ & Unique $\uparrow$ & NSPDK $\downarrow$ \\
\hline
\multirow{2}{*}{QM9}
& Flowette (graphette prior) & \textbf{99.81 \tiny $\pm$ 0.09} & \textbf{99.30 \tiny $\pm$ 0.05} & \textbf{0.0003 \tiny $\pm$ 0.0002} \\
& Flowette (standard graphon) & 75.50 \tiny $\pm$ 0.03 & 68.90 \tiny $\pm$ 0.05 & 0.0760 \tiny $\pm$ 0.0006 \\
\hline
\multirow{2}{*}{ZINC250K}
& Flowette (graphette prior) & \textbf{99.90 \tiny $\pm$ 0.10} & \textbf{100.00 \tiny $\pm$ 0.0} & \textbf{0.0006 \tiny $\pm$ 0.0004} \\
& Flowette (standard graphon) & 80.00 \tiny $\pm$ 0.05 & 85.00 \tiny $\pm$ 0.03 & 0.0560 \tiny $\pm$ 0.0005 \\
\specialrule{.1em}{.05em}{.05em}
\end{tabular}}
\label{tab:graphette_molecular}
\end{table}

\subsubsection{Ablation of FGW coupling vs. Alternative Pairings}
\label{subsec:ablation_fgw_coupling_vs_alternative_pairing}
We provide a controlled ablation comparing FGW coupling against independent (random) pairing and Euclidean OT, with the graphette prior and all loss terms held fixed. As shown in Tables~\ref{tab:fgw_synthetic} and \ref{tab:fgw_molecular}, FGW consistently outperforms both alternatives across all datasets and metrics. On Tree, Deg increases by 0.00802 (random) and 0.00073 (Euclidean OT), while V.U.N decreases by 18.55 and 9.28, respectively. On SBM, Clus. increases by 0.4251 (random) and 0.0424 (Euclidean OT), while V.U.N decreases by 19.23 and 10.99. On QM9, Valid decreases by 9.78 (random) and 5.31 (Euclidean OT), Unique by 9.65 and 6.05, and NSPDK increases by 0.0177 and 0.0062. On ZINC250K, Valid decreases by 10.40 (random) and 5.90 (Euclidean OT), Unique by 8.50 and 4.50, and NSPDK increases by 0.0154 and 0.0024. These consistent degradations confirm that FGW coupling is a critical independent component, producing more coherent and reliable supervision than both feature-agnostic random pairing and feature-only Euclidean OT.
\begin{table}[t]
\centering
\caption{Ablation study: FGW coupling vs. alternative pairings on synthetic datasets.}
\scalebox{0.60}{
\begin{tabular}{l l c c c c c c c}
\specialrule{.1em}{.05em}{.05em}
Dataset & Model Variant & Deg $\downarrow$ & Clus. $\downarrow$ & Orbit $\downarrow$ & Valid $\uparrow$ & Unique $\uparrow$ & Novel $\uparrow$ & V.U.N $\uparrow$ \\
\hline
\multirow{3}{*}{Tree}
& Flowette (FGW) & \textbf{0.00005 \tiny $\pm$ 0.002} & \textbf{0.0000 \tiny $\pm$ 0.0} & \textbf{0.0000 \tiny $\pm$ 0.0} & \textbf{100.0 \tiny $\pm$ 0.0} & \textbf{100.0 \tiny $\pm$ 0.0} & \textbf{100.0 \tiny $\pm$ 0.0} & \textbf{100.0 \tiny $\pm$ 0.0} \\
& Flowette (random) & 0.00807 \tiny $\pm$ 0.005 & 0.0055 \tiny $\pm$ 0.003 & 0.0000 \tiny $\pm$ 0.0 & 90.50 \tiny $\pm$ 0.06 & 90.0 \tiny $\pm$ 0.02 & 100.0 \tiny $\pm$ 0.0 & 81.45 \tiny $\pm$ 0.04 \\
& Flowette (Euclidean OT) & 0.00078 \tiny $\pm$ 0.003 & 0.0003 \tiny $\pm$ 0.005 & 0.0000 \tiny $\pm$ 0.0 & 95.50 \tiny $\pm$ 0.05 & 95.0 \tiny $\pm$ 0.04 & 100.0 \tiny $\pm$ 0.0 & 90.72 \tiny $\pm$ 0.06 \\
\hline
\multirow{3}{*}{SBM}
& Flowette (FGW) & \textbf{0.0076 \tiny $\pm$ 0.003} & \textbf{0.0508 \tiny $\pm$ 0.01} & \textbf{0.0000 \tiny $\pm$ 0.0} & \textbf{100.0 \tiny $\pm$ 0.0} & \textbf{100.0 \tiny $\pm$ 0.0} & \textbf{100.0 \tiny $\pm$ 0.0} & \textbf{100.0 \tiny $\pm$ 0.0} \\
& Flowette (random) & 0.0151 \tiny $\pm$ 0.003 & 0.4759 \tiny $\pm$ 0.001 & 0.0012 \tiny $\pm$ 0.001 & 89.50 \tiny $\pm$ 0.03 & 90.25 \tiny $\pm$ 0.05 & 100.0 \tiny $\pm$ 0.0 & 80.77 \tiny $\pm$ 0.05 \\
& Flowette (Euclidean OT) & 0.0091 \tiny $\pm$ 0.002 & 0.0932 \tiny $\pm$ 0.005 & 0.0005 \tiny $\pm$ 0.004 & 93.50 \tiny $\pm$ 0.06 & 95.20 \tiny $\pm$ 0.02 & 100.0 \tiny $\pm$ 0.0 & 89.01 \tiny $\pm$ 0.03 \\
\specialrule{.1em}{.05em}{.05em}
\end{tabular}}
\label{tab:fgw_synthetic}
\end{table}

\begin{table}[t]
\centering
\caption{Ablation study: FGW coupling vs. alternative pairings on molecular datasets.}
\scalebox{0.8}{
\begin{tabular}{l l c c c}
\specialrule{.1em}{.05em}{.05em}
Dataset & Model Variant & Valid $\uparrow$ & Unique $\uparrow$ & NSPDK $\downarrow$ \\
\hline
\multirow{3}{*}{QM9}
& Flowette (FGW) & \textbf{99.81 \tiny $\pm$ 0.09} & \textbf{99.30 \tiny $\pm$ 0.05} & \textbf{0.0003 \tiny $\pm$ 0.0002} \\
& Flowette (random) & 90.03 \tiny $\pm$ 0.01 & 89.65 \tiny $\pm$ 0.04 & 0.0180 \tiny $\pm$ 0.0005 \\
& Flowette (Euclidean OT) & 94.50 \tiny $\pm$ 0.03 & 93.25 \tiny $\pm$ 0.02 & 0.0065 \tiny $\pm$ 0.0006 \\
\hline
\multirow{3}{*}{ZINC250K}
& Flowette (FGW) & \textbf{99.90 \tiny $\pm$ 0.10} & \textbf{100.00 \tiny $\pm$ 0.0} & \textbf{0.0006 \tiny $\pm$ 0.0004} \\
& Flowette (random) & 89.50 \tiny $\pm$ 0.07 & 91.50 \tiny $\pm$ 0.05 & 0.0160 \tiny $\pm$ 0.0003 \\
& Flowette (Euclidean OT) & 94.00 \tiny $\pm$ 0.05 & 95.50 \tiny $\pm$ 0.03 & 0.0030 \tiny $\pm$ 0.0004 \\
\specialrule{.1em}{.05em}{.05em}
\end{tabular}}
\label{tab:fgw_molecular}
\end{table}

\subsubsection{Ablation of FGW vs. Standard Gromov-Wasserstein Coupling}
\label{subsec:fgw_vs_gw}
Standard GW coupling is recovered from FGW (Eq.~\eqref{eq:app_fgw}) as a special case: setting $\alpha=1$ removes the feature term entirely, and even at $\alpha<1$ the feature term becomes non-discriminative when node features are absent or constant. Our motivation for using FGW throughout is therefore not to outperform GW on purely structural datasets, but to provide a single coupling mechanism that also handles attributed graphs: for molecules, topology-only matching can align structurally similar but chemically incompatible atoms, introducing noisy feature interpolation and velocity targets, which FGW resolves by jointly matching topology and atom identity.

To isolate the feature contribution, we compare FGW ($\alpha=0.5$) against GW ($\alpha=1.0$) on the non-attributed Tree and SBM benchmarks, holding the graphette prior, velocity-field architecture, and training objective fixed (Table~\ref{tab:fgw_vs_gw}). As expected, the two variants perform essentially identically: both achieve $100\%$ validity, uniqueness, novelty, and V.U.N., with only negligible differences in structural MMD metrics. This confirms that FGW reduces to GW when attributes are uninformative, while retaining the flexibility needed for attributed molecular graphs (Sec.~\ref{sec:fgw_fm}).
\begin{table}[h]
\centering
\caption{FGW ($\alpha=0.5$) vs. standard GW ($\alpha=1.0$) coupling on non-attributed synthetic benchmarks.}
\label{tab:fgw_vs_gw}
\scalebox{0.75}{\begin{tabular}{llrrrrrrr}
\toprule
Dataset & Coupling & Deg. $\downarrow$ & Clus. $\downarrow$ & Orbit $\downarrow$ & Valid $\uparrow$ & Unique $\uparrow$ & Novel $\uparrow$ & V.U.N $\uparrow$ \\
\midrule
\multirow{2}{*}{Tree}
& FGW ($\alpha{=}0.5$) & $0.00005$ {\tiny $\pm$ 0.002} & $0.0000$ {\tiny $\pm$ 0.0} & $0.0000$ {\tiny $\pm$ 0.0} & $100.0$ {\tiny $\pm$ 0.0} & $100.0$ {\tiny $\pm$ 0.0} & $100.0$ {\tiny $\pm$ 0.0} & $100.0$ {\tiny $\pm$ 0.0} \\
& GW ($\alpha{=}1.0$)  & $0.00010$ {\tiny $\pm$ 0.003} & $0.0000$ {\tiny $\pm$ 0.0} & $0.0000$ {\tiny $\pm$ 0.0} & $100.0$ {\tiny $\pm$ 0.0} & $100.0$ {\tiny $\pm$ 0.0} & $100.0$ {\tiny $\pm$ 0.0} & $100.0$ {\tiny $\pm$ 0.0} \\
\midrule
\multirow{2}{*}{SBM}
& FGW ($\alpha{=}0.5$) & $0.0076$ {\tiny $\pm$ 0.003} & $0.0508$ {\tiny $\pm$ 0.01} & $0.0000$ {\tiny $\pm$ 0.0} & $100.0$ {\tiny $\pm$ 0.0} & $100.0$ {\tiny $\pm$ 0.0} & $100.0$ {\tiny $\pm$ 0.0} & $100.0$ {\tiny $\pm$ 0.0} \\
& GW ($\alpha{=}1.0$)  & $0.0080$ {\tiny $\pm$ 0.005} & $0.0595$ {\tiny $\pm$ 0.03} & $0.0000$ {\tiny $\pm$ 0.0} & $100.0$ {\tiny $\pm$ 0.0} & $100.0$ {\tiny $\pm$ 0.0} & $100.0$ {\tiny $\pm$ 0.0} & $100.0$ {\tiny $\pm$ 0.0} \\
\bottomrule
\end{tabular}}
\end{table}

\subsubsection{Ablation of Pre- vs. Post-Projection Validity}
\label{subsec:pre_post_projection}
The terminal projection described in Sec.~\ref{subsec:discretisation} is a deterministic, inference-only step, so its contribution to reported validity can be isolated from that of the learned dynamics by evaluating the same generated trajectories before and after projection is applied. Table~\ref{tab:pre_post_projection} reports this comparison on QM9, ZINC250K, and MOSES.
\begin{table}[h]
\centering
\caption{Validity, uniqueness, novelty, and scaffold similarity before and after valence-aware terminal projection. Pre-projection metrics isolate the contribution of the learned dynamics; post-projection metrics reflect the full generation pipeline reported in Tables~\ref{Tab:qm9_zinc} and~\ref{Tab:guacamol_moses}.}
\label{tab:pre_post_projection}
\scalebox{0.9}{\begin{tabular}{llrrrr}
\toprule
Dataset & Variant & Valid & Unique & Novel & Scaff. \\
\midrule
\multirow{2}{*}{QM9}
& Pre-projection
& $98.06$ {\tiny $\pm$ 0.06}
& $98.28$ {\tiny $\pm$ 0.07}
& $98.00$ {\tiny $\pm$ 0.04}
& $11.30$ \\
& Post-projection
& $\mathbf{99.81}$ {\tiny $\pm$ 0.09}
& $\mathbf{99.30}$ {\tiny $\pm$ 0.05}
& $\mathbf{99.40}$ {\tiny $\pm$ 0.06}
& $\mathbf{13.60}$ \\
\midrule
\multirow{2}{*}{ZINC250K}
& Pre-projection
& $98.25$ {\tiny $\pm$ 0.02}
& $98.35$ {\tiny $\pm$ 0.04}
& $98.56$ {\tiny $\pm$ 0.05}
& $14.10$ \\
& Post-projection
& $\mathbf{99.90}$ {\tiny $\pm$ 0.10}
& $\mathbf{100.00}$ {\tiny $\pm$ 0.0}
& $\mathbf{100.00}$ {\tiny $\pm$ 0.0}
& $\mathbf{15.45}$ \\
\midrule
\multirow{2}{*}{MOSES}
& Pre-projection
& $97.95$ {\tiny $\pm$ 0.05}
& $98.00$ {\tiny $\pm$ 0.03}
& $97.06$ {\tiny $\pm$ 0.05}
& $14.45$ \\
& Post-projection
& $\mathbf{99.05}$ {\tiny $\pm$ 0.0}
& $\mathbf{99.90}$ {\tiny $\pm$ 0.0}
& $\mathbf{98.10}$ {\tiny $\pm$ 0.0}
& $\mathbf{15.30}$ \\
\bottomrule
\end{tabular}}
\end{table}
Pre-projection outputs already achieve high validity ($98.06\%$ on QM9, $98.25\%$ on ZINC250K, $97.95\%$ on MOSES), indicating that the learned flow itself produces predominantly chemically feasible structures. Projection provides a modest additional gain of $1.75$, $1.65$, and $1.10$ percentage points respectively, rather than converting largely invalid outputs into valid molecules. Uniqueness and novelty shift only moderately, and scaffold scores remain close between the two settings, indicating that projection corrects residual local valence violations without substantially redefining molecular identity or global structure. This is consistent with Flowette's training objective: the differentiable valence regulariser $\mathcal{L}_{\mathrm{val}}$ already encourages chemically feasible endpoints during learning (Sec.~\ref{sec:learning_objective}), while terminal projection acts as a final safeguard against residual discretisation or finite-step integration error rather than as the primary source of validity.

Post-generation validity correction of this kind is not unique to Flowette: MoFlow \citep{zang2020moflow} applies post-hoc validity correction when assembling generated molecular graphs, GraphEBM \citep{liu2021graphebm} applies validity correction to generated node and adjacency tensors, and autoregressive methods such as GraphAF \citep{shi2020graphaf} enforce validity checks during generation. We therefore report Flowette's pre-projection metrics above as a direct, projection-free comparison point, and its post-projection metrics separately (Tables~\ref{Tab:qm9_zinc} and \ref{Tab:guacamol_moses}) as the performance of the complete generation pipeline. We note that applying an identical projection to a baseline is meaningful only when that baseline exposes comparable atom- and bond-level confidence scores; absent such scores, confidence-based bond correction would require an arbitrary ranking over discrete outputs and would not constitute an equivalent procedure. For baselines that do expose compatible logits or probabilities, the same projection could be applied and evaluated directly.

\section{Observed Limitations} \label{sec:limitations}
While Flowette shows strong performance, several aspects offer avenues for further improvement: (i) online FGW coupling incurs $\mathcal{O}(B^2 n^3)$ precomputation cost (Table~\ref{tab:runtime}), which may limit scalability and motivates more efficient couplings; (ii) graphette specification relies on coarse domain knowledge (e.g., rings, hubs, trees); our controlled ablations (Sec.~\ref{subsec:graphette_sensitivity}) show this choice matters -- replacing the intended edit function with an incompatible or contradictory one substantially degrades validity and NSPDK on QM9 and ZINC250K, so automating $f$ via graphon-estimation methods \citep{xia2023implicit, azizpour2025scalable} is a promising direction rather than a purely cosmetic one; (iii) the optional endpoint regularizers $\mathcal{L}_{\mathrm{val}}$ and $\mathcal{L}_{\mathrm{atom}}$ encode coarse molecular constraints; analogous endpoint constraints may be required in other domains with hard combinatorial validity conditions, although Flowette achieves strong synthetic performance without these chemistry-aware terms ($100\%$ V.U.N.\ on Tree and SBM); (iv) the continuous relaxation requires a discrete projection step at generation time, which may introduce minor artefacts and motivates tighter integration.

\section{Graphs generated by Flowette}
\label{sec:generated-graphs}

\begin{figure}[!ht]
    \centering
    \includegraphics[width=0.95\textwidth]{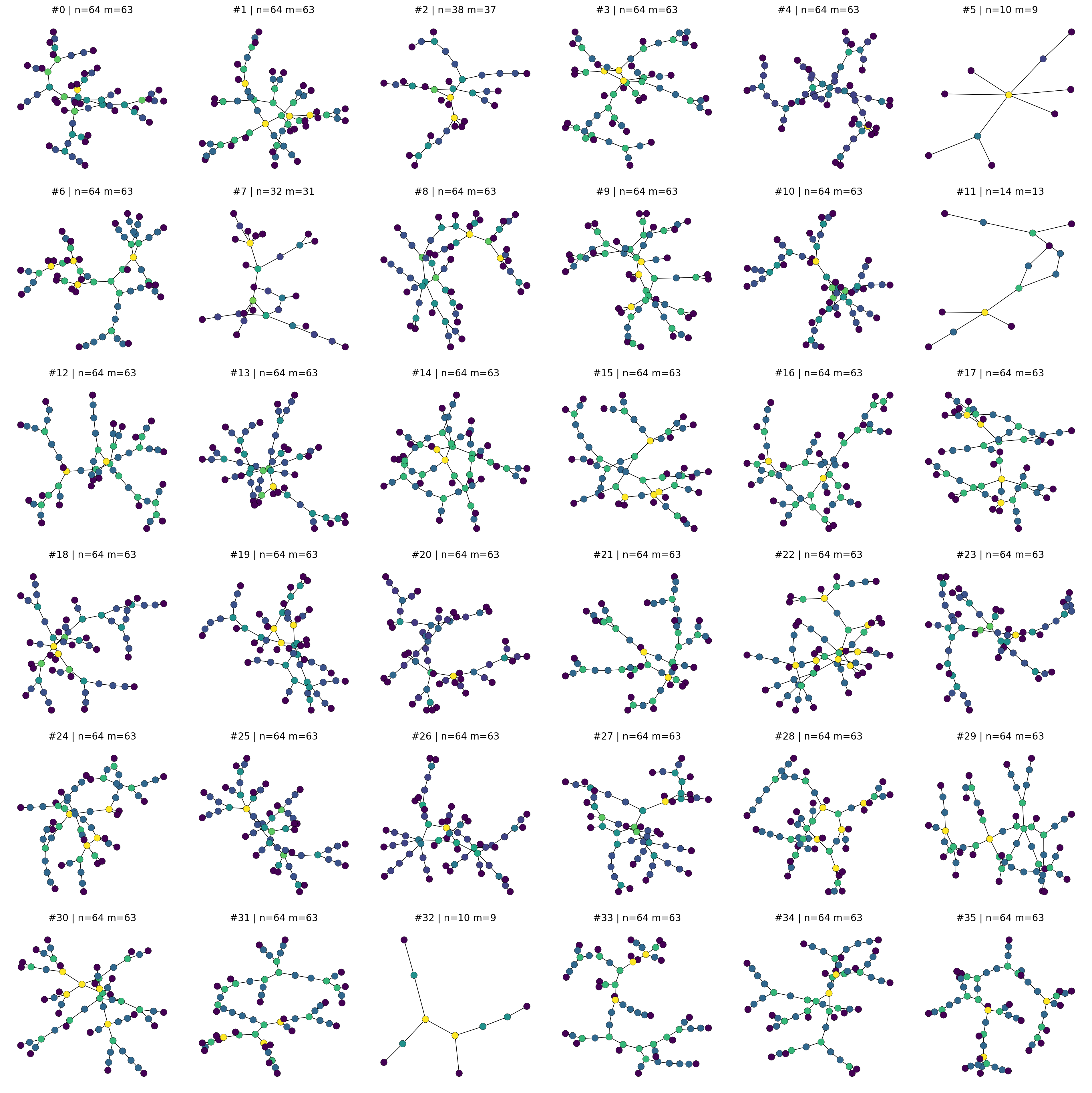}
    \caption{Graphs generated by Flowette for the Tree synthetic datasets.
    \label{fig:tree_graphs}}\vspace{-0.2cm}
\end{figure}

\begin{figure}[t!]
    \centering
    \includegraphics[width=0.95\textwidth]{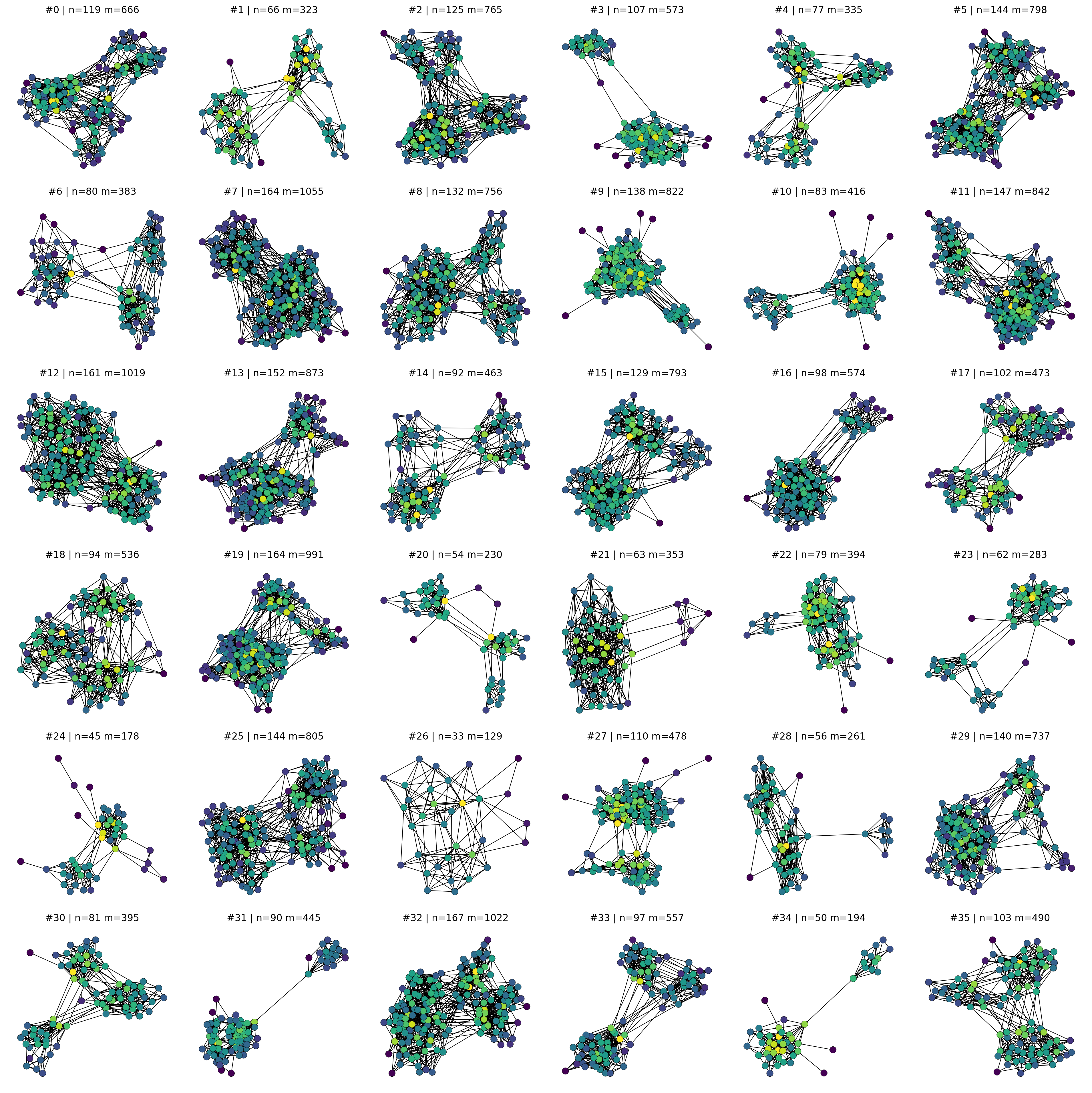}
    \caption{Graphs generated by Flowette for the SBM synthetic datasets.
    \label{fig:tree_graphs}}\vspace{-0.2cm}
\end{figure}

\begin{figure}[t!]
    \centering
    \includegraphics[width=0.95\textwidth]{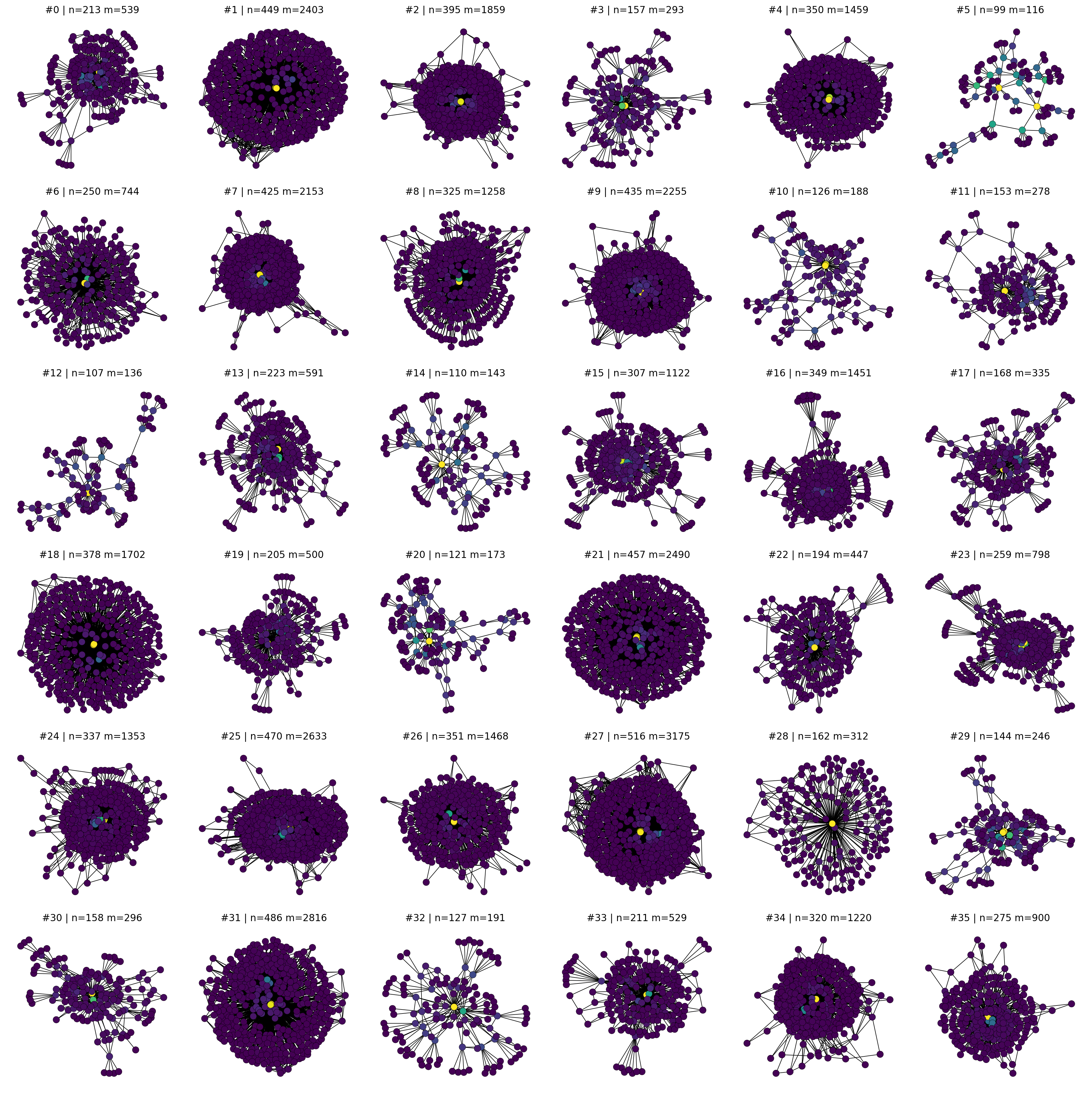}
    \caption{Graphs generated by Flowette for the Ego-small synthetic datasets.
    \label{fig:tree_graphs}}\vspace{-0.2cm}
\end{figure}

\begin{figure}[t!]
    \centering
    \includegraphics[width=0.70\textwidth]{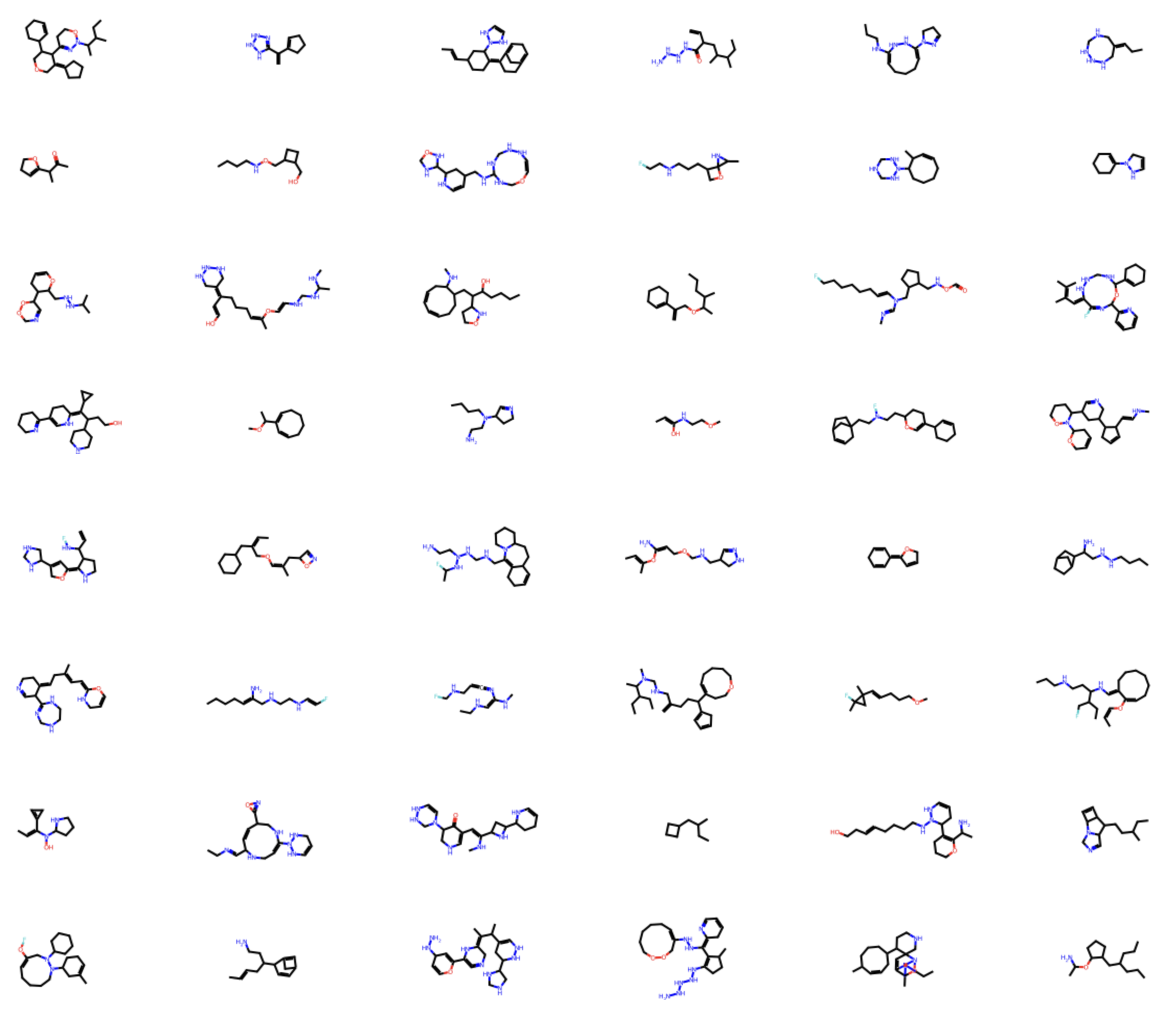}
    \caption{Graphs generated by Flowette for the ZINC250k molecular datasets.
    \label{fig:tree_graphs}}\vspace{-0.2cm}
\end{figure}

\begin{figure}[t!]
    \centering
    \includegraphics[width=0.70\textwidth]{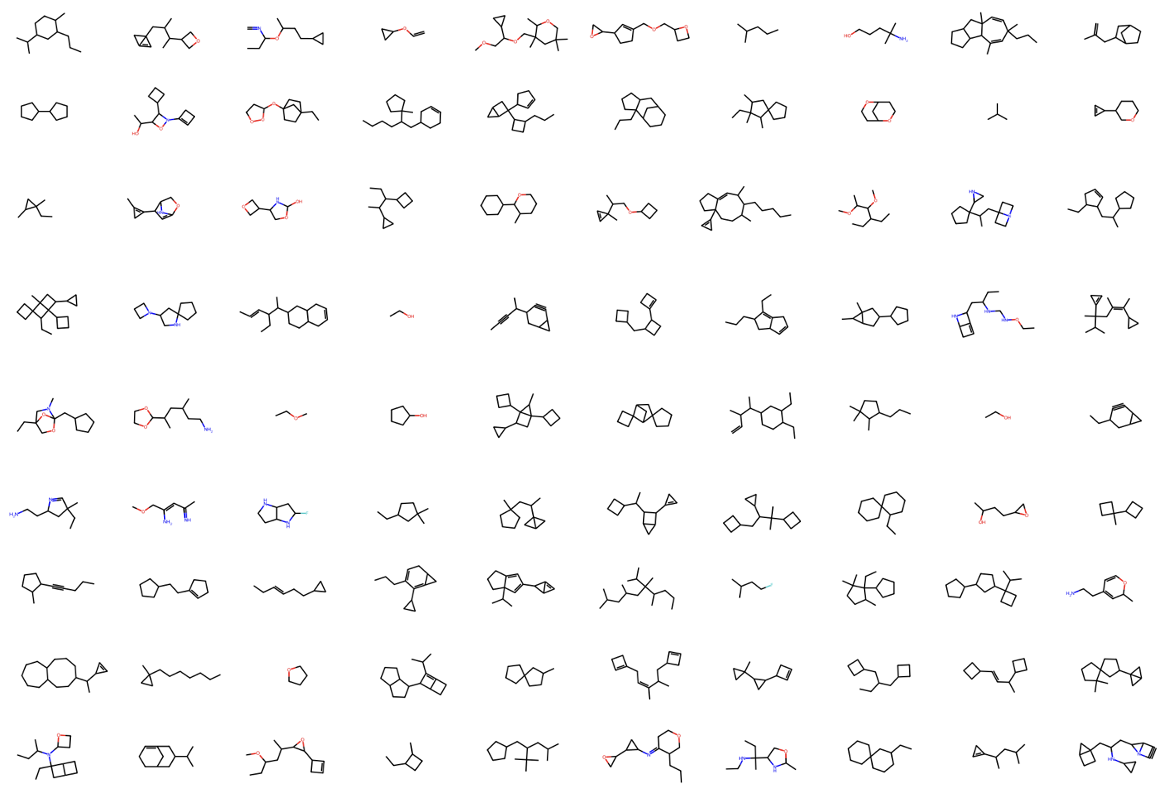}
    \caption{Graphs generated by Flowette for the QM9 molecular datasets.
    \label{fig:tree_graphs}}\vspace{-0.2cm}
\end{figure}

\begin{figure}[t!]
    \centering
    \includegraphics[width=0.70\textwidth]{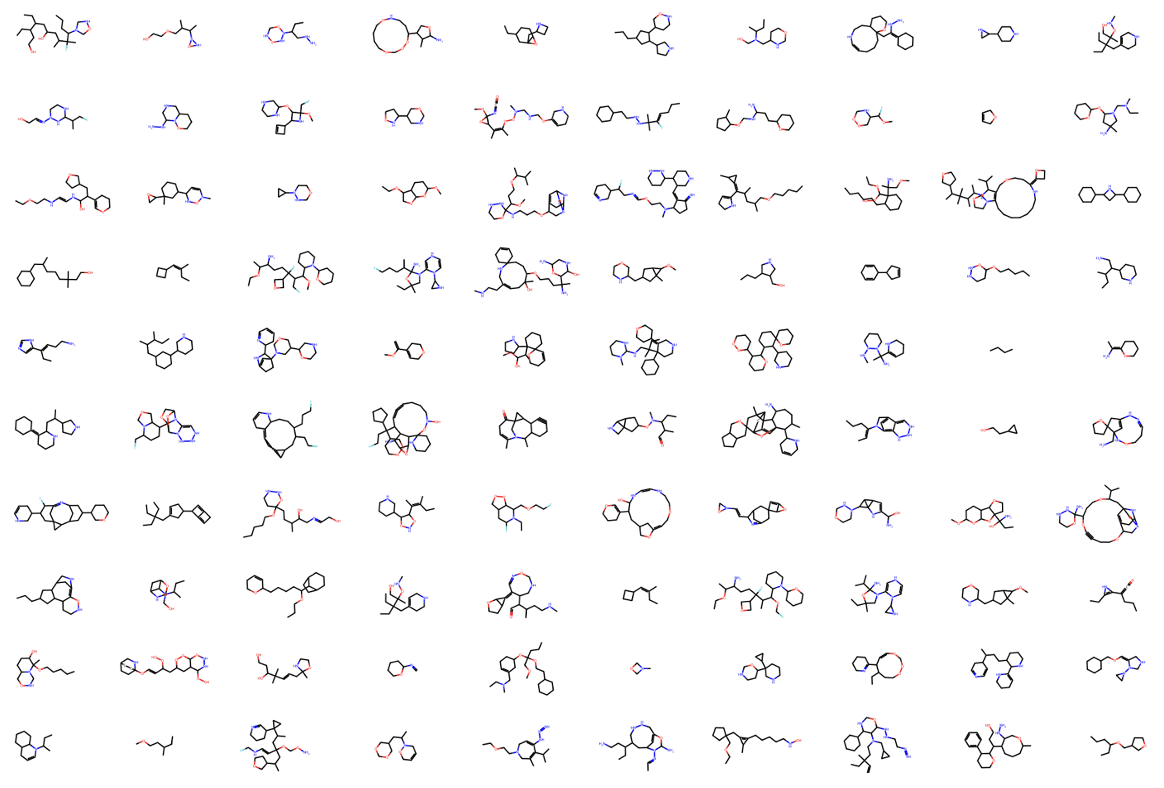}
    \caption{Graphs generated by Flowette for the Guacamol molecular datasets.
    \label{fig:tree_graphs}}\vspace{-0.2cm}
\end{figure}

\begin{figure}[t!]
    \centering
    \includegraphics[width=0.70\textwidth]{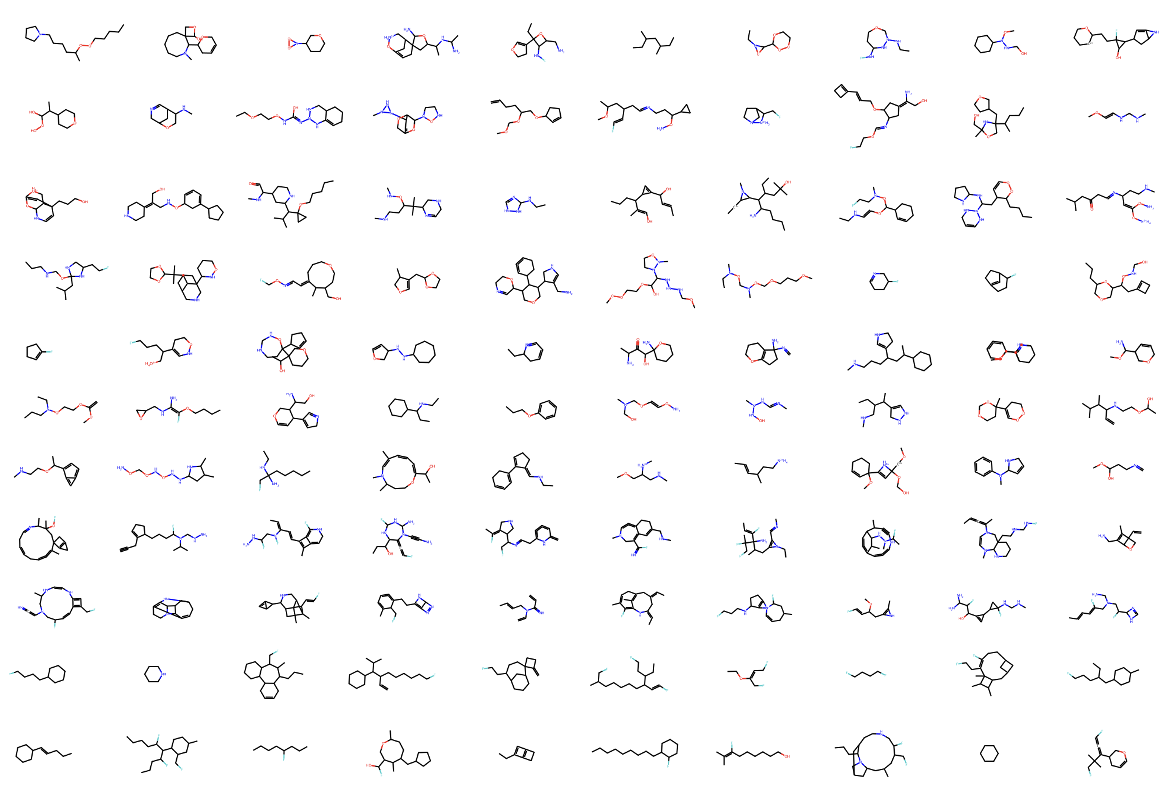}
    \caption{Graphs generated by Flowette for the MOSES molecular datasets.
    \label{fig:tree_graphs}}\vspace{-0.2cm}
\end{figure}

\newpage
\section*{NeurIPS Paper Checklist}

The checklist is designed to encourage best practices for responsible machine learning research, addressing issues of reproducibility, transparency, research ethics, and societal impact. Do not remove the checklist: {\bf The papers not including the checklist will be desk rejected.} The checklist should follow the references and follow the (optional) supplemental material.  The checklist does NOT count towards the page limit. 

Please read the checklist guidelines carefully for information on how to answer these questions. For each question in the checklist:
\begin{itemize}
    \item You should answer \answerYes{}, \answerNo{}, or \answerNA{}.
    \item \answerNA{} means either that the question is Not Applicable for that particular paper or the relevant information is Not Available.
    \item Please provide a short (1--2 sentence) justification right after your answer (even for \answerNA). 
\end{itemize}

{\bf The checklist answers are an integral part of your paper submission.} They are visible to the reviewers, area chairs, senior area chairs, and ethics reviewers. You will also be asked to include it (after eventual revisions) with the final version of your paper, and its final version will be published with the paper.

The reviewers of your paper will be asked to use the checklist as one of the factors in their evaluation. While \answerYes{} is generally preferable to \answerNo{}, it is perfectly acceptable to answer \answerNo{} provided a proper justification is given (e.g., error bars are not reported because it would be too computationally expensive'' or ``we were unable to find the license for the dataset we used''). In general, answering \answerNo{} or \answerNA{} is not grounds for rejection. While the questions are phrased in a binary way, we acknowledge that the true answer is often more nuanced, so please just use your best judgment and write a justification to elaborate. All supporting evidence can appear either in the main paper or the supplemental material, provided in appendix. If you answer \answerYes{} to a question, in the justification please point to the section(s) where related material for the question can be found.

IMPORTANT, please:
\begin{itemize}
    \item {\bf Delete this instruction block, but keep the section heading ``NeurIPS Paper Checklist"},
    \item  {\bf Keep the checklist subsection headings, questions/answers and guidelines below.}
    \item {\bf Do not modify the questions and only use the provided macros for your answers}.
\end{itemize}


\begin{enumerate}

\item {\bf Claims}
    \item[] Question: Do the main claims made in the abstract and introduction accurately reflect the paper's contributions and scope?
    \item[] Answer: \answerYes{} 
    \item[] Justification: The abstract and introduction accurately convey the paper's scope and contributions, clearly outlining the key claims, underlying assumptions, and limitations.

\item {\bf Limitations}
    \item[] Question: Does the paper discuss the limitations of the work performed by the authors?
    \item[] Answer: \answerYes{} 
    \item[] Justification: The limitations are discussed under "Observed Limitations" on page 10.

\item {\bf Theory assumptions and proofs}
    \item[] Question: For each theoretical result, does the paper provide the full set of assumptions and a complete (and correct) proof?
    \item[] Answer: \answerYes{} 
    \item[] Justification: A complete set of assumptions, theorems, and proofs is provided in Sections 3 and 4, as well as Appendices E and H.

    \item {\bf Experimental result reproducibility}
    \item[] Question: Does the paper fully disclose all the information needed to reproduce the main experimental results of the paper to the extent that it affects the main claims and/or conclusions of the paper (regardless of whether the code and data are provided or not)?
    \item[] Answer: \answerYes{} 
    \item[] Justification: The paper provides all the information needed to reproduce the main experimental results on pages 8-9, and in Appendix J.

\item {\bf Open access to data and code}
    \item[] Question: Does the paper provide open access to the data and code, with sufficient instructions to faithfully reproduce the main experimental results, as described in supplemental material?
    \item[] Answer: \answerYes{} 
    \item[] Justification: The paper provides open access to the data and code on page 8. The code is available at: \url{https://github.com/wokas36/Flowette}

\item {\bf Experimental setting/details}
    \item[] Question: Does the paper specify all the training and test details (e.g., data splits, hyperparameters, how they were chosen, type of optimizer) necessary to understand the results?
    \item[] Answer: \answerYes{} 
    \item[] Justification: Yes, all the training and test details are provided on pages 8-10, and Appendix J.

\item {\bf Experiment statistical significance}
    \item[] Question: Does the paper report error bars suitably and correctly defined or other appropriate information about the statistical significance of the experiments?
    \item[] Answer: \answerYes{} 
    \item[] Justification: Mean and standard-deviation is provided in many experimental results.

\item {\bf Experiments compute resources}
    \item[] Question: For each experiment, does the paper provide sufficient information on the computer resources (type of compute workers, memory, time of execution) needed to reproduce the experiments?
    \item[] Answer: \answerYes{} 
    \item[] Justification: The paper provides sufficient information on computational resources and runtime characteristics to support reproducibility.
    
\item {\bf Code of ethics}
    \item[] Question: Does the research conducted in the paper conform, in every respect, with the NeurIPS Code of Ethics \url{https://neurips.cc/public/EthicsGuidelines}?
    \item[] Answer: \answerYes{} 
    \item[] Justification: Yes, this work adheres to the NeurIPS Code of Ethics.

\item {\bf Broader impacts}
    \item[] Question: Does the paper discuss both potential positive societal impacts and negative societal impacts of the work performed?
    \item[] Answer: \answerNA{} 
    \item[] Justification: Not required for our work.
    
\item {\bf Safeguards}
    \item[] Question: Does the paper describe safeguards that have been put in place for responsible release of data or models that have a high risk for misuse (e.g., pre-trained language models, image generators, or scraped datasets)?
    \item[] Answer: \answerNA{} 
    \item[] Justification: The paper does not present any such risks.

\item {\bf Licenses for existing assets}
    \item[] Question: Are the creators or original owners of assets (e.g., code, data, models), used in the paper, properly credited and are the license and terms of use explicitly mentioned and properly respected?
    \item[] Answer: \answerYes{} 
    \item[] Justification: All datasets and models are properly cited and are publicly available.

\item {\bf New assets}
    \item[] Question: Are new assets introduced in the paper well documented and is the documentation provided alongside the assets?
    \item[] Answer: \answerNA{} 
    \item[] Justification: No new assets introduced.

\item {\bf Crowdsourcing and research with human subjects}
    \item[] Question: For crowdsourcing experiments and research with human subjects, does the paper include the full text of instructions given to participants and screenshots, if applicable, as well as details about compensation (if any)? 
    \item[] Answer: \answerNA{} 
    \item[] Justification: Not used in our work.

\item {\bf Institutional review board (IRB) approvals or equivalent for research with human subjects}
    \item[] Question: Does the paper describe potential risks incurred by study participants, whether such risks were disclosed to the subjects, and whether Institutional Review Board (IRB) approvals (or an equivalent approval/review based on the requirements of your country or institution) were obtained?
    \item[] Answer: \answerNA{} 
    \item[] Justification: Not required for our work.

\item {\bf Declaration of LLM usage}
    \item[] Question: Does the paper describe the usage of LLMs if it is an important, original, or non-standard component of the core methods in this research? Note that if the LLM is used only for writing, editing, or formatting purposes and does \emph{not} impact the core methodology, scientific rigor, or originality of the research, declaration is not required.
    \item[] Answer: \answerNA{} 
    \item[] Justification: Doesn't involve LLMs.

\end{enumerate}

\end{document}